%% file: main-arxiv.tex
\documentclass[letterpaper,11pt]{article}
\usepackage{fullpage}
\input{macros}

\usepackage[style=alphabetic,maxnames=10,backref=true]{biblatex}
\newcommand{\citet}[1]{\citeauthor*{#1}~\cite{#1}}
\newtoggle{submission}
\settoggle{submission}{true}

\usepackage[utf8]{inputenc} 
\usepackage[T1]{fontenc}    
\usepackage{url}            
\usepackage{booktabs}       
\usepackage{amsfonts}       
\usepackage{nicefrac}       
\usepackage{microtype}      
\usepackage{graphicx}

\title{Performative Reinforcement Learning}

\author{%
Debmalya Mandal\\ 
MPI-SWS\\
  \texttt{dmandal@mpi-sws.org}\\
  \and Stelios Triantafyllou\\
  MPI-SWS\\
  \texttt{strianta@mpi-sws.org}
  \and Goran Radanovic\\
MPI-SWS\\
  \texttt{gradanovic@mpi-sws.org}\\
}

\begin{document}

\maketitle

\input{0_abstract}

\input{1_introduction}

\input{2_related_work}

\input{3_preliminaries}

\input{4_convergence_retraining}

\input{5_experiments}
\input{6_conclusion}
\printbibliography
\appendix

\input{7.0_appendix_main}
\end{document}

%% file: macros.tex
\usepackage[colorlinks=true, citecolor=blue]{hyperref}

\usepackage{algorithm}
\usepackage{algorithmic}

\usepackage{etoolbox}

\usepackage{savesym,multirow,url,xspace,graphicx,hyperref,enumitem,color}
\usepackage[toc,page,header]{appendix}
\usepackage{minitoc}
\usepackage{dsfont}
\usepackage{multicol}

\usepackage{tikz}
\usetikzlibrary{automata, positioning, arrows}
\tikzset{
	->, 
	every state/.style={thick, fill=gray!10}, 
	initial text=$ $, 
}
\usetikzlibrary{arrows.meta}
\usepackage{pgfplots}
\pgfplotsset{width=10cm,compat=1.9}
\usepackage{diagbox}
\usepackage{caption}

\makeatletter
\newif\if@restonecol
\makeatother

\usepackage{amsmath,amssymb,xfrac,amsthm}
\usepackage{mathtools}
\usepackage{wrapfig}
\setitemize{noitemsep,topsep=1pt,parsep=1pt,partopsep=1pt, leftmargin=12pt}
\newtheorem{lemma}{Lemma}
\newtheorem{theorem}{Theorem}
\newtheorem*{theorem*}{Theorem}

\newtheorem{proposition}{Proposition}
\newtheorem{definition}{Definition}
\newtheorem{assumption}{Assumption}

\makeatletter

\newcommand{\Rmnum}[1]{\expandafter\@slowromancap\romannumeral #1@}
\makeatother
\usepackage{caption}

\usepackage{enumitem}
\usepackage{wasysym}

\DeclareMathOperator*{\argmin}{arg\,min}
\DeclareMathOperator*{\argmax}{arg\,max}

\usepackage{bm}
\usepackage{bbm}
\usepackage{cleveref}
\let\mathbbm\mathds
\newcommand{\norm}[1]{\left\lVert#1\right\rVert}

\newcommand{\abs}[1]{\left|#1\right|}

\newcommand{\Pro}{\mathbb{P}}

\newcommand{\agentone}{$A_1$}
\newcommand{\agenttwo}{$A_2$}

\newcommand{\one}{\mathbbm{1}}
\newcommand{\set}[1]{\left\{ #1 \right\}}
\renewcommand{\bar}{\overline}
\renewcommand{\tilde}{\widetilde}
\newcommand{\E}{\mathbb{E}}
\newcommand{\calL}{\mathcal{L}}
\newcommand{\calP}{\mathcal{P}}

\newcommand{\calR}{\mathcal{R}}

\newcommand{\calD}{\mathcal{D}}
\newcommand{\calH}{\mathcal{H}}
\newcommand{\calC}{\mathcal{C}}

\newcommand{\R}{\mathbb{R}}

\newcommand{\eps}{\varepsilon}

\newcommand{\tilh}{\tilde{h}}
\newcommand{\proj}{\mathrm{Proj}}
\newcommand{\GD}{\mathrm{GD}}
\newcommand{\Identity}{\mathrm{Id}}

%% file: 0_abstract.tex
\begin{abstract}
We introduce the framework of performative reinforcement learning where the policy chosen by the learner affects the underlying reward and transition dynamics of the environment. Following the recent literature on performative prediction~\cite{PZMH20}, we introduce the concept of performatively stable policy. We then consider a regularized version of the reinforcement learning problem and show that repeatedly optimizing this objective converges to a performatively stable policy under reasonable assumptions on the transition dynamics. Our proof utilizes the dual perspective of the reinforcement learning problem and may be of independent interest in analyzing the convergence of other algorithms with decision-dependent environments.
We then extend our results for the setting where the learner just performs gradient ascent steps instead of fully optimizing the objective, and for the setting where the learner has access to a finite number of trajectories from the changed environment. For both the settings, we leverage the dual formulation of performative reinforcement learning, and establish convergence to a stable solution. Finally, through extensive experiments on a grid-world environment, we demonstrate the dependence of convergence on various parameters e.g. regularization, smoothness, and the number of samples.
\end{abstract}

%% file: 1_introduction.tex
\section{Introduction}
Over the last decade, advances in reinforcement learning techniques enabled several breakthroughs in AI. These milestones include AlphaGo~\cite{AlphaGo}, Pluribus~\cite{BS19}, and AlphaStar~\cite{AlphaStar}. Such success stories of reinforcement learning in multi-agent game playing environments have naturally led to the adoption of RL in many real-world scenarios e.g. recommender systems~\cite{Charu20}, and healthcare~\cite{ERR+19}. However, these critical domains often  pose new challenges including the mismatch between deployed policy and the target environment.

Existing frameworks of reinforcement learning ignore the fact that a deployed policy might change the underlying environment (i.e., reward, or probability transition function, or both). Such a mismatch between the deployed policy and the environment often arises in practice.  For example, recommender systems often use contextual Markov decision process to model interaction with a user~\cite{HHMM+20}. In such a contextual MDP, the initial context/user feature is drawn according a distribution, then the user interacts with the platform according to the context-specific MDP. However, it has been  repeatedly observed that such recommender systems not only change the user demographics (i.e. distribution of contexts) but also how they interact with the platforms~\cite{CSE18, MAPM+20}. Our second  example comes from autonomous vehicles (AV). Even if we ignore the multi-agent aspect of these learning algorithms, a deployed AV might change how the pedestrians, and other cars behave, and the resulting environment might be quite different from what the designers of the AV had in mind~\cite{NNPS17}.

Recently, \citet{PZMH20} introduced the notion of \emph{performative prediction}, where the predictions made by a classifier changes the data distribution. However, in the context of reinforcement learning, the situation is different as the changing transition dynamics introduces additional complexities. If the underlying probability transition function changes, then the class of feasible policies and/or models changes with time. This implies that we need a framework that is more general than the framework of performative prediction, and can model policy-dependent outcomes in reinforcement learning.

\textbf{Our Contributions}: In this paper, we introduce the notion of \emph{performative reinforcement learning} where the deployed policy not only changes the reward vector but also the underlying transition probability function. We introduce the notion of \emph{performatively stable policy} and show under what conditions various repeated retraining methods (e.g., policy optimization, gradient ascent etc.) converges to such a stable solution. Our precise contributions are the following.
\begin{itemize}
    \item We consider a regularized version of the reinforcement learning problem where the variables are long-term discounted state-action occupancy measures. We show that, when both the probability transition function and the reward function changes smoothly in response to the occupancy measure, repeated optimization of regularized reinforcement learning converges to a stable solution.
    \item We then show that if the learner performs repeated projected gradient ascent steps, then also convergence is guaranteed provided that the step-size is small enough. Compared to the supervised learning setting~\cite{PZMH20}, the projection step is necessary as the probability transition function, and hence the space of occupancy measures change with time.
    \item Next we extend our result to the finite samples setting, where the learner has access to a collection of samples from the updated environment. For this setting, we use an empirical version of the Lagrangian of the regularized RL problem. We show that repeatedly solving a saddle point of this empirical Lagrangian (max player corresponds to the learner) also converges to a stable solution provided that the number of samples is large enough.
    \item Finally, we empirically evaluate the effect of various parameter choices (regularization, smoothness, number of samples etc.) through extensive experiments on a two-agent grid-world environment. In this environment, the first agent performs various types of repeated retraining, whereas the second agent responds according to a smooth response function.
\end{itemize}

\textbf{Our Techniques}: At a high level, our theoretical results might look similar to the results derived by \citet{PZMH20}. However, there are many challenges.
\begin{itemize}
    \item We repeatedly maximize a regularized objective whose feasible region is the space of feasible occupancy measures. As the probability transition function changes with time, the feasible region of the optimization problem also changes with time. So ideas from supervised classification setting~\cite{MPZH20} cannot be applied directly. Therefore, we look at the dual problem which is strongly-convex and mapping from occupancy measure to the corresponding dual optimal solution turns out to be a contraction. We believe that the dual perspective on performative prediction might be useful for analyzing the convergence of other algorithms with decision-dependent outcomes.
    \item For performative reinforcement learning, the finite sample setting is very different than the supervised learning setting. This is because we do not have independent sample access from the new environment. At time-step $t$, we can only access the new model through the policy $\pi_t$ (or occupancy measure $d_t$). In that sense, the learner faces an offline reinforcement learning problem where the samples are collected from the behavior policy $\pi_t$. This is also the reason we need an additional overlap assumption, which is often standard in offline reinforcement learning~\cite{MS08}.
\end{itemize}

%% file: 2_related_work.tex

\textbf{Related Work:}
\citet{PZMH20} introduced the notion of \emph{performative prediction}. Subsequent papers have considered several aspects of this framework including optimization~\cite{MPZH20, MPZ21}, multi-agent systems~\cite{NFDF+22}, and population dynamics~\cite{BHK20}. However, to the best of our knowledge, performative prediction in sequential decision making is mostly unexplored. A possible exception is ~\cite{JLOS21} who consider a setting where the transition and reward of the underlying MDP depend non-deterministically on the deployed policy.  Since the mapping is non-deterministic, it doesn't  lead to a notion of equilibrium, and the authors instead focus on the optimality and convergence of various RL algorithms.

\textbf{Stochastic Stackelberg Games}: Our work is also closely related to the literature on stochastic games \cite{shapley1953stochastic,filar2012competitive}, and in particular, those that study Stackelberg (commitment) strategies \cite{von2010market}, where a leader commits a policy to which a follower (best) responds. While different algorithmic approaches have been proposed for computing Stackelberg equilibria (SE) in stochastic games or related frameworks~\cite{vorobeychik2012computing,letchford2012computing,dimitrakakis2017multi}, computing optimal commitment policies has shown to be a computationally intractable  (NP-hard) problem~\cite{letchford2012computing}. More recent works have studied learning SE in stochastic games, both from practical perspective~\cite{rajeswaran2020game,mishra2020model,huang2022robust} and theoretical perspective~\cite{bai2021sample,zhong2021can}. The results in this paper differ from this line of work in two ways. Firstly, our framework abstracts the response model of an agent’s effective environment in that it does not model it through a rational agency with a utility function. Instead, it is more aligned with the approaches that learn the response function of the follower agent~\cite{sinha2016learning,kar2017cloudy,sessa2020learning}, out of which the closest to our work is \cite{sessa2020learning} that considers repeated games. Secondly, given that we consider solution concepts from performative prediction rather than SE, we focus on repeated retraining as the algorithm of interest, rather than bi-level optimization approaches.



\textbf{Other related work:} We also draw a connection to other RL frameworks. Naturally, this work relates to RL settings that study non-stationary environments. These include recent learning-theoretic results, such as \cite{gajane2018sliding,fei2020dynamic,domingues2021kernel,cheung2020reinforcement,wei2021non} that allow non-stationary rewards and transitions provided a bounded number or amount of changes (under no-regret regime), the extensive literature on learning under adversarial reward functions  (e.g., \cite{even2004experts,neu2012adversarial,dekel2013better,rosenberg2019online}), or the recent works on learning under corrupted feedback \cite{lykouris2021corruption}.
However, the setting of this paper is more structured, i.e., the environment responds to the deployed policy and does not arbitrarily change. 
Our work is also broadly related to the extensive literature on multi-agent RL literature -- we refer the reader to \cite{zhang2021multi} for a selective overview. A canonical example of a multi-agent setting that closely relates to the setting of this paper is human-AI cooperation, where the AI's policy influences the human behavior~\cite{dimitrakakis2017multi,nikolaidis2017game,crandall2018cooperating,radanovic2019learning,carroll2019utility}. In fact, our experiments are inspired by human-AI cooperative interaction. While the algorithmic framework of repeated retraining has been discussed in some of the works on cooperative AI (e.g., see \cite{carroll2019utility}), these works do not provide a formal treatment of the problem at hand. Finally, this paper also relates to the extensive literature on offline RL~\cite{levine2020offline} as the learner faces an offline RL problem when performing repeated retraining with finite samples.
We utilize the analysis of \cite{ZHHJ+22} to establish finite sample guarantees, under a standard assumption on sample generation \cite{MS08, FSM10, XJ21}, and overlap in occupancy measure~\cite{MS08, ZHHJ+22}. 
Note that offline RL has primarily focused on static RL settings in which the policy of a learner does not affect the model of the underlying environment.


%% file: 3_preliminaries.tex
\section{Model}\label{sec.model}
We are primarily concerned with Markov Decision Processes (MDPs) with a fixed state space $S$, action set $A$, discount factor $\gamma$, and starting state distribution $\rho$. The reward and the probability transition functions of the MDP will be functions of the adopted policy. We consider infinite-horizon setting where the learner's goal is to  minimize the total sum of discounted rewards. We will write $s_k$ to denote the state visited at time-step $k$ and $a_k$ to denote the action taken at time-step $k$.
When the learner adopts policy $\pi$, the underlying MDP has reward function $r_\pi$ and probability transition function $P_\pi$. We will write $M(\pi)$ to denote the corresponding MDP, i.e., $M(\pi) = (S, A, P_\pi, r_\pi, \rho)$. Note that only the reward and the transition probability function change according to the policy $\pi$. 

When an agent adopts policy $\pi$ and the underlying MDP is $M(\pi') = (S,A,P_{\pi'}, r_{\pi'}, \rho)$ the probability of a trajectory $\tau = (s_k, a_k)_{k=0}^{\infty}$ is given as $\Pro(\tau) = \rho(s_0) \prod_{k=1}^\infty P_{\pi'}(s_{k+1}|s_k, \pi(s_k))$. We will write $\tau \sim \Pro^{\pi}_{\pi'}$ to denote such a trajectory $\tau$. Given a policy $\pi$ and an underlying MDP $M(\pi')$ we write $V^{\pi}_{\pi'}(\rho)$ to define the value function w.r.t. the starting state distribution $\rho$. This is defined as follows.
\begin{equation}
    \label{defn:value-function}
    V^{\pi}_{\pi'}(\rho) = \E_{\tau \sim \Pro^{\pi}_{\pi'}} \left[\sum_{k=0}^\infty \gamma^k r_{\pi'}(s_k,a_k) \rvert \rho \right]
\end{equation}
\textbf{Solution Concepts}: Given the definition of the value function~\cref{defn:value-function}, we can now define the solution concepts for our setting. First we define the performative value function of a policy which is the expected total return of the policy given the environment changes in response to the policy. 
\begin{definition}[Performative Value Function]
Given a policy $\pi$, and a starting state distribution $\rho \in \Delta(S)$, the performative value function $V^{\pi}_\pi(\rho)$ is defined as 
$
V^{\pi}_{\pi}(\rho) = \E_{\tau \sim \Pro^{\pi}_{\pi}} \left[\sum_{t=0}^\infty \gamma^t r_{\pi}(s_t, a_t) \rvert \rho \right]
$.
\end{definition}

We now define the performatively optimal policy, which  maximizes performative value function.
\begin{definition}[Performatively Optimal Policy]\label{defn:perf-stability}
A policy $\pi$ is performatively optimal if it maximizes performative value function, i.e., 
$
\pi \in \argmax_{\pi'} V^{\pi'}_{\pi'}(\rho)
$.
\end{definition}
We will write $\pi_{P}$ to denote the  performatively optimal policy. Although, $\pi_{P}$ maximizes the performative value function, it need not be stable, i.e., the policy need not be optimal with respect to the changed environment $M(\pi_P)$. We next define the notion of performatively stable policy which captures this notion of stability.
\begin{definition}[Performatively Stable Policy]
A policy $\pi$ is performatively stable if it satisfies the
condition
$
\pi \in \argmax_{\pi'} V^{\pi'}_\pi(\rho)
$.
\end{definition}
We will write $\pi_S$ to denote the performatively stable policy. The definition of performatively stable policy implies that if the underlying MDP is $M(\pi_S)$ then an optimal policy is $\pi_{S}$. This means after deploying the policy $\pi_S$ in the MDP $M(\pi_S)$ the environment doesn't change and the learner is also optimizing her reward in this stable environment.
We next show that even for an MDP with a single state, these two solution concepts can be very different.

\textbf{Example}: Consider an MDP with single state $s$ and two actions $a$ and $b$. If a policy plays arm $a$ with probability $\theta$ and $b$ with probability $1-\theta$ then we have $r(s,a) = \frac{1}{2} - \epsilon \theta$ and $r(s,b) = \frac{1}{2} + \epsilon \theta$ for some $\epsilon < \frac{1}{4}$. Note that if $\theta_S = 0$ then both the actions give same rewards, and the learner can just play action $b$. Therefore, a policy that always plays action $b$ is a stable policy and achieves a reward of $\frac{1}{2(1-\gamma)}$. On the other hand, a policy that always plays action $a$ with probability $\theta = 1/4$ has the performative value function of
$$
\frac{\theta (1/2 - \epsilon \theta)}{1-\gamma} + \frac{(1-\theta) (1/2 + \epsilon \theta)}{1-\gamma} = \frac{1/2 + \epsilon/ 8}{1-\gamma}
$$
So, for  $\epsilon>0$, a performatively optimal policy can achieve higher value function than a stable policy.


We will mainly consider with the tabular MDP setting where the number of states and actions are finite. Even though solving tabular MDP in classic reinforcement learning problem is relatively straight-forward, we will see that even the tabular setting raises many interesting challenges for the setting of performative reinforcement learning. 

\textbf{Discounted State, Action Occupancy Measure}: Note that it is not a priori clear if there always exists a performatively stable policy (as defined in~\eqref{defn:perf-stability}). This is because such existence guarantee is usually established through a fixed-point argument, but the set of optimal solutions need not be convex. If both $\pi_1$ and $\pi_2$ optimizes \eqref{defn:perf-stability}, then their convex combination might not be optimal. So, in order to find a stable policy, we instead consider the linear programming formulation of reinforcement learning. Given a policy $\pi$, its long-term discounted state-action occupancy measure in the MDP $M(\pi)$ 
is defined as
$
d^\pi(s,a) = \E_{\tau \sim \Pro^\pi_\pi} \left[\sum_{k=0}^\infty \gamma^k \one\set{s_k = s, a_k = a}\rvert \rho\right]
$.
Given an occupancy measure $d$, one can consider the following policy $\pi^d$ which has occupancy measure $d$.
\begin{align}\label{eq:dtopi}
    \pi^d(a|s) = \left\{\begin{array}{cc}
        \frac{d(s,a)}{\sum_b d(s,b)} & \textrm{ if } \sum_a d(s,a) > 0 \\
        \frac{1}{A} & \textrm{ otherwise }
    \end{array} \right.
\end{align}
With this definition, we can pose the problem of finding a performatively stable occupancy measure. 
An occupancy measure $d$ is performatively stable if it is the optimal solution of the following problem.
\begin{align}
        d \in &\argmax_{d \ge 0} \sum_{s,a} d(s,a) r_d(s,a)\label{eq:perf-stable-policy}\\
        \textrm{s.t.} & \sum_a d(s,a) = \rho(s) + \gamma \cdot \sum_{s',a} d(s',a) P_d(s',a,s) \ \forall s \nonumber 
\end{align}
With slight abuse of notation we are writing $P_d$ as $P_{\pi^d}$ (as defined in equation~\eqref{eq:dtopi}). If either the probability transition function or the reward function changes drastically in response to the occupancy measure then the optimization problem~\ref{eq:perf-stable-policy} need not even have a stable point. Therefore, as is standard in performative prediction, we make the following sensitivity assumption regarding the underlying environment.
\begin{assumption}\label{sensitivity-assumption}
The reward and probability transition mappings  
are $(\eps_r, \eps_p)$-sensitive i.e. the following holds for any two occupancy measures $d$ and $d'$
\begin{align*}
&\norm{r_d - r_{d'}}_2 \le \eps_r \norm{d - d'}_2,\ \norm{P_d -P_{d'}}_2 \le \eps_p \norm{d - d'}_2
\end{align*}
\end{assumption}

Since the objective function of \cref{eq:perf-stable-policy} is convex (in fact linear), and the set of optimal solutions is convex, a simple application of Kakutani fixed point theorem~\cite{Glicksberg52} shows that there always exists a performative stable solution.\footnote{The proof of this result and all other results are provided in the appendix.}
\begin{proposition}\label{prop:existence}
Suppose assumption~\ref{sensitivity-assumption} holds for some constants $(\eps_r, \eps_p)$, then the optimization problem~\ref{eq:perf-stable-policy} always has a fixed point.
\end{proposition}




%% file: 4_convergence_retraining.tex

\section{Convergence of Repeated Retraining}\label{sec.convergence_retraining}  
Even though the optimization problem~\eqref{eq:perf-stable-policy} is guaranteed to have a stable solution, it is not clear that repeatedly optimizing this objective converges to such a point. We now consider a regularized version of the optimization problem~\eqref{eq:perf-stable-policy}, and attempt to obtain a stable solution of the regularized problem. In subsection~\eqref{subsec:apx-unregularization} we will show that such a stable solution guarantees approximate stability with respect to the original unregularized objective~\eqref{eq:perf-stable-policy}.
\begin{align}
        \max_{d \ge 0}\ &  \sum_{s,a} d(s,a) \textcolor{blue}{r_d(s,a)} - \frac{\lambda}{2}\norm{d}_2^2\label{eq:regularized-rl-discounting-occupancy}\\
       \textrm{s.t. } & \sum_a d(s,a) = \rho(s) + \gamma \cdot \sum_{s',a} d(s',a) \textcolor{blue}{P_d(s',a,s)}\ \forall s \nonumber
\end{align}

Here $\lambda > 0$ is a constant that determines the strong-concavity of the objective. Before analyzing the behavior of repeatedly optimizing the new objective~\eqref{eq:regularized-rl-discounting-occupancy} we discuss two important issues. First, we consider quadratic regularization for simplicity, and our results can be easily extended to any strongly-convex regularizer. Second, we apply regularization in the occupancy measure space, but regularization in policy space is commonly used~\cite{MBMG+16}. 
%
Since the performatively stable occupancy measure $d_S$ is not known, a common strategy adopted is repeated policy optimization. At time $t$, the learner obtains the occupancy measure $d_t$, and deploys the  policy $\pi_t$ (as defined in \eqref{eq:dtopi}). In response, the environment changes to $P_t = P_{d_t}$ and $r_{t} = r_{d_t}$, and the learning agent solves the following optimization problem to obtain the next occupancy measure $d_{t+1}$.
\begin{align}
        \max_{d\ge 0}\ &  \sum_{s,a} d(s,a) \textcolor{black}{r_t(s,a)} - \frac{\lambda}{2}\norm{d}_2^2\label{eq:regularized-rl-discounting}\\
       \textrm{s.t. } & \sum_a d(s,a) = \rho(s) + \gamma \cdot \sum_{s',a} d(s',a) \textcolor{black}{P_t(s',a,s)}\ \forall s\nonumber
\end{align}

We next show that repeatedly solving the  problem (\ref{eq:regularized-rl-discounting}) converges to a  stable point. 
%
%

\begin{theorem}\label{thm:primal-convergence}
Suppose assumption \ref{sensitivity-assumption} holds with $\lambda > \frac{12S^{3/2}(2\epsilon_r + 5S\epsilon_p)}{(1-\gamma)^4}$. Let $\mu = \frac{12S^{3/2}(2\epsilon_r + 5S\epsilon_p)}{\lambda (1-\gamma)^4}$. Then for any $\delta > 0$ we have
$$
\norm{d_t - d_S}_2 \le \delta \quad \forall t \ge 2 (1-\mu)^{-1}\ln \left( {2}/{\delta (1-\gamma)}\right)
$$
\end{theorem}

Here we discuss some of the main challenges behind the proof of this theorem.
\begin{itemize}
\item The primal objective function (\ref{eq:regularized-rl-discounting}) is strongly concave but the feasible region of the optimization problem changes with each iteration. So we cannot apply the results from performative prediction \cite{PZMH20}, and instead, look at the dual objective which is $A(1-\gamma)^2/\lambda$-strongly convex.
\item Although the dual problem is strongly convex, it does not satisfy Lipschitz continuity w.r.t. $P$. However, we show that the norm of the optimal solution of the dual problem is bounded by $O\left( S/ (1-\gamma)^2\right)$ and this is sufficient to show that the dual objective is Lipschitz-continuous with respect to $P$ at the dual optimal solution. We show that the proof argument used in \citet{PZMH20} works if we replace global Lipschitz-continuity by such local Lipschitz-continuity.
\item Finally, we translate back the bound about the dual solution  to a guarantee about the primal solution ($\norm{d_t - d_S}_2)$ using the strong duality of the optimization problem~\ref{eq:regularized-rl-discounting}. This step crucially uses the quadratic regularization in the primal.
\end{itemize}

Here we make several observations regarding the assumptions required by Theorem~\ref{thm:primal-convergence}. First,
Theorem~\ref{thm:primal-convergence} suggests that for a given sensitivity $(\epsilon_r,\epsilon_p)$ and discount factor $\gamma$, one can  choose the  parameter $\lambda$ so that the convergence to a stable point is guaranteed.
Second,
for a given value of $\lambda$ and $\gamma$ if the sensitivity is small enough, then repeatedly optimizing \ref{eq:regularized-rl-discounting} converges to a stable point.


\subsection{Gradient Ascent Based Algorithm}
We now extend our result for the setting where the learner does not fully solve the optimization problem~\ref{eq:regularized-rl-discounting} every iteration. Rather, the learner takes a gradient step with respect to the changed environment 
every iteration. Let $\calC_t$ denote the set of occupancy measures that are compatible with probability transition function $P_t$.
\begin{align}
    \calC_t = \bigg\{d : &\sum_a d(s,a) = \rho(s) + \gamma  \sum_{s',a} d(s',a) P_t(s',a,s) \ \forall s  
     \textrm{and } d(s,a) \ge 0 \ \forall s,a \bigg\} \label{eq:defn-c-t}
\end{align}
Then the gradient ascent algorithm first takes a gradient step according to the objective function $r_t^\top d - \frac{\lambda}{2} \norm{d}_2^2$ and then projects the resulting occupancy measure onto  $\calC_t$.
\begin{align}
    d_{t+1} = \proj_{\calC_t} \left(d_t + \eta \cdot\left(r_t - \lambda d_t \right) \right) = \proj_{\calC_t} \left( (1-\eta \lambda) d_t + \eta r_t \right) \label{eq:gradient-ascent-pp}
\end{align}
Here $\proj_\calC(v)$ finds a point in $\calC$ that is closest to $v$ in $L_2$-norm. 
We next show that repeatedly taking projected gradient ascent steps with appropriate step size $\eta$ converges to a stable point. 

 
 \begin{theorem}\label{thm:rga-convergence}
 Let $\lambda \ge \max\set{4\epsilon_r, 2S, \frac{20\gamma^2 S^{1.5}(\epsilon_r + \epsilon_p)}{(1-\gamma)^2}}$, step-size $\eta = \frac{1}{\lambda}$ and $\mu = \sqrt{\frac{64\gamma^2 \epsilon_p^2}{(1-\gamma)^4}\left( 1 + \frac{30 \gamma^4 S^2}{(1-\gamma)^4}\right)}$. Suppose assumption~\ref{sensitivity-assumption} holds with $\epsilon_p < \min \set{\frac{\gamma S}{3}, \frac{(1-\gamma)^4}{100 \gamma^3 S}}$. Then for any $\delta > 0$ we have
 $$
 \norm{d_t - d_S}_2 \le \delta \quad \forall t \ge (1-\mu)^{-1}\ln \left( {2}/{\delta(1-\gamma)}\right)
 $$
 \end{theorem}
 
 \textbf{Proof Sketch}: First, the projection step~\ref{eq:gradient-ascent-pp} can be computed through the following convex  program. 
\begin{align}
        \min_{d\ge 0}\  &\frac{1}{2}\norm{d - (1-\eta \lambda) d_t  - \eta r_t }_2^2\label{eq:primal-projection} \\
        \textrm{s.t. } & \sum_a d(s,a) = \rho(s) + \gamma \cdot \sum_{s',a} d(s',a) P_t(s',a,s) \ \forall s \nonumber 
\end{align}
 Even though the objective above is convex, its feasible region changes with time. So we again look at the dual objective which is strongly concave and has a fixed feasible region. Given an occupancy measure $d_t$, let $\GD_{\eta}(d_t)$ be the optimal solution of the problem~\ref{eq:gradient-ascent-pp}. We show that if the step-size $\eta$ is chosen small enough then the operator $\GD_\eta(\cdot)$ is a contraction mapping by first proving the corresponding optimal dual solution forms a contraction, and then using strong duality to transfer the guarantee back to the primal optimal solutions. Finally, we show that the fixed point of the mapping $\GD_\eta(\cdot)$ indeed coincides with the performatively stable solution $d_S$.
 
 \subsection{Finite Sample Guarantees}\label{sec:finite-sample}
So far we assumed that after deploying the policy corresponding to the occupancy measure $d_t$ we observe the updated environment $M_t = (S,A,P_t,r_t,\gamma)$. However, in practice, we do not have access to the true model but only have access to samples from the updated environment. 
Our setting is more challenging than the finite samples setting considered by \citet{PZMH20}. Unlike the supervised learning setting, we do not have access to independent samples from the new environment. At time $t$ we can deploy policy $\pi_t$ corresponding to the occupancy measure $d_t$, and can access trajectories from the new environment $M_t$ only through the policy $\pi_t$. Therefore, at every step, the learner faces an offline reinforcement learning problem where the policy $\pi_t$ is a behavioral policy.

A standard assumption in offline reinforcement learning is overlap in occupancy measure between the behavior policy and a class of target policies~\cite{MS08}. Without such overlap, one can get no information regarding the optimal policy from the trajectories visited by the behavioral policy. We make the following assumption regarding the overlap in occupancy measure between a deployed policy and the optimal policy in the changed environment.

\begin{assumption}\label{asn:overlap}
Given an occupancy measure $d$, let $\rho^\star_d$ be the optimal occupancy measure maximizing ~\ref{eq:regularized-rl-discounting}, and $\bar{d}$ is the occupancy measure of $\pi^d$ in $P_d$. Then there exists $B > 0$ so that
$$
\max_{s,a} \abs{\frac{\rho^\star_d(s,a)} {\bar{d}(s,a)} } \le B \ \quad \forall d
$$
\end{assumption}
When there is no performativity, $\bar{d}$ equals $d$ and the assumption states overlap between the occupancy measure of the deployed policy and the optimal policy. This is the standard assumption of single policy coverage in offline reinforcement learning. When there is performativity, $\bar{d}$ can be different than $d$ since the deployed policy $\pi^d$ might have occupancy measure different than $d$ in the changed model $P_d$, and in that case we require overlap between $\bar{d}$ and the optimal occupancy measure. Assumption~\eqref{asn:overlap} is also significantly weaker than the uniform coverage assumption which requires $\max_d \max_{s,a} \overline{d}(s,a) > 0$ as it allows the possibility that $\bar{d}(s,a)=0$ as long as the optimal policy doesn't visit state $s$ or never takes action $a$ in state $s$. 


\textbf{Data}: We assume the following model of sample generation at time $t$. Given the occupancy measure $d_t$ let  the normalized occupancy measure be $\tilde{d}_t (s,a) = (1-\gamma) {d_t(s,a)}$. First, sample a state, action pair $(s_i,a_i)$ i.i.d as $(s_i,a_i) \sim \tilde{d}_t$, then reward $r_i \sim r_t(s_i,a_i)$, and finally the next state $s_i' \sim P_t(\cdot|s_i,a_i)$. We  have access to $m_t$ such tuples at time $t$ and the data collected at time is given as $\calD_t = \set{(s_i, a_i, r_i, s_i')}_{i=1}^{m_t}$. We would like to point out that this is a standard model of sample generation in offline reinforcement learning (see e.g. \cite{MS08, FSM10, XJ21}).
%
%
%
%

With finite samples, the learner needs to optimize an empirical version of the optimization problem~\ref{eq:regularized-rl-discounting}, and the choice of such an empirical objective is important for convergence. We follow the recent literature on offline reinforcement learning~\cite{ZHHJ+22} and consider the  Lagrangian of \cref{eq:regularized-rl-discounting}.
\begin{align*}
    &\calL(d,h; M_t) =  d^\top {r}_t - \frac{\lambda}{2} \norm{d}_2^2 + \sum_s h(s)  \left( -\sum_a d(s,a) + \rho(s) + \gamma \cdot \sum_{s',a} d(s',a) {P}_t(s',a,s) \right)\nonumber \\
    = &- \frac{\lambda}{2} \norm{d}_2^2 + \sum_s h(s) \rho(s)+
     \sum_{s,a} d_t(s,a) \frac{d(s,a)}{d_t(s,a)}  \left(r_t(s,a) - h(s) + \gamma \sum_{s'} P_t(s,a,s') h(s') \right) 
\end{align*}
The above expression motivates the following empirical version of the Lagrangian.
\begin{align}
    \hat{\calL}(d,h; M_t) &=  - \frac{\lambda}{2} \norm{d}_2^2 + \sum_s h(s) \rho(s) + \sum_{i=1}^{m_t} \frac{d(s_i,a_i)}{d_t(s_i,a_i)}  \frac{ r(s_i,a_i) - h(s_i) + \gamma h(s_i')  }{m_t(1-\gamma)}\label{eq:empirical-Lagrangian}
\end{align}
We repeatedly solve for a saddle point of the objective above.
\begin{align}\label{eq:repeated-optim-finite}
    (d_{t+1}, h_{t+1}) = \argmax_{d} \ \argmin_{h} \hat{\calL}(d,h; M_t)
\end{align}
The next theorem provides convergence guarantees of the repeated optimization procedure~\eqref{eq:repeated-optim-finite} provided that the number of samples is large enough.

\begin{theorem}[Informal Statement]\label{thm:finite-samples-convergence}
Suppose assumption~\ref{sensitivity-assumption} holds with $\lambda \ge \frac{24 S^{3/2} (2\epsilon_r + 5S\epsilon_p)}{(1-\gamma)^4}$, and assumption~\ref{asn:overlap} holds with parameter $B$. Let $\mu = \frac{24 S^{3/2} (2\epsilon_r + 5S\epsilon_p)}{(1-\gamma)^4}$. For any $\delta > 0$, and error probability $p$ if we repeatedly solve the optimization problem~\eqref{eq:repeated-optim-finite} with number of samples $m_t \ge \tilde{O}\left( \frac{A^2 B^2}{\delta^4 (2\epsilon_r + 5S\epsilon_p)^2} \ln\left( \frac{t}{p}\right)\right)$\footnote{Here we ignore terms that are logarithmic in $S, A$, and $1/\delta$.} then with probability at least $1-p$ we have
$$
\norm{d_t - d_S}_2 \le \delta \ \forall t \ge (1-\mu)^{-1}\ln \left( {2}/{\delta(1-\gamma)}\right)
$$
\end{theorem}
\textbf{Proof Sketch}:
\begin{itemize}
    \item We first show that the empirical version of the Lagrangian $\hat{\calL}(d,h; M_t)$ is close to the true Lagrangian $\calL(d,h; M_t)$ with high probability. This is shown using the Chernoff-Hoeffding inequality and an $\epsilon$-net argument over the variables. Here we use the fact that for the dual variables we can just consider the space $\calH = \set{h: \norm{h}_2 \le {3S}/{(1-\gamma)^2}}$ as the optimal solution is guaranteed to exist in this space.
    \item We then show that an optimal saddle point of the empirical Lagrangian~\ref{eq:empirical-Lagrangian} is close to the optimal solution of the true Largrangian. In particular, if $\abs{\calL(d,h; M_t) - \hat{\calL}(d,h; \hat{M}_t)} \le \epsilon$ then we are guaranteed that $\norm{d_{t+1} - \GD(d_t)}_2 \le O(\epsilon)$. Here $\GD(d_t)$ denotes the solution obtained from optimizing the exact function.
    \item By choosing an appropriate number of samples, we can make the error term $\epsilon$ small enough, and establish the following recurrence relation:
    $
    \norm{d_{t+1} - d_S}_2 \le \beta \delta + \beta \norm{d_t - d_S}_2
    $ for $\beta < 1/2$.
    The rest of the proof follows the main idea of the proof of Theorem 3.10 from ~\cite{MPZH20}. If $\norm{d_t - d_S}_2 > \delta$ then we get a contraction mapping. On the other hand, if $\norm{d_t - d_S}_2 \le \delta$ then subsequent iterations cannot move $d_t$ outside of the $\delta$-neighborhood of $d_S$.
\end{itemize}

\subsection{Approximating the Unregularized Objective}\label{subsec:apx-unregularization}
Theorem~\eqref{thm:primal-convergence} shows that repeatedly optimizing objective~\eqref{eq:regularized-rl-discounting-occupancy} converges to a stable policy (say $d^\lambda_S$) with respect to the regularized objective~\eqref{eq:regularized-rl-discounting-occupancy}. Here we show that the solution $d^\lambda_S$ approximates the performatively stable and performatively optimal policy with respect to the  unregularized objective~\eqref{eq:perf-stable-policy}. 

\begin{theorem}\label{eq:apx-stable-point}
There exists a choice of the regularization parameter ($\lambda$) such that repeatedly optimizing objective~\eqref{eq:regularized-rl-discounting} converges to a policy ($d^\lambda_S$) with the following guarantee\footnote{$C(\tilde{d})$ denotes the set of occupancy measures that are feasible with respect to $\mathcal{P}(\tilde{d}) = P_{\tilde{d}}$.}
\begin{align*}
\sum_{s,a} r_{d^\lambda_S} (s,a) d^\lambda_S(s,a) \ge  \max_{d \in \mathcal{C}(d^\lambda_S)} \sum_{s,a} r_{d^\lambda_S} (s,a) d(s,a) - O\left({ S^{3/2}(\epsilon_r + S \epsilon_p)}/{(1-\gamma)^6} \right)
\end{align*}
\end{theorem}
Note that as $\epsilon = \max\set{\epsilon_r, \epsilon_p}$ converges to zero, the policy $d^\lambda_S$ also approaches a performatively stable solution with respect to the original unregularized objective.

\begin{theorem}[Informal Statement]\label{thm:apx-perf-optimal}
Let $d_{PO}$ be the performatively optimal solution with respect to the  unregularized objective and let $\eps = \max\set{\epsilon_r, \epsilon_p}$. Then there exists a value of $\lambda$ such that repeatedly optimizing objective ~\eqref{eq:regularized-rl-discounting} converges to a policy ($d^\lambda_S$) with the following guarantee
\begin{align*}
\sum_{s,a} r_{d^\lambda_S}(s,a) d^\lambda_{S}(s,a) \ge \sum_{s,a} r_{d_{PO}}(s,a) d_{PO}(s,a)  - O\left(\max\set{\frac{S^{5/3} A^{1/3} \epsilon^{2/3}}{(1-\gamma)^{14/3}}, \frac{\epsilon S}{(1-\gamma)^4} } \right)
\end{align*}
\end{theorem}

We again see that as $\epsilon$ converges to zero, $d^\lambda_S$ approaches a performatively optimal solution with respect to the original objective.
The proof of theorem~\eqref{thm:apx-perf-optimal} uses the following bound on the  distance between the performatively stable solution and the optimal solution.
$
\norm{d^\lambda_S - d^\lambda_{PO}}_2 \le O\left( \frac{S^{2} \sqrt{A}}{\lambda (1-\gamma)^4}\left( \epsilon_r \left( 1 + \gamma \sqrt{S}\right) +  \epsilon_p  \frac{\gamma {S} }{(1-\gamma)^2}\right) \right)
$

We believe that the bounds of theorems~\eqref{thm:apx-perf-optimal} and \eqref{eq:apx-stable-point} can be improved with more careful analysis. However, the error bound should grow as $\gamma$ decreases. This is because the diameter (or max $L_2$ norm) of occupancy measure is  most $1/(1-\gamma)^2$ and even in performative prediction such an approximation bound grows with the diameter of the model.\footnote{For example, see proposition E.1 ~\cite{PZMH20}, which is stated for diameter $1$ and convex loss function. }

%% file: 5_experiments.tex
\section{Experiments}\label{sec.experiments}

\captionsetup[figure]{aboveskip=1pt,belowskip=1pt}
\begin{figure*}[!t]
\captionsetup[subfigure]{aboveskip=-2pt,belowskip=0pt}
\centering
    \begin{minipage}[c]{0.32\textwidth}
     \captionsetup{type=figure}
        \caption{RPO $\lambda=1$}
        \label{fig : betas}
        \includegraphics[width=\linewidth]{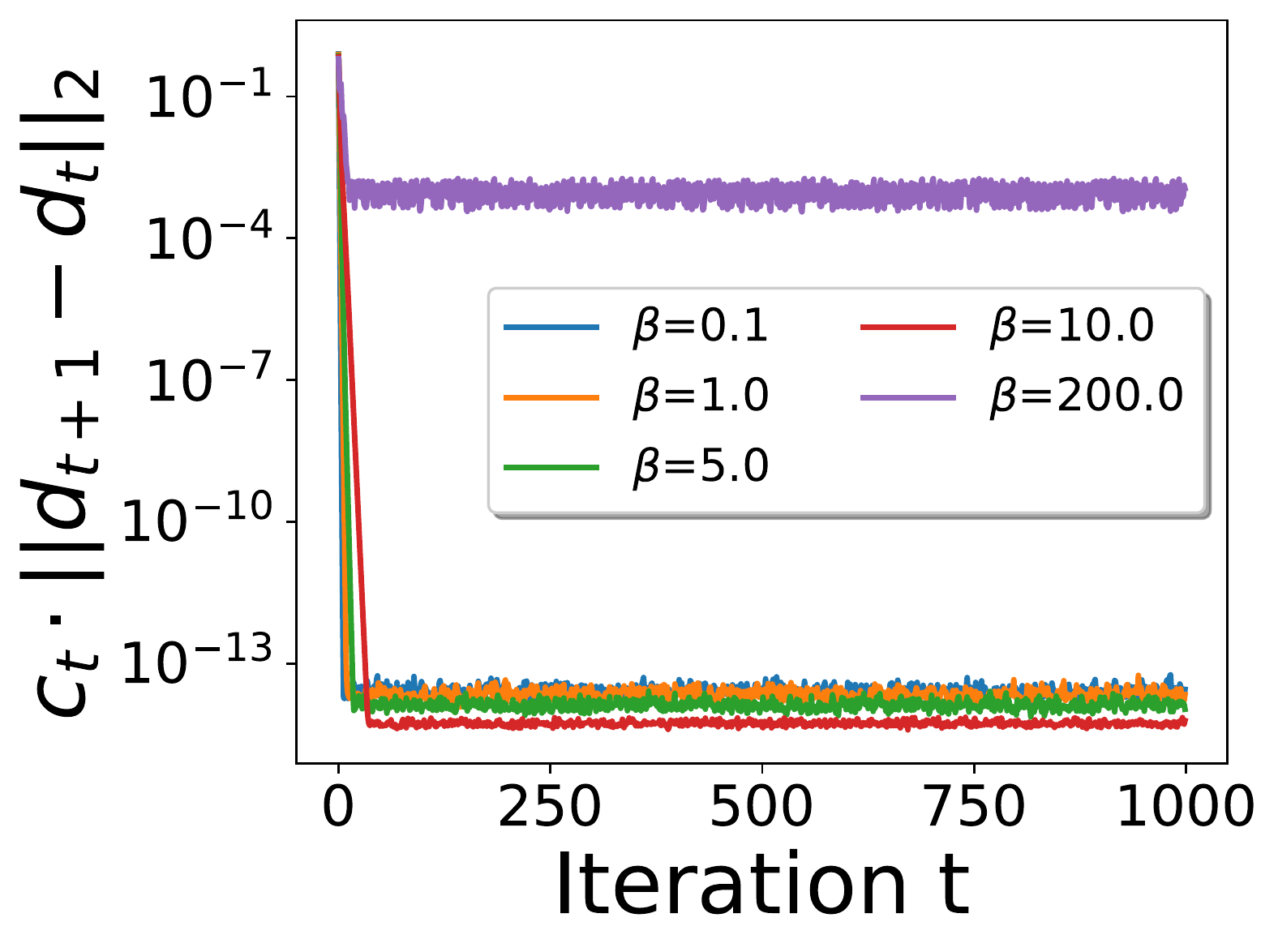}
    \end{minipage}\hfill\hfill%
    \begin{minipage}[c]{0.32\textwidth}
     \captionsetup{type=figure}
        \caption{RPO $\beta=10$}
        \label{fig : lambdas}
        \includegraphics[width=\textwidth]{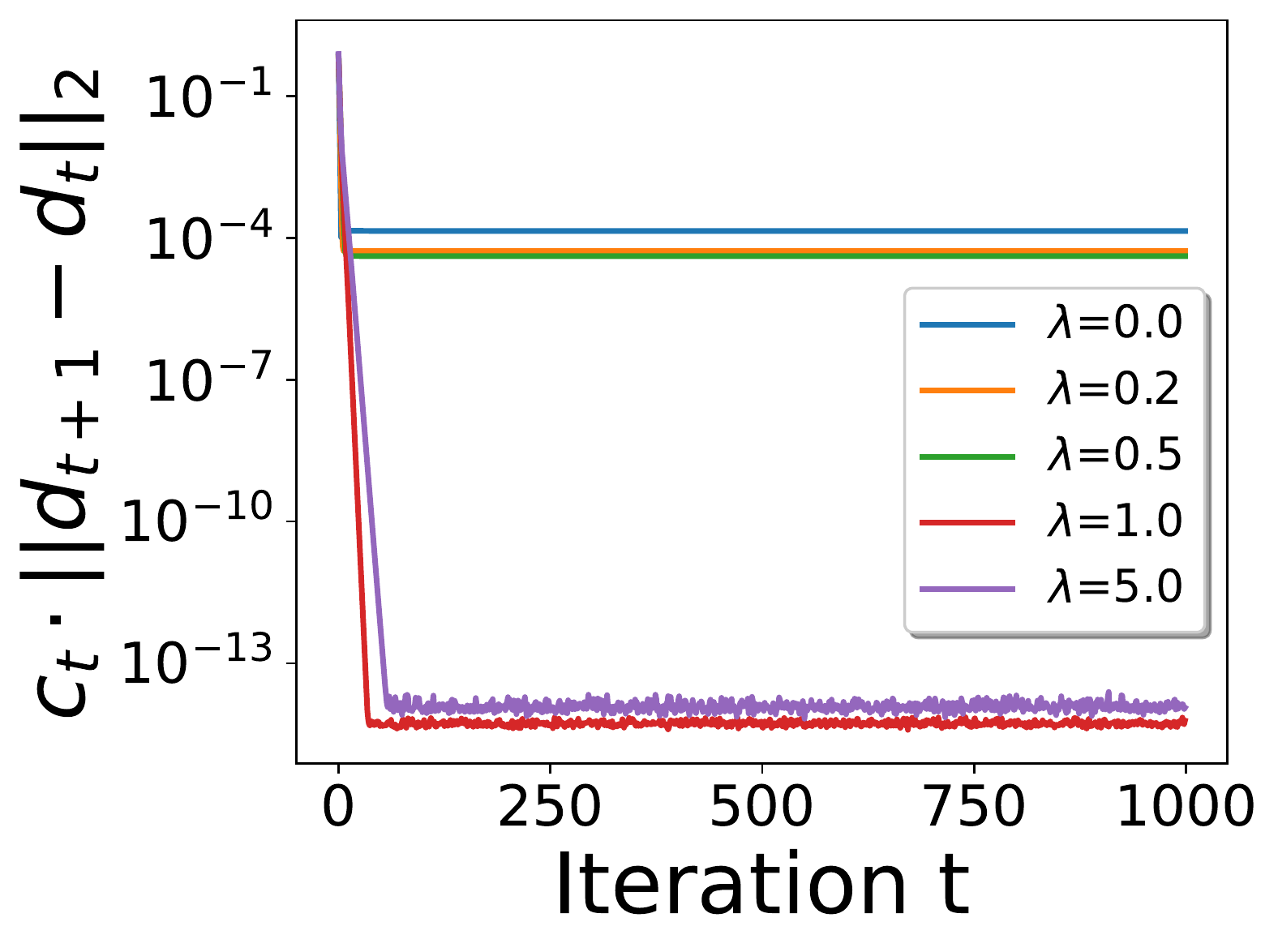}
    \end{minipage}\hfill%
    \begin{minipage}[c]{0.32\textwidth}
     \captionsetup{type=figure}
        \caption{RGA $\lambda=1$, $\beta=5$}
        \label{fig : etas_diff}
        \includegraphics[width=\textwidth]{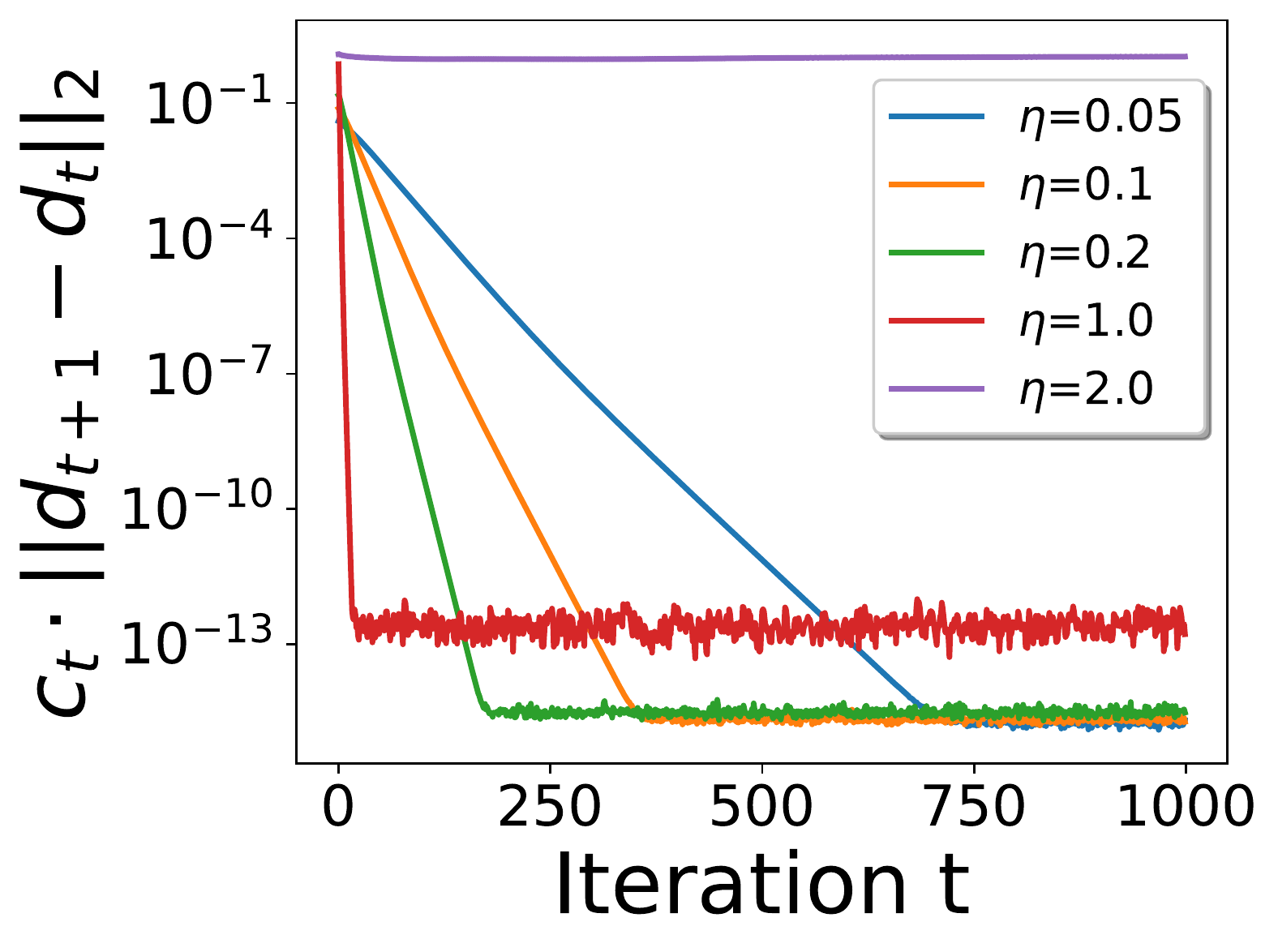}
    \end{minipage}\\
    \begin{minipage}[c]{0.32\textwidth}
      \captionsetup{type=figure}
        \caption{RGA Gap $\lambda=1$, $\beta=5$}
        \label{fig : etas_gap}
        \includegraphics[width=\textwidth]{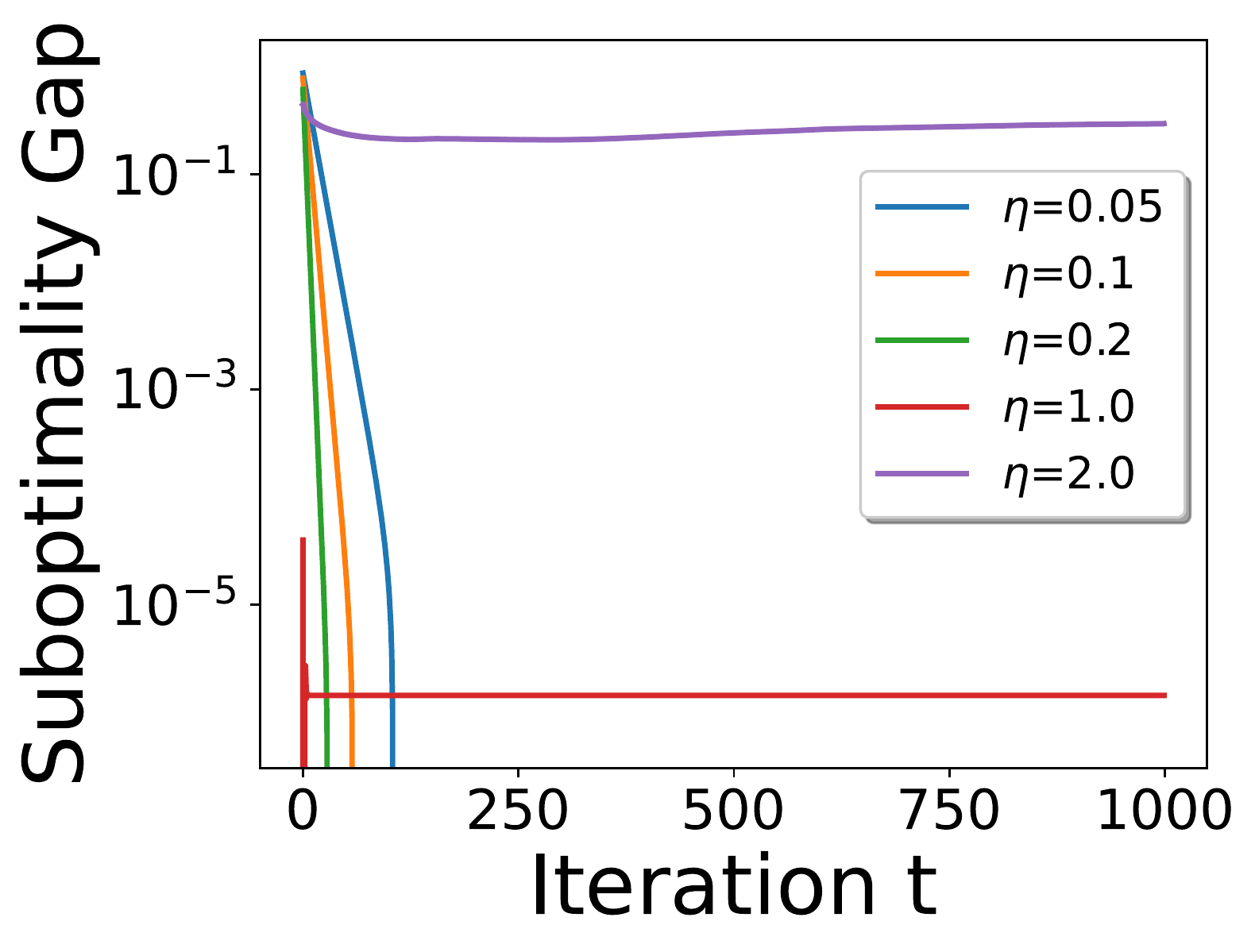}
    \end{minipage}\hfill%
    \begin{minipage}[c]{0.32\textwidth}
    \captionsetup{type=figure}
        \caption{ROL FS $\lambda=1$, $\beta=5$}
        \label{fig : m_rpo}
        \includegraphics[width=\textwidth]{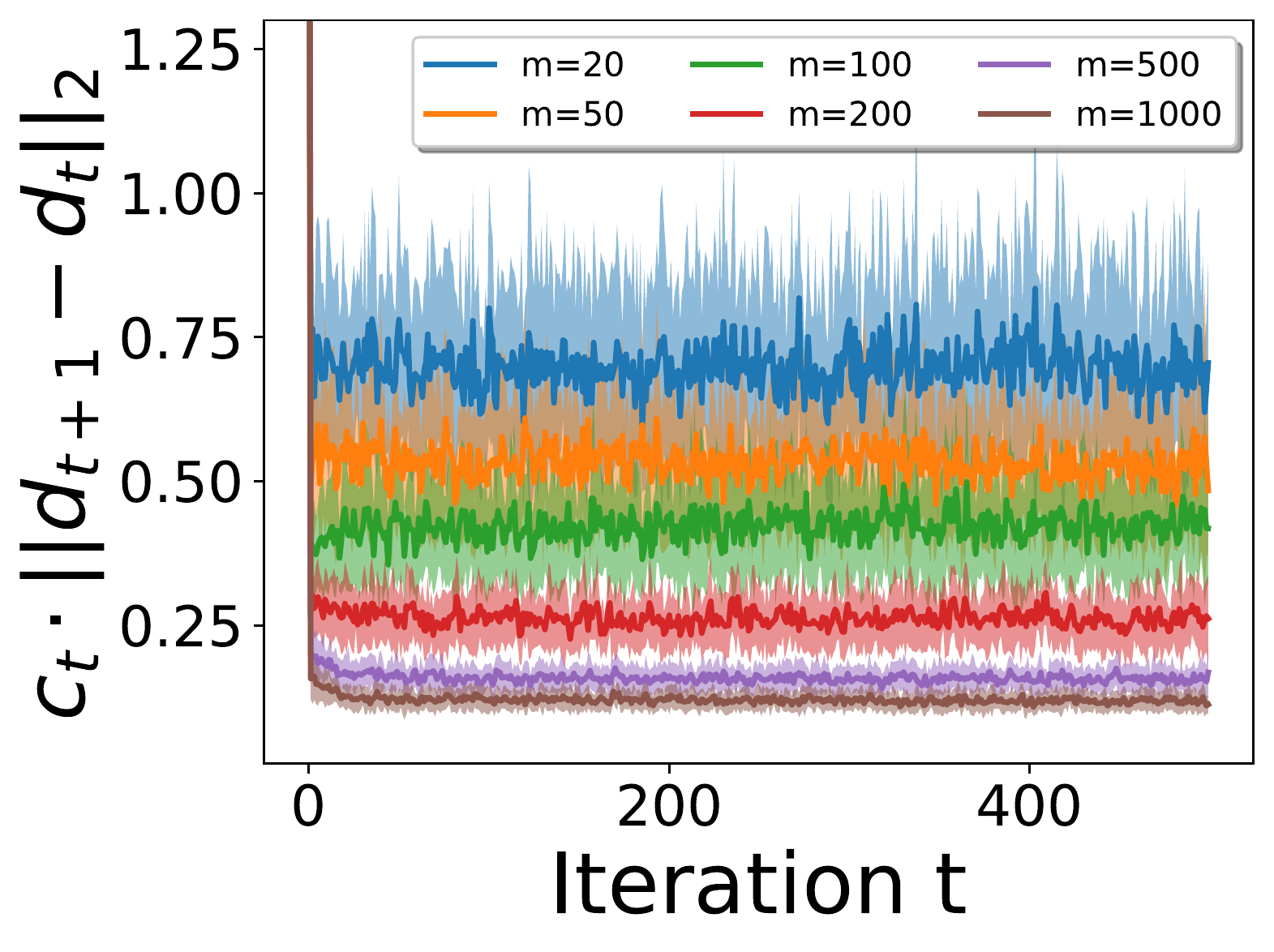}
    \end{minipage}\hfill%
    \begin{minipage}[c]{0.32\textwidth}
     \captionsetup{type=figure}
        \caption{RPO FS $\lambda=1$, $\beta=5$}
        \label{fig : m_ga}
        \includegraphics[width=\textwidth]{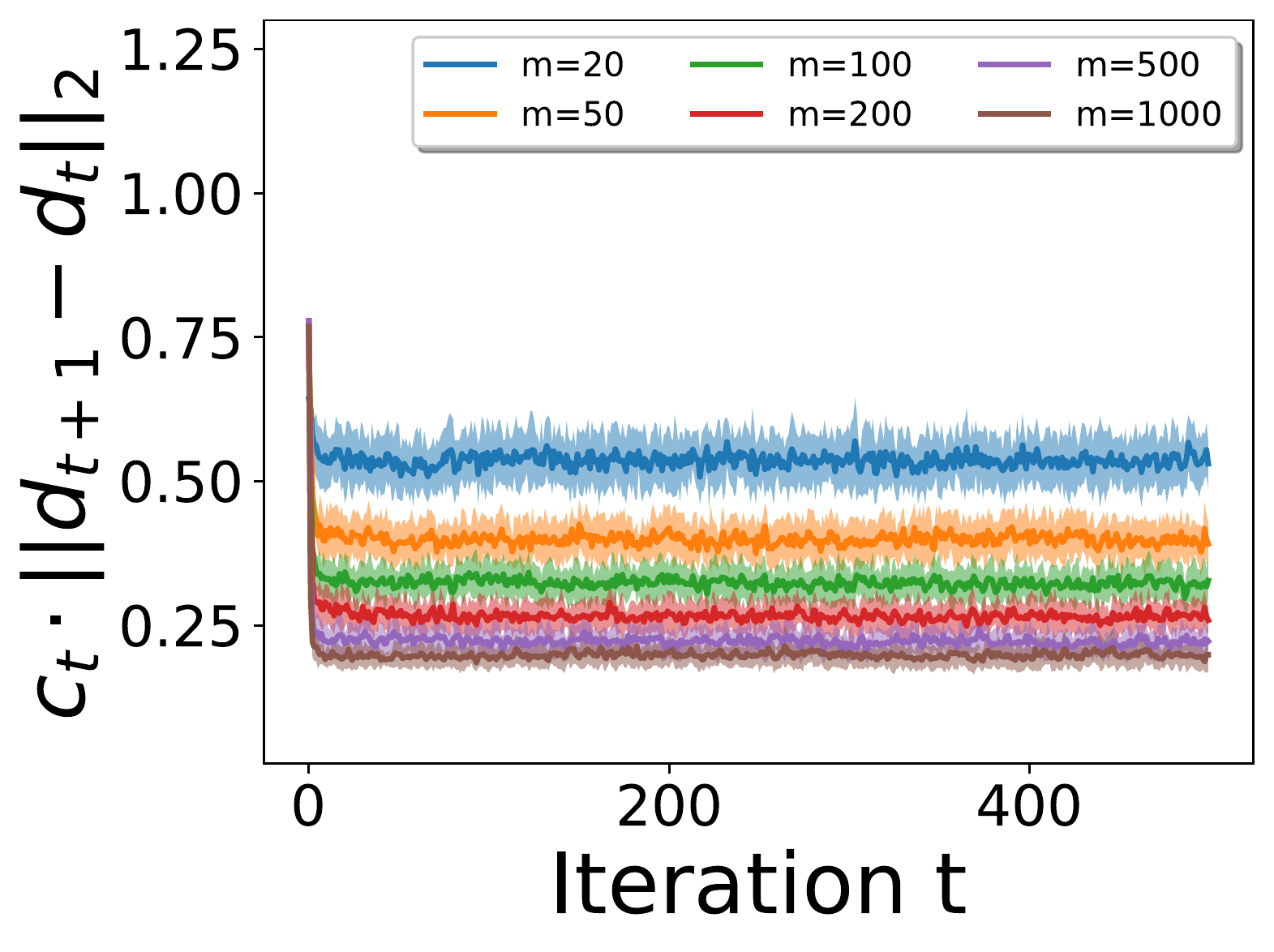}
    \end{minipage}
\caption{
Experimental results for \textit{Gridworld}. All plots were generated with $\gamma=0.9$ and $1000$ iterations.
We normalized the distance between iterations $t$ and $t+1$ with $c_t = \frac{1}{||d_t||_2}$. 
RPO stands for repeated policy optimization, RGA for repeated gradient ascent, ROL for repeatedly solving (empirical) Lagrangian and FS for finite samples. The parameters are $\lambda$ (regularization), $\beta$ (smoothness), $\eta$ (step-size), and $m$ (number of trajectories). 
} 
\label{fig : main_plots}
\end{figure*}

In this section, we experimentally evaluate the performance of various repeated retraining methods, and determine the effects of various parameters 
on convergence. All experiments are conducted on a  grid-world environment proposed by \cite{triantafyllou2021blame}.\footnote{The original single-agent version of this environment can be found in \cite{voloshin2019empirical}.}
We next describe  how this environment is adapted  for simulating performative reinforcement learning.
%


\begin{wrapfigure}{l}{3cm}
\captionsetup{type=figure}
\captionsetup{aboveskip=2pt,belowskip=1pt}
\includegraphics[width=0.2\textwidth]{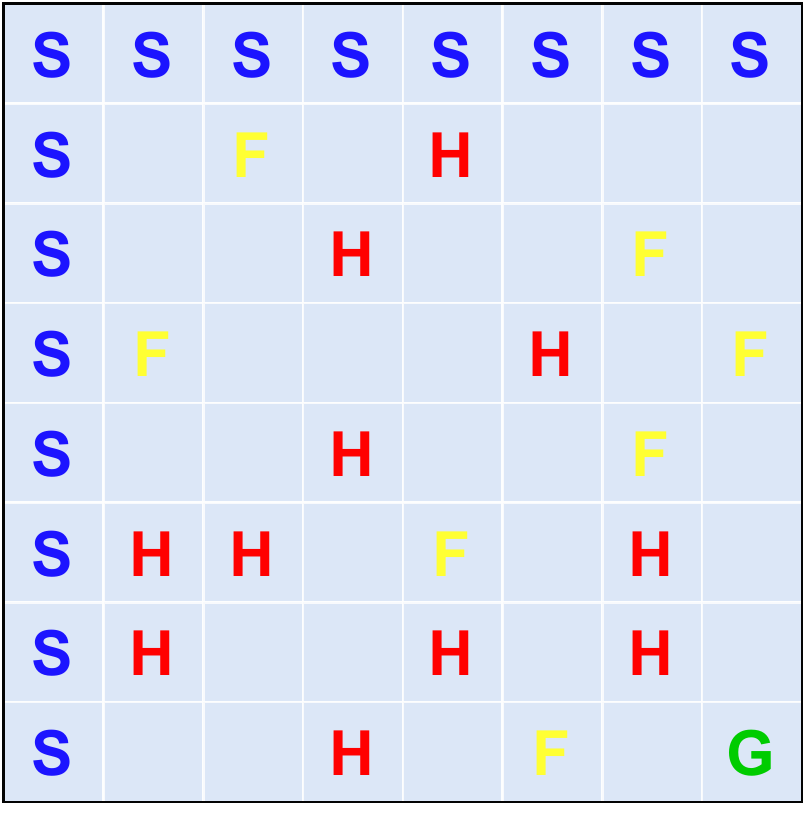}
\captionof{figure}{Gridworld}
\label{fig: gridworld}
\end{wrapfigure}
\textbf{Gridworld:} We consider the grid-world environment shown in Fig.~\ref{fig: gridworld},
in which $n+1$ agents share control over an actor. The agents' objective is to guide the actor from some initial state $S$ to the terminal state $G$, while minimizing their total discounted cost. We will call the first agent the principal, and the other $n$ agents the followers.
In each state, the agents select their actions simultaneously.
The principal agent, \agentone, chooses the direction of the actor's next
move by taking one of four actions (left, right, up, and down). 

Any other agent, $A_j$ decides to either intervene or not in \agentone's action. If the majority of the $n$ follower agents choose not to intervene, then the actor moves one cell towards the direction chosen by \agentone, otherwise it moves one cell towards the new direction chosen by the majority of the followers.
Note that the principal and the followers' policies determine a policy for the actor agent.

%

The cost at each state is defined according to the grid-world shown in Fig. \ref{fig: gridworld}. If the actor visits a blank or an $S$ cell, then all the agents  receive a small negative reward equal to $-0.01$. If an $F$ cell is visited, then a slightly more increased cost equal to $-0.02$ is imposed and for $H$ cells a severe penalty of
$-0.5$ is inflicted. 
Additionally, whenever any $A_j$ decides to intervene, it pays an additional cost of $-0.05$.

\textbf{Response Model:} We implement agent~\agentone~as a learner who performs repeated retraining. The initial policy of agent~\agentone~is an $\epsilon$-optimal single-agent policy ($\epsilon=0.1$) assuming that no other agent intervenes. Subsequently, agent~\agentone~performs one type of repeated retraining (e.g. gradient ascent).The follower agents, on the
other hand, always respond to the policy of \agentone~according to a response model. In particular, given a policy of \agentone~(say $\pi_1$), we first compute an optimal $Q$-value function for agent $A_j$, $Q^{*|\pi_1}_j$. Note that $Q^{*|\pi_1}_j$ is computed w.r.t. a perturbed grid-world, and which was generated by performing a random cell perturbation on the grid-world of Fig. \ref{fig: gridworld}. The perturbed grid-worlds are different for different agents. Then we compute an average $Q$-function defined as $\bar{Q}^{*|\pi_1} = \frac{1}{n} \sum_{j=2}^{n+1} \bar{Q}^{*|\pi_1}_j$.
Then a policy $\pi_2$ adopted by  the Boltzmann softmax operator with parameter $\beta$.
Then a policy $\pi_2$ is determined by the Boltzmann softmax operator with temperature parameter $\beta$,
$\pi_2(a_i|s) = \frac{e^{\beta \cdot \bar{Q}^{*|\pi_1}(s, a_i)}}{\sum_j e^{\beta \cdot \bar{Q}^{*|\pi_1}(s, a_j)}}$.
Note that the new policy $\pi_2$ effectively plays the role of a changing environment by responding to the majority of the $n$ follower agents. Additionally, parameter $\beta$ controls the smoothness of the changing environment, as viewed by \agentone.

\textbf{Repeated Policy Optimization:} We first consider the scenario where agent \agentone~gets complete knowledge of the updated reward and probability transition function at time $t$. In our setting, this means that \agentone~decides on $\pi_1^t$, all the other agents respond according to the softmax operator and jointly determines $\pi_2^t$, and then agent \agentone~observes the new policy $\pi_2^t$. This lets \agentone~to compute new probability transition function $P_t$, and reward function $r_t$, and solve optimization problem~\ref{eq:regularized-rl-discounting}. The solution is the new occupancy measure $d_1^{t+1}$ for \agentone, and \agentone~computes new policy $\pi^{t+1}_1$ for time $t+1$ by normalization using \cref{eq:dtopi}.
Plot \ref{fig : betas} shows the convergence results of the repeated policy optimization  for different values of  $\beta$, with $\lambda$ fixed to $1$. We see that if the response function of the environment (i.e., the policy of agent \agenttwo) is not smooth enough (e.g., for $\beta=200$), the algorithm fails to converge to a stable solution.
In Plot \ref{fig : lambdas} we fix $\beta$ to $10$ and vary the value of parameter $\lambda$ (strength of regularization). We can see that the algorithm converges only for large enough values of the regularization constant $\lambda$. Furthermore, we observe that the convergence is faster as $\lambda$ increases. 

\textbf{Repeated Gradient Ascent:} We now see what happens if agent \agentone~uses repeated gradient ascent instead of fully optimizing the objective each iteration. Here also \agentone~fully observes $\pi^t_2$ (hence $P_t$ and $r_t$) at time $t$. However, instead of full optimization, \agentone~performs a projected gradient ascent step~\eqref{eq:gradient-ascent-pp} to compute the next occupancy measure $d^{t+1}_1$. 
Plot \ref{fig : etas_diff} shows the performance of repeated gradient ascent for different values of the step-size $\eta$. We observe that when $\eta$ is small (e.g,. $\eta \le O(1/\lambda)$), the learner converges to a stable solution. But  the learner diverges for large $\eta$. Additionally, the convergence is faster for step-size closer to $1/\lambda$ (as suggested by Theorem~\ref{thm:rga-convergence}). 
Since, repeated gradient ascent does not fully solve the optimization problem~\ref{eq:regularized-rl-discounting}, we also plot the suboptimality gap of the current solution~\ref{fig : etas_gap}. This is measured as the difference between the objective value for the best feasible solution (w.r.t. $M_t$) and current solution ($d^t_1$), and then normalized by the former. As the step-size $\eta$ is varied, we see a trend  similar to Plot \ref{fig : etas_diff}.

\textbf{Effect of Finite Samples:} Finally, we investigate the scenario where \agentone~does not know $\pi_2^t$ at time $t$, and obtains samples from the new environment $M_t$ by deploying $\pi^t_1$. In our experiments, instead of sampling from the occupancy measure, \agentone~directly samples $m$ trajectories of fixed length (up to $100$) following policy $\pi^t_1$. 
We considered two approaches for obtaining the new policy $\pi^{t+1}_1$. First, \agentone~solves the empirical Lagrangian~\eqref{eq:empirical-Lagrangian} through an iterative method. In particular, we use an alternate optimization based method (algorithm~\eqref{alg:alternate-optimization}) where $h_n$ is updated through follow the regularized leader (FTRL) algorithm and $d_n$ is updated through best response. \footnote{Since the objective~\eqref{eq:empirical-Lagrangian} is linear in $h$ and concave in $d$, following standard arguments~\cite{FS96} it is straightforward to show that algorithm~\eqref{alg:alternate-optimization} finds an $\varepsilon$-approximate saddle point in $O(SAB/(1-\gamma)^2 \varepsilon^2)$ iterations.}

\begin{algorithm}[!h]
\caption{Alternating Optimization for the Empirical Lagrangian\label{alg:alternate-optimization}}
\begin{algorithmic}
    \STATE Set $H = \frac{3S}{(1-\gamma)^2}$, and $\calH = \set{h : \norm{h}_2 \le H}$.
    \FOR{$n=1,2,..,N$}
    {
        \STATE  $h_n = \argmin_{h \in \calH} \sum_{n'=1}^{n-1} \left\langle \nabla_h \hat{\calL}(d_{n'}, h; M_t), h \right \rangle + \beta \norm{h}_2^2$ 
        \STATE  $d_n = \arg\max_{d \ge 0,\ d(s,a) \le B \hat{d}_t(s,a) \forall s,a} \hat{\calL}(d, h_n; M_t)$ 
    }
    \ENDFOR
    \STATE Return $\bar{d} = \frac{1}{N} \sum_{n=1}^N d_n$.
    \end{algorithmic}
\end{algorithm}
%
%

Second,  \agentone~computes  estimates of probability transition function ($\widehat{P}_{t}$), and reward function ($\widehat{r}_t$), and solves \cref{eq:regularized-rl-discounting} with these estimates.
For both versions, we ran our experiments with $20$ different seeds, and  plots \ref{fig : m_rpo} and \ref{fig : m_ga} show the average errors along with the standard errors for the two approaches. For both settings, we observe that as $m$ increases, the algorithms eventually converge to a smaller neighborhood around the stable solution. However, for large number of samples ($m=500$ or $1000$), directly solving the Lagrangian (figure~\ref{fig : m_rpo}) gives improved result.

%% file: 6_conclusion.tex
\section{Conclusion}\label{sec.conclusion}

In this work, we introduce the framework of performative reinforcement learning and study under what conditions repeated retraining methods (e.g., policy optimization, gradient ascent) converges to a stable policy. In the future, it would be interesting to extend our framework to handle high dimensional state-space, and general function approximation. 
The main challenge with general function approximation is that a stable policy might not exist, so the first step would be to characterize under what conditions there is  a fixed point. Moreover, most RL algorithms with function approximation work in the policy space. So, another challenge lies in generalizing optimization problem~\ref{eq:regularized-rl-discounting} to handle representations of states and actions. 

Another interesting question is to resolve the hardness of finding stable policy with respect to the unregularized objective. To the best of our knowledge, this question is unresolved even for performative prediction with just convex loss function. It could be interesting to explore connections between our repeated optimization procedure and standard reinforcement learning methods, e.g., policy gradient methods~\cite{MBMG+16, NJG17}. However, we note that policy gradient methods typically work in the policy space, and might not even converge to a stable point under changing environments. Finally, for the finite samples setting, it would be interesting to use offline reinforcement learning algorithms~\cite{levine2020offline} for improving the speed of convergence to a stable policy.

%% file: 7.0_appendix_main.tex
\section{Additional Information on Experimental Setup and Results}

In this section, we provide additional information on the experimental setup (Section \ref{sec.app_exp_info}), as well as additional experimental results (Section \ref{sec.app_add_res}). We also provide information regarding the total amount of compute time and the type of resources that were used (Section \ref{sec.app_time_res}).

\subsection{Additional Information on Experimental Setup}\label{sec.app_exp_info}

\textbf{Gridworld:} The initial state of the actor in the grid-world is selected uniformly at random from the
cells denoted by $S$. Additionally, the actor remains at its current cell in case it attempts a move that would lead it outside the grid-world. Regarding the reward function, all the  agents receive a positive reward equal to $+1$ whenever the actor reaches the terminal state $G$.

\textbf{Response Model:} The response model of a follower agent is based on a perturbed grid-world  instead of the one in Fig. \ref{fig: gridworld}. In other words, each of the $n$ follower agents  sees different cell costs than the ones \agentone~sees. As a result, they might respond to the policy of \agentone, by adopting a policy that performs unnecessary or even harmful interventions w.r.t. the grid-world of Fig. \ref{fig: gridworld}.
A perturbed grid-world  is generated from the grid-world of Fig. \ref{fig: gridworld} with the following procedure. First,
%
$G$ and $S$ cells stay the same between the two grid-worlds. Then, any blank, $F$ or $H$ cell remains unchanged with probability $0.7$, and with probability $0.3$ we perturb its type to blank , $F$ or $H$ (the perturbation is done uniformly at random).

\subsection{Additional Experimental Results}\label{sec.app_add_res}

In this section, we provide additional insights on the interventional policies of the follower agents. The repeated retraining method we use in these experiments is the repeated policy optimization. More specifically, we present a visual representation of the limiting environment i.e. the majority of the agents' policy in the limit, i.e., after the method has converged to a stable solution. The configurations are set to $\lambda=1$, $\beta=5$, and we vary discount factor $\gamma$.

As mentioned in Section \ref{sec.experiments}, the policy of the follower agents can be thought of as a changing environment that responds to the policy updates of \agentone. To visualize how this environment looks like in the limit, we depict in Figure \ref{fig : limiting_envs} several limiting policies of the follower agents.  
From the Figures \ref{fig : gamma=.5_env},  and \ref{fig : gamma=.9_env} we observe that for smaller discount factor, the majority of the follower agents tend to intervene closer to the goal state. 
\captionsetup[figure]{belowskip=0pt}
\begin{figure}[!h]
\centering
    \begin{minipage}[c]{0.35\textwidth}
        \includegraphics[width=\textwidth]{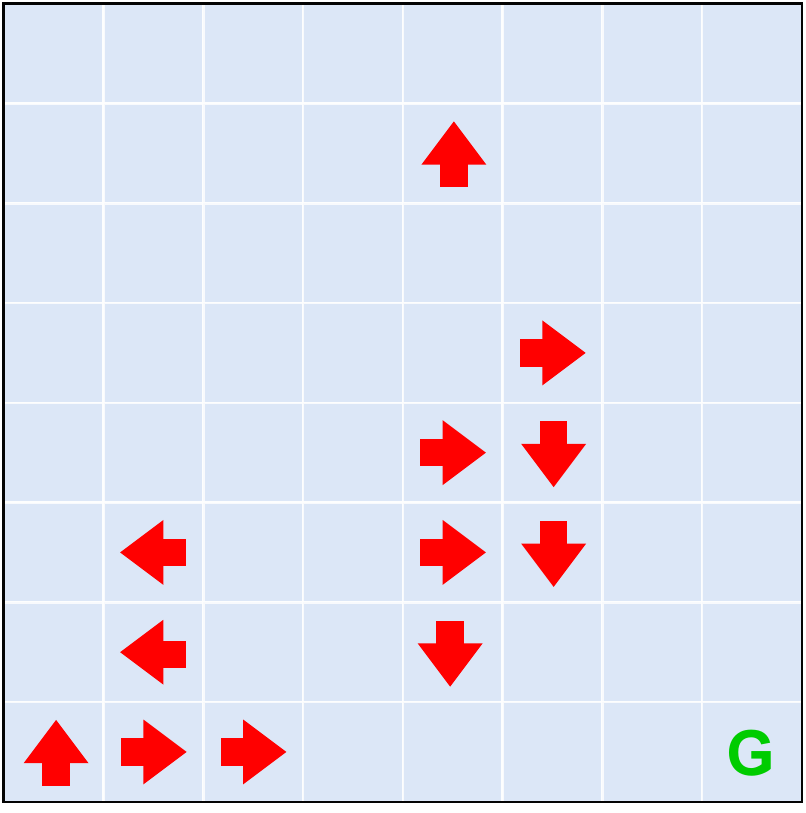}
        \captionsetup{type=figure}
        \caption{$\gamma=0.5$}
        \label{fig : gamma=.5_env}
    \end{minipage}\hspace{1cm}
    \begin{minipage}[c]{0.35\textwidth}
        \includegraphics[width=\textwidth]{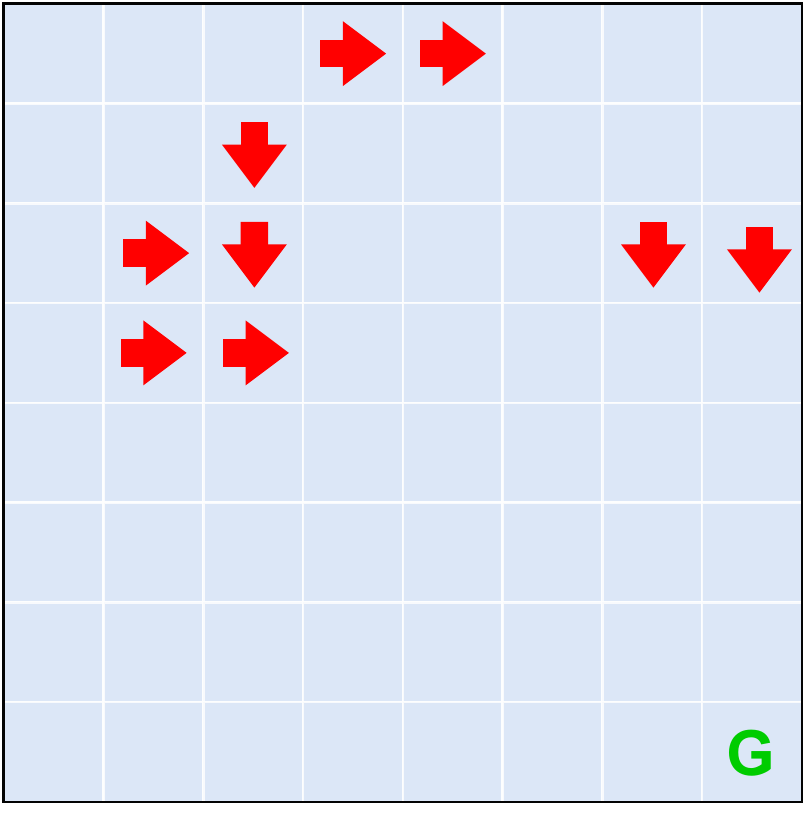}
        \captionsetup{type=figure}
        \caption{$\gamma=0.9$}
        \label{fig : gamma=.9_env}
    \end{minipage}
\caption{
Figures \ref{fig : gamma=.5_env},  and \ref{fig : gamma=.9_env} visualize two instances of the interventional policy of agent \agenttwo~in the \textit{Gridworld} environment. All figures correspond to the majority of the followers' policy   at convergence for various values of the discount factor $\gamma$. Empty cells denote states where majority of the agents' most probable action is to not intervene. For cells with a red arrow the (highest probability) action of the majority of the follower agents is to intervene by forcing the actor to move one cell towards the direction of the arrow.
} 
\label{fig : limiting_envs}
\end{figure}
\subsection{Total Amount of Compute and Type of Resources}\label{sec.app_time_res}

All experiments were conducted on a computer cluster with machines equipped with $2$ Intel Xeon E5-2667 v2 CPUs with 3.3GHz (16 cores) and 50 GB RAM. Table \ref{tab.times} reports the total computation times for our experiments (Section \ref{sec.experiments}). Note that at each iteration of the repeated gradient ascent experiment, apart from the gradient step a full solution of the optimization problem~\ref{eq:regularized-rl-discounting} was also computed, in order to report the suboptimality gap.

\captionsetup[table]{skip=5pt}
\renewcommand{\arraystretch}{2}
\begin{table}[!h]
    \begin{center}
            \begin{tabular}{|c|c|}
                \hline
                Repeated Policy Optimization & $767$ sec\\
                \hline
                Repeated Gradient Ascent & $964$ sec\\
                \hline
                Repeated Policy Optimization with Finite Samples & 33746 sec\\
                \hline
                Repeated Gradient Ascent with Finite Samples & 35396 sec\\
                \hline
            \end{tabular}
    \captionof{table}{Total computation times for the different experiments described in Section \ref{sec.experiments}.
    \label{tab.times}}
    \end{center}
\end{table}
\renewcommand{\arraystretch}{1}

\section{Missing Proofs}

\input{main_convergence}

\input{rga_convergence}

\input{finite_sample_convergence}

\subsection{Proof of Proposition~\ref{prop:existence}}

\begin{proof}
Let $\calD$ be the following set $\calD = \set{d \in \R^{S\times A} : d(s,a) \ge 0 \ \forall s,a \textrm{ and } \sum_{s,a} d(s,a) = \frac{1}{1-\gamma} }$. We define a set-valued function $\phi: \calD \rightarrow 2^{\calD}$ as follows.
\begin{align}
\label{eq:standard-LP}
\begin{split}
    \phi(d) = \argmax_{\tilde{d}  \ge 0} &\ \sum_{s,a} \tilde{d} (s,a) r_d(s,a)  \\
    \textrm{s.t.} & \sum_a \tilde{d} (s,a) = \rho(s) + \gamma \cdot \sum_{s',a} \tilde{d} (s',a) P_d(s',a,s)\ \forall s
\end{split}
\end{align}
Any fixed point of $\phi(\cdot)$ corresponds to a stable point. First,
note that $\phi(d)$ is non-empty as one can always choose $\tilde{d}$ to be the occupancy measure associated with any arbitrary policy $\pi$ in an MDP with probability transition function $P_d$. Now suppose $d_1, d_2 \in S(d)$. Then for any $\rho \in [0,1]$ it is easy to show that $\rho d_1 + (1-\rho) d_2 \in \phi(d)$. This is because all the constraints are linear, so $\rho d_1 + (1-\rho) d_2$ is feasible. Moreover, the objective is linear, so $\rho d_1 + (1-\rho) d_2$ also attains the same objective value.

We now show that the function $\phi$ is upper hemicontinuous. Let $L$ be the Lagrangian of the optimization problem~\eqref{eq:standard-LP}.
$$
L(\tilde{d},h; M_d) = \sum_{s,a} \tilde{d}(s,a) r_d(s,a) + \sum_s h(s) \left( \sum_a \tilde{d}(s,a) - \rho(s) - \gamma \cdot \sum_{s',a} \tilde{d}(s',a) P_d(s',a,s) \right)
$$
Note that the Lagrangian is continuous (in fact linear) in $\tilde{d}$, and continuous in $d$ (from the assumption of $(\epsilon_r, \epsilon_p)$-sensitivity). Finally, observe the alternative definition of the function $\phi$.
$$
\phi(d) = \argmax_{\tilde{d} \in \calD} \min_h L(\tilde{d}, h; M_d)
$$
Since the minimum of continuous functions is also continuous, and the set $\calD$ is compact, we can apply Berge's maximum theorem to conclude that the function $\phi(\cdot)$ is upper hemicontinuous. Now an application of Kakutani fixed point theorem~\cite{Glicksberg52} shows that $\phi$ has a fixed point.

\end{proof}

\subsection{Assumptions Regarding Quadratic Regularizer}
Throughout we performed repeated optimization with quadratic regularization. Our proof techniques can be easily generalized if we consider a strongly convex regularizer $\mathrm{R}(d)$. Suppose, at time $t$ we solve the following optimization problem.
\begin{align}\label{eq:general-regularized-rl}
\begin{split}
    \max_{\tilde{d}  \ge 0} &\ \sum_{s,a} \tilde{d} (s,a) r_t(s,a) - \mathrm{R}(\tilde{d}) \\
    \textrm{s.t.} & \sum_a \tilde{d} (s,a) = \rho(s) + \gamma \cdot \sum_{s',a} \tilde{d} (s',a) P_t(s',a,s)\ \forall s
\end{split}
\end{align}
Since $\mathrm{R}$ is strongly convex, $(\mathrm{R}')^{-1}$ exists and we can use this result to show that the dual of the \eqref{eq:general-regularized-rl} is strongly convex. In fact, as in the proof of theorem~\ref{thm:primal-convergence} we can write down the lagrangian $\calL(d,h)$ and at an optimal solution we must have $\nabla_d \calL(d,h) = 0$. This gives the following expression.
\begin{equation}\label{eq:generalized-dtoh}
d(s,a) = (\mathrm{R}')^{-1}\left(r_t(s,a) - h(s) + \gamma \cdot \sum_{\tilde{s}} h(\tilde{s}) P_t(s,a,\tilde{s}) \right)
\end{equation}
We can use the result above to show that the dual is strongly convex and  the optimal dual solutions form a contraction. Then we can translate this guarantee back to the primal using \eqref{eq:generalized-dtoh}. 

{\color{black} 

\subsection{Omitted Proofs from Subsection~\ref{subsec:apx-unregularization}}

We will write $d^\lambda_{PO}$ to write the performatively optimal solution when using the regularization parameter $\lambda$ i.e.

\begin{align}
    \begin{split}
        d^\lambda_{PO} \in \argmax_{d\ge 0}\ &  \sum_{s,a} d(s,a) {r^{\lambda}_{PO}(s,a)} - \frac{\lambda}{2}\norm{d}_2^2\\
       \textrm{s.t. } & \sum_a d(s,a) = \rho(s) + \gamma \cdot \sum_{s',a} d(s',a) {P^\lambda_{PO}(s',a,s)}\ \forall s\\
    \end{split}
\end{align}

Here we write $r^\lambda_{PO} = \mathcal{R}(d^\lambda_{PO})$ and $P^\lambda_{PO} = \mathcal{P}(d^\lambda_{PO})$ to denote the reward function and probability transition function in response to the optimal occupancy measure $d^\lambda_{PO}$. 
We will also write $d^\lambda_S$ to denote the performatively stable solution and $r^\lambda_S$ (resp. $P^\lambda_S$) to denote the corresponding reward (resp. probability transition) function. The next lemma bounds the distance between $d^\lambda_S$ and $d^\lambda_{PO}$.

\subsubsection{Proof of Theorem~\ref{eq:apx-stable-point}}
\begin{proof}
Suppose repeatedly maximizing the regularized objective converges to a stable solution $d^\lambda_S$ i.e.
$$
\sum_{s,a} r_{d^\lambda_S} (s,a) d^\lambda_S(s,a) - \frac{\lambda}{2} \norm{d^\lambda_S}_2^2 \ge \max_{d \in \mathcal{C}(d^\lambda_S) } \sum_{s,a} r_{d^\lambda_S} (s,a) d(s,a) - \frac{\lambda}{2} \norm{d}_2^2 
$$
Therefore,
\begin{align*}
    \sum_{s,a} r_{d^\lambda_S} (s,a) d^\lambda_S(s,a) &\ge \max_{d \in \mathcal{C}(d^\lambda_S) } \sum_{s,a} r_{d^\lambda_S} (s,a) d(s,a) - \frac{\lambda}{2} \norm{d}_2^2 \\
    &\ge \max_{d \in \mathcal{C}(d^\lambda_S) } \sum_{s,a} r_{d^\lambda_S} (s,a) d(s,a) - \frac{\lambda}{2(1-\gamma)^2} 
\end{align*}
The last inequality uses $\norm{d}_2^2 = \sum_{s,a} d(s,a)^2 = (1-\gamma)^{-2} \sum_{s,a} \left((1-\gamma) d(s,a)\right)^2 \le (1-\gamma)^{-2} \sum_{s,a} (1-\gamma) d(s,a) = (1-\gamma)^{-2}$. Now we substitute $\lambda = \frac{12 S^{3/2}(2 \epsilon_r + 5 S \epsilon_p)}{(1-\gamma)^4}$ from theorem~\ref{thm:primal-convergence} and get the following bound.
\begin{align*}
    \sum_{s,a} r_{d^\lambda_S} (s,a) d^\lambda_S(s,a) &\ge  \max_{d \in \mathcal{C}(d^\lambda_S)}  \sum_{s,a} r_{d^\lambda_S} (s,a) d(s,a) - \frac{6 S^{3/2}(2 \epsilon_r + 5 S \epsilon_p)}{(1-\gamma)^6}
\end{align*}
\end{proof}

\subsubsection{Formal Statement and Proof of Theorem~\ref{thm:apx-perf-optimal}}
\begin{theorem}
Let $d_{PO}$ be the performatively optimal solution with respect to the original (unregularized) objective. Then there exists a choice of regularization parameter ($\lambda$) such that repeatedly optimizing objective ~\eqref{eq:dual-discounted} converges to a policy ($d^\lambda_S$) with the following guarantee
$$
\sum_{s,a} r_{d^\lambda_S}(s,a) d^\lambda_{S}(s,a) \ge \sum_{s,a} r_{d_{PO}}(s,a) d_{PO}(s,a) - \Delta$$
where
$$\Delta = O\left(\max \set{\frac{SA^{1/3}}{(1-\gamma)^{10/3}} \left( (1+\gamma \sqrt{S}) \epsilon_r + \frac{\gamma S \epsilon_p}{(1-\gamma)^2} \right)^{2/3} , \frac{\epsilon_r}{(1-\gamma)^2} + \frac{\epsilon_p S}{(1-\gamma)^4}} \right)
$$
\end{theorem}
\begin{proof}
Let us  write $h^\lambda_{PO}$ to denote the dual optimal solution i.e. 
$$
h^\lambda_{PO} \in \argmin_h \calL_d(h; M^\lambda_{PO})
$$
Moreover, let $h^\lambda_{S}$ be the  dual optimal solution corresponding to the stable solution $d^\lambda_S$. 
\begin{align*}
    &\sum_{s,a} r_{d_{PO}}(s,a) d_{PO}(s,a) - \sum_{s,a} r_{d^\lambda_S}(s,a) d^\lambda_{S}(s,a) \\
    =& \left(\sum_{s,a} r_{d_{PO}}(s,a) d_{PO}(s,a) - \frac{\lambda}{2} \norm{d_{PO}}_2^2 \right) + \frac{\lambda}{2} \norm{d_{PO}}_2^2 \\ -& \left(\sum_{s,a} r_{d^\lambda_S}(s,a) d^\lambda_{S}(s,a) - \frac{\lambda}{2} \norm{d^\lambda_S}_2^2 \right) - \frac{\lambda}{2} \norm{d^\lambda_S}_2^2\\
    \le & \left(\sum_{s,a} r_{d^\lambda_{PO}}(s,a) d^\lambda_{PO}(s,a) - \frac{\lambda}{2} \norm{d^\lambda_{PO}}_2^2 \right) + \frac{\lambda}{2} \norm{d_{PO}}_2^2 \\ -& \left(\sum_{s,a} r_{d^\lambda_S}(s,a) d^\lambda_{S}(s,a) - \frac{\lambda}{2} \norm{d^\lambda_S}_2^2 \right) - \frac{\lambda}{2} \norm{d^\lambda_S}_2^2\\
    \le & \calL_d(h^\lambda_{PO}; M^\lambda_{PO})  - \calL_d(h^\lambda_S; M^\lambda_S) + \frac{\lambda}{2} \norm{d_{PO}}_2^2
\end{align*}
The first inequality uses the fact that $d^\lambda_{PO}$ is the performatively optimal solution with regularization parameter $\lambda$. The second inequality uses strong duality and expresses the objective in terms of optimal dual variables. We now bound the difference $\calL_d(h^\lambda_{PO}; M^\lambda_{PO})  - \calL_d(h^\lambda_S; M^\lambda_S)$.
\begin{align*}
     &\calL_d(h^\lambda_{PO}; M^\lambda_{PO})  - \calL_d(h^\lambda_S; M^\lambda_S)\\
    = &\ \calL_d(h^\lambda_{PO}; M^\lambda_{PO})  - \calL_d(h^\lambda_S; M^\lambda_{PO}) + \calL_d(h^\lambda_S; M^\lambda_{PO}) - \calL_d(h^\lambda_S; M^\lambda_S) \\
    \le &\ \calL_d(h^\lambda_S; M^\lambda_{PO}) - \calL_d(h^\lambda_S; M^\lambda_S) \quad \textrm{[Since $h^\lambda_{PO}$ minimizes $\calL_d(\cdot; M^\lambda_{PO})$]}\\
    \le &\ \frac{\norm{h^\lambda_S}_2 \sqrt{A}}{\lambda} \left( (1+\gamma \sqrt{S}) \epsilon_r + \gamma(2\sqrt{S} + \norm{h^\lambda_S}_2) \epsilon_p\right) \norm{d^\lambda_S - d^\lambda_{PO}}_2 \ \textrm{[By inequality~\ref{eq:upper-bound-diff-perf} ]}\\
    \le &\ \frac{S^3 A}{\lambda^2 (1-\gamma)^6} \left( (1+\gamma \sqrt{S}) \epsilon_r + \gamma \left(2\sqrt{S} + \frac{3S}{(1-\gamma)^2}\right) \epsilon_p\right)^2 \ \textrm{[By lemma~\ref{lem:deviation-between-stability-optimality}]}
\end{align*}
The term $\norm{d_{PO}}_2^2$ can be bounded as $\sum_{s,a} d^2_{PO}(s,a) = (1-\gamma)^{-2} \sum_{s,a} (d_{PO}(s,a) (1-\gamma))^2 \le (1-\gamma)^{-2} \sum_{s,a} d_{PO}(s,a)  (1-\gamma) = (1-\gamma)^{-2} $. This gives us the following bound.

\begin{align*}
    &\sum_{s,a} r_{d_{PO}}(s,a) d_{PO}(s,a) - \sum_{s,a} r_{d^\lambda_S}(s,a) d^\lambda_{S}(s,a) \\
    \le &\ \frac{S^3 A}{\lambda^2 (1-\gamma)^6} \left( (1+\gamma \sqrt{S}) \epsilon_r + \gamma \left(2\sqrt{S} + \frac{3S}{(1-\gamma)^2}\right) \epsilon_p\right)^2 + \frac{\lambda}{2(1-\gamma)^2}\\
    = &\ \frac{1}{\lambda^2} \underbrace{\frac{S^3 A}{ (1-\gamma)^6} \left( (1+\gamma \sqrt{S}) \epsilon_r + \gamma \left(2\sqrt{S} + \frac{3S}{(1-\gamma)^2}\right) \epsilon_p\right)^2}_{:= T_1} + \lambda \underbrace{\frac{1}{2(1-\gamma)^2} }_{:= T_2}
\end{align*}
Note that in order to apply lemma~\ref{lem:deviation-between-stability-optimality} we need $\lambda \ge \lambda_0 = 2 \left(2 \epsilon_r + 9 \epsilon_p S (1-\gamma)^{-2} \right)$. So we consider two cases. First if $(2 T_1 / T_2)^{-1/3} > \lambda_0$. In that case, we can use $\lambda = (2 T_1 / T_2)^{-1/3}$ and get the following upper bound.
\begin{align*}
&\sum_{s,a} r_{d_{PO}}(s,a) d_{PO}(s,a) - \sum_{s,a} r_{d^\lambda_S}(s,a) d^\lambda_{S}(s,a) \le O\left( T_1^{1/3} T_2^{2/3}\right)\\
=&\ O\left(\frac{SA^{1/3}}{(1-\gamma)^{10/3}} \left( (1+\gamma \sqrt{S}) \epsilon_r + \gamma \left(2\sqrt{S} + \frac{3S}{(1-\gamma)^2}\right) \epsilon_p\right)^{2/3} \right)
\end{align*}
On the other hand, if $(2 T_1 / T_2)^{-1/3} \le \lambda_0$ then we can substitute $\lambda = \lambda_0$ and get  the following bound.
\begin{align*}
    &\sum_{s,a} r_{d_{PO}}(s,a) d_{PO}(s,a) - \sum_{s,a} r_{d^\lambda_S}(s,a) d^\lambda_{S}(s,a) \le \frac{1}{\lambda_0^2} T_1 + \lambda_0 T_2\\
    = &\frac{1}{(T_1^{1/3} T_2^{-1/3})^2} + \lambda_0 T_2 = T_1^{1/3}T_2^{2/3} + \lambda_0 T_2 \le 2 \lambda_0 T_2\\
    =&\ O\left( \frac{\epsilon_r}{(1-\gamma)^2} + \frac{\epsilon_p S}{(1-\gamma)^4}\right)
\end{align*}
\end{proof}

\begin{lemma} \label{lem:deviation-between-stability-optimality}
Suppose $\lambda \ge 2 \left(2\epsilon_r + \frac{9\epsilon_p S }{ (1-\gamma)^2}\right)$. Then we have
$$
\norm{d^\lambda_S - d^\lambda_{PO}}_2 \le O\left( \frac{S^{2} \sqrt{A}}{\lambda (1-\gamma)^4}\left( \epsilon_r \left( 1 + \gamma \sqrt{S}\right) +  \epsilon_p  \frac{\gamma {S} }{(1-\gamma)^2}\right) \right)
$$
\end{lemma}
\begin{proof}
Let us  write $h^\lambda_{PO}$ to denote the dual optimal solution i.e. 
$$
h^\lambda_{PO} \in \argmin_h \calL_d(h; M^\lambda_{PO})
$$
Moreover, let $h^\lambda_{S}$ be the  dual optimal solution corresponding to the stable solution $d^\lambda_S$. 

Since the dual $\calL_d(\cdot; M^\lambda_{PO})$ objective is strongly convex (lemma~\ref{lem:strong-convexity}) and $h^\lambda_{PO}$ is the corresponding optimal solution we have,
\begin{equation}\label{eq:dual-deviation-lambda}
\calL_d(h^\lambda_{S}; M^\lambda_{PO}) - \calL_d(h^\lambda_{PO}; M^\lambda_{PO}) \ge \frac{A(1-\gamma)^2}{2\lambda} \norm{h^\lambda_S - h^\lambda_{PO}}_2^2
\end{equation}
From lemma~\eqref{lem:deviation-d-to-h} we get the following bound.
\begin{align*}
   \left( 1 - \frac{2\epsilon_r + 3\epsilon_p \norm{{h}^\lambda_{S}}_2}{\lambda}\right) \norm{d^\lambda_S -{d}^\lambda_{PO}}_2 \le  \frac{3\sqrt{AS}}{\lambda} \norm{h^\lambda_S - {h}^\lambda_{PO}}_2
\end{align*}
Substituting the above bound in \cref{eq:dual-deviation-lambda} and using lemma~\ref{lem:bound-optimal-dual} we get the following inequality for any $\lambda > 2 \epsilon_r + 9\epsilon_r S / (1-\gamma)^2$.
\begin{align}
    \calL_d(h^\lambda_{S}; M^\lambda_{PO}) - \calL_d(h^\lambda_{PO}; M^\lambda_{PO}) \ge \frac{(1-\gamma)^2}{18S}\left( 1 - \frac{2 \epsilon_r + 3\epsilon_p \norm{h^\lambda_S}_2    }{\lambda}\right)^2 \norm{d^\lambda_S - d^\lambda_{PO} }_2^2 \label{eq:lower-bound-diff-perf}
\end{align}

We now upper bound $\calL_d(h^\lambda_{S}; M^\lambda_{PO}) - \calL_d(h^\lambda_{S}; M^\lambda_{S})$. Using lemma~\ref{lem:ubd-model-deviation} we get the following bound.
\begin{align}
    \calL_d(h^\lambda_{S}; M^\lambda_{PO}) - \calL_d(h^\lambda_{S}; M^\lambda_{S}) &\le \frac{\norm{h^\lambda_S}_2 \sqrt{A}}{\lambda} \left( (1+\gamma \sqrt{S}) \norm{r^\lambda_S - r^\lambda_{PO}}_2 + \gamma(2\sqrt{S} + \norm{h^\lambda_S}_2) \norm{P^\lambda_S - P^\lambda_{PO}}_2\right)\nonumber \\
    &\le \frac{\norm{h^\lambda_S}_2 \sqrt{A}}{\lambda} \left( (1+\gamma \sqrt{S}) \epsilon_r + \gamma(2\sqrt{S} + \norm{h^\lambda_S}_2) \epsilon_p\right) \norm{d^\lambda_S - d^\lambda_{PO}}_2 \label{eq:upper-bound-diff-perf}
\end{align}
Note that the following sequence of inequalities hold.
$$
\calL_d(h^\lambda_S; M^\lambda_{PO}) \ge \calL_d(h^\lambda_{PO}; M^\lambda_{PO}) \ge \calL_d(h^\lambda_S; M^\lambda_S)
$$
The first inequality is true because $h^\lambda_{PO}$ minimizes $\calL_d(\cdot; M^\lambda_{PO})$. The second inequality holds because the primal objective at performative optimal solution ($d^\lambda_{PO}$) upper bound the primal objective at performatively stable solution ($d^\lambda_S$) and by strong duality the primal objectives are equal to the corresponding dual objectives.
Therefore we must have $\calL_d(h^\lambda_{S}; M^\lambda_{PO}) - \calL_d(h^\lambda_{S}; M^\lambda_{S}) \ge  \calL_d(h^\lambda_{S}; M^\lambda_{PO}) - \calL_d(h^\lambda_{PO}; M^\lambda_{PO})$, and by using equations \eqref{eq:upper-bound-diff-perf} and \eqref{eq:lower-bound-diff-perf} we get the following bound on $\norm{d^\lambda_S - d^\lambda_{PO}}_2$.
\begin{align*}
    \norm{d^\lambda_S - d^\lambda_{PO}}_2 &\le \frac{\norm{h^\lambda_S}_2 \sqrt{A}}{\lambda} \left( (1+\gamma \sqrt{S}) \epsilon_r + \gamma(2\sqrt{S} + \norm{h^\lambda_S}_2) \epsilon_p\right) \frac{18S}{(1-\gamma)^2}\left( 1 - \frac{2 \epsilon_r + 3\epsilon_p \norm{h^\lambda_S}_2    }{\lambda}\right)^{-2}\\
    &\le \frac{108 S^2 \sqrt{A} \lambda}{(1-\gamma)^4}\left( \epsilon_r (1+\gamma \sqrt{S})  + \gamma \sqrt{S} \epsilon_p \left(2 + \frac{3\sqrt{S}}{(1-\gamma)^2}\right) \epsilon_p\right)
\end{align*}
The last line uses lemma~\ref{lem:bound-optimal-dual} and $\lambda \ge 2 (2\epsilon_r + 9\epsilon_p S / (1-\gamma)^2)$

\end{proof}



\begin{lemma}\label{lem:ubd-model-deviation}
$\abs{\calL_d(h; M) - \calL_d(h; \widehat{M})} \le \frac{\norm{h}_2 \sqrt{A}}{\lambda} \left( (1+\gamma \sqrt{S}) \norm{r - \hat{r}}_2 + \gamma(2\sqrt{S} + \norm{h}_2) \norm{P - \widehat{P}}_2\right)$
\end{lemma}
\begin{proof}
From the definition of the dual objective~\ref{eq:dual-discounted} we have,
\begin{align*}
    &\abs{\calL_d(h; M) - \calL_d(h; \widehat{M})} \le \frac{1}{\lambda}\underbrace{\abs{\sum_{s,a}h(s) (r(s,a) - \hat{r}(s,a))} }_{:= T_1}\\
    &+ \frac{\gamma}{\lambda} \underbrace{\abs{\sum_{s,s',a} h(s) \left( r(s',a) P(s',a,s) - \hat{r}(s',a) \hat{P}(s',a,s)\right)}}_{:= T_2}\\
    &+ \frac{\gamma}{\lambda} \underbrace{\abs{\sum_{s,a} h(s) \sum_{\tilde{s}}h(\tilde{s}) \left( P(s,a,\tilde{s}) - \hat{P}(s,a,\tilde{s})\right)}}_{:= T_3}\\
    &+ \frac{\gamma^2}{2\lambda} \underbrace{\abs{\sum_{s,a} \sum_{\tilde{s}, \hat{s}} h(\tilde{s}) h(\hat{s}) \left(P(s,a,\hat{s}) P(s,a,\tilde{s}) - \widehat{P}(s,a,\hat{s}) \widehat{P}(s,a,\tilde{s}) \right)}}_{:= T_4}
\end{align*}
\begin{align*}
    T_1 &\le \sqrt{\sum_s h(s)^2} \sqrt{\sum_s \left(\sum_a (r(s,a) - \hat{r}(s,a))\right)^2} \le \norm{h}_2 \sqrt{A} \norm{r - \hat{r}}_2
\end{align*}
\begin{align*}
    T_2 &\le \abs{\sum_{s,s',a} h(s) P(s',a,s) \left(r(s',a) - \hat{r}(s',a) \right)} + \abs{\sum_{s,s',a} h(s) \hat{r}(s',a) \left(P(s',a,s) - \hat{P}(s',a,s) \right)}  \\
    &\le \sqrt{\sum_s h(s)^2} \sqrt{\sum_s \left(  \sum_{s',a}P(s',a,s)\left( r(s',a) - \hat{r}(s',a)\right)\right)^2 } \\&+ \sqrt{\sum_s h(s)^2}\sqrt{\sum_s \left( \sum_{s',a} \left( P(s',a,s) - \hat{P}(s',a,s)\right)\right)^2}\\
    &\le \norm{h}_2 \sqrt{SA} \sqrt{\sum_s \sum_{s',a} P(s',a,s) \left(r(s',a) - \hat{r}(s',a) \right)^2} + \norm{h}_2 \sqrt{SA} \sqrt{\sum_s \sum_{s',a} \left( P(s',a,s) - \hat{P}(s',a,s)\right)^2}\\
    &\le \norm{h}_2 \sqrt{SA} \norm{r - \hat{r}}_2 + \norm{h}_2 \sqrt{SA} \norm{P - \hat{P}}_2 
\end{align*}
\begin{align*}
    T_3 &\le \norm{h}_2 \sqrt{\sum_s \left(\sum_{\tilde{s},a} h(\tilde{s}) \left(P(s,a,\tilde{s}) - \widehat{P}(s,a,\tilde{s}) \right) \right)^2}\\
    &\le \norm{h}_2 \sqrt{\sum_s \norm{h}_2^2 \sum_{\tilde{s}} \left( \sum_a \left( P(s,a,\tilde{s}) - \widehat{P}(s,a,\tilde{s})\right) \right)^2} \le \norm{h}_2^2 \sqrt{A} \norm{P - \widehat{P}}_2
\end{align*}
\begin{align*}
    T_4 &\le \abs{\sum_{s,a} \sum_{\tilde{s}, \hat{s}} h(\tilde{s}) h(\hat{s}) P(s,a,\hat{s}) \left(P(s,a,\tilde{s}) - \widehat{P}(s,a,\tilde{s}) \right) }\\
    &+ \abs{\sum_{s,a} \sum_{\tilde{s}, \hat{s}} h(\tilde{s}) h(\hat{s}) \widehat{P}(s,a,\hat{s}) \left(P(s,a,\hat{s}) - \widehat{P}(s,a,\hat{s}) \right) }\\
    &\le 2 \norm{h}_2 \abs{\sum_{s,a} \sum_{\tilde{s}} h(\tilde{s})\abs{ P(s,a,\tilde{s}) - \widehat{P}(s,a,\tilde{s})} } \quad \textrm{[By $\sum_{\hat{s}} h(\hat{s}) P(s,a,\hat{s}) \le \norm{h}_2 \sqrt{\sum_{\hat{s}} P(s,a,\hat{s})} = \norm{h}_2$]}\\
    &\le 2 \norm{h}_2^2 \sqrt{\sum_{\tilde{s}} \left( \sum_{s,a} \abs{P(s,a,\tilde{s}) - \widehat{P}(s,a,\tilde{s})}\right)^2} \\
    &\le 2\norm{h}_2^2 \sqrt{SA} \norm{P - \widehat{P}}_2
\end{align*}
Substituting the upper bounds on $T_1, T_2, T_3$, and $T_4$ into the upper bound on $\calL_d(h; M) - \calL_d(h; \widehat{M})$ gives the desired bound.
\end{proof}

}

%% file: main_convergence.tex
\subsection{Proof of Convergence of Repeated Maximization (Theorem~\ref{thm:primal-convergence})}
\begin{proof}
We first compute the dual of the concave optimization problem~\ref{eq:regularized-rl-discounting}. The Lagrangian is given as
\begin{align*}
\calL(d,h) = d^\top r_t - \frac{\lambda}{2} \norm{d}_2^2 + \sum_s h(s) \left( -\sum_a d(s,a) + \rho(s) + \gamma \cdot \sum_{s',a} d(s',a) P_t(s',a,s) \right)
\end{align*}
At an optimal solution we must have $\nabla_d\calL(d,h)  = 0$, which gives us the following expression for $d$.
\begin{align}
\label{eq:dsa-optimal}
d(s,a) = \frac{r_t(s,a)}{\lambda} - \frac{h(s)}{\lambda} + \frac{\gamma}{\lambda} \sum_{\tilde{s}} h(\tilde{s}) P_t(s,a,\tilde{s})
\end{align}
Substituting the above value of $d$ we get the following dual problem. 
\begin{align}
\begin{split}\label{eq:dual-discounted}
\min_{h \in \R^S} \ &-\frac{1}{\lambda} \sum_{s,a} h(s) r_t(s,a) + \frac{\gamma}{\lambda} \sum_s \sum_{s',a} h(s) r_t(s',a) P_t(s',a,s) + \sum_s h(s) \rho(s)\\
&+ \frac{A}{2\lambda} \sum_s h(s)^2 - \frac{\gamma}{\lambda} \sum_{s,a} h(s) \sum_{\tilde{s} } h(\tilde{s}) P_t(s,a,\tilde{s}) 
+ \frac{\gamma^2}{2\lambda} \sum_{s,a} \sum_{\tilde{s}, \hat{s}} h(\tilde{s}) h(\hat{s}) P_t(s,a,\hat{s}) P_t(s,a,\tilde{s})
\end{split}
\end{align}

Note that the dual objective is parameterized by reward function $r_t$ and probability transition function $P_t$ which are the parameters corresponding to the occupancy measure $d_t$. We will write $\calL(\cdot; M_t)$ to denote this dual objective function.

For a given occupancy measure (i.e. $d_t = d$) we will write $\GD(d)$ to denote the optimal solution to the primal problem~\ref{eq:regularized-rl-discounting}. We first aim to show that the operator $\GD(\cdot)$ is a contraction mapping. Consider two occupancy measures $d$ and $\hat{d}$. Let $r$ (resp. $\hat{r}$) be the reward functions in response to the occupancy measure $d$ (resp. $\hat{d}$). Similarly, let $P$ (resp. $\hat{P}$) be the probability transition function in response to the occupancy measure $d$ (resp. $\hat{d}$).



Let $h$ (resp. $\hat{h}$) be the optimal dual solutions corresponding to the occupancy measures $d$ (resp. $\hat{d}$) i.e. $h \in \argmax_{h'} \calL(h'; M)$ and $\hat{h} \in \argmax_{h'} \calL(h'; \hat{M})$. Lemma \ref{lem:strong-convexity} proves that the objective is $A(1-\gamma)^2 / \lambda$ strongly convex. Therefore, we have the following two inequalities.

\begin{align}
&\calL(h; M) - \calL(\hat{h}; M) \ge \left(h - \hat{h} \right)^\top \nabla \calL(\hat{h}; M) + \frac{A(1-\gamma)^2}{2\lambda} \norm{h - \hat{h}}_2^2 \label{eq:strong-convexity-bound}\\
&\calL(\hat{h}; M) - \calL(h; M) \ge \frac{A(1-\gamma)^2}{2\lambda} \norm{h - \hat{h}}_2^2 \label{eq:optimality-gap-bound}
\end{align}

These two inequalities give us the following bound.
\begin{equation}\label{eq:upper-bound-gradient}
-\frac{A(1-\gamma)^2}{\lambda} \norm{h - \hat{h}}_2^2 \ge (h - \hat{h})^\top \nabla \calL(\hat{h};M)
\end{equation}
We now bound the Lipschitz constant of the term $ (h - \hat{h})^\top \calL_d(\hat{h};M)$ with respect to the MDP $M$. Lemma \ref{lem:lipschitz-wrt-M} gives us the following bound.
$$
\norm{\nabla \calL(\hat{h};M) - \nabla \calL(\hat{h};\hat{M})}_2 \le \frac{4S\sqrt{A}}{\lambda} \norm{r - \hat{r}}_2 + \left( \frac{4\gamma \sqrt{SA}}{\lambda} + \frac{6\gamma \sqrt{A}S}{\lambda} \norm{\hat{h}}_2\right) \norm{ P - \hat{P}}_2 
$$
Now notice that the dual variable $\hat{h}$ is actually an optimal solution and we can use lemma~\ref{lem:bound-optimal-dual} to bound its norm by $\frac{3S}{(1-\gamma)^2}$. Furthermore, under assumption~\ref{sensitivity-assumption}, we have $\norm{r - \hat{r}}_2 \le \epsilon_r \norm{d - \hat{d}}_2$ and $\norm{P - \hat{P}}_2 \le \epsilon_p \norm{d - \hat{d}}_2$. Substituting these bounds we get the following inequality.
\begin{align*}
&\norm{\nabla \calL(\hat{h};M) - \nabla \calL(\hat{h};\hat{M})}_2 \le \frac{4S\sqrt{A}}{\lambda} \epsilon_r \norm{d - \hat{d}}_2 + \left( \frac{4\gamma \sqrt{SA}}{\lambda} + \frac{6\gamma \sqrt{A}S}{\lambda} \frac{3S}{(1-\gamma)^2}\right) \epsilon_p \norm{d - \hat{d}}_2 \\
&\le \left(\frac{4S\sqrt{A}\epsilon_r}{\lambda} + \frac{10 \gamma S^2 \sqrt{A} \epsilon_p}{\lambda(1-\gamma)^2} \right) \norm{d - \hat{d}}_2\\
\end{align*}
We now substitute the above bound in equation \ref{eq:upper-bound-gradient}.
\begin{align*}
&-\frac{A(1-\gamma)^2}{\lambda} \norm{h - \hat{h}}_2^2 \ge (h - \hat{h})^\top \nabla \calL(\hat{h};M) \\
&= (h - \hat{h})^\top \left( \nabla \calL(\hat{h};M) - \nabla \calL(\hat{h}; \hat{M})\right)\ \textrm{[As $\hat{h}$ is optimal for $\calL(\cdot; \hat{M})$]}\\
&\ge - \norm{h - \hat{h}}_2 \norm{\nabla \calL(\hat{h};M) - \nabla \calL(\hat{h};\hat{M})}_2 \\
&\ge - \norm{h - \hat{h}}_2 \left(\frac{4S\sqrt{A}\epsilon_r}{\lambda} + \frac{10 \gamma S^2 \sqrt{A} \epsilon_p}{\lambda(1-\gamma)^2} \right) \norm{d - \hat{d}}_2
\end{align*}
Rearranging we get the following inequality.
\begin{align*}
&\norm{h - \hat{h}}_2 \le  \frac{\lambda}{A(1-\gamma)^2} \left(\frac{4S\sqrt{A}\epsilon_r}{\lambda} + \frac{10 \gamma S^2 \sqrt{A} \epsilon_p}{\lambda(1-\gamma)^2} \right) \norm{d - \hat{d}}_2\\
\end{align*}
Recall that $\GD(d)$ (resp. $\GD(\hat{d})$) are the optimal solution corresponding to the primal problem when the deployed occupancy measure is $d$ (resp. $\hat{d}$). Therefore, we can apply lemma~\ref{lem:deviation-d-to-h} to obtain the following bound.
\begin{align*}
&\norm{\GD(d)-\GD(\hat{d})}_2 \le \left( 1 + \frac{4\epsilon_r + 6\epsilon_p \norm{\hat{h}}_2}{\lambda}\right) \frac{3\sqrt{AS}}{\lambda} \norm{h - \hat{h}}_2\\
&\le \left( 1 + \frac{4\epsilon_r + 6\epsilon_p \cdot 3S/(1-\gamma)^2}{\lambda}\right) \frac{3\sqrt{AS}}{\lambda} \frac{\lambda}{A(1-\gamma)^2} \left(\frac{4S\sqrt{A}\epsilon_r}{\lambda} + \frac{10 \gamma S^2 \sqrt{A} \epsilon_p}{\lambda(1-\gamma)^2} \right) \norm{d - \hat{d}}_2\\
&\le \underbrace{\left( 1 + \frac{4\epsilon_r + 6\epsilon_p \cdot 3S/(1-\gamma)^2}{\lambda}\right) \frac{3\sqrt{S}}{\sqrt{A}(1-\gamma)^2} \left(\frac{4S\sqrt{A}\epsilon_r}{\lambda} + \frac{10 \gamma S^2 \sqrt{A} \epsilon_p}{\lambda(1-\gamma)^2} \right)}_{:=\beta} \norm{d - \hat{d}}_2
\end{align*}

Now it can be easily verified that if $\lambda > 12S^{3/2}(1-\gamma)^{-4}(2\epsilon_r + 5S\epsilon_p)$ then $\beta = \frac{12S^{3/2}(2\epsilon_r + 5S\epsilon_p)}{\lambda (1-\gamma)^4} < 1$. This implies that the operator $\GD(\cdot)$ is a contraction mapping and the sequence of iterates $\{d_t\}_{t \ge 1}$ converges to a fixed point. In order to determine the speed of convergence let us substitute $d=d_t$ and $\hat{d} = d_S$. This gives us
$
\norm{\GD(d_t) - d_S}_2 \le \beta \norm{d_t - d_S}_2
$. As $\GD(d_t) = d_{t+1}$ we have $\norm{d_{t+1} - d_S}_2 \le \beta \norm{d_t - d_S}_2$. After $t$ iterations we have $\norm{d_t - d_S}_2 \le \beta^t \norm{d_0 - d_S}_2$. Therefore, if $t \ge \ln\left(\norm{d_0 - d_S}_2 / \delta \right) / \ln(1/\beta)$ we are guaranteed that $\norm{d_t - d_S}_2 \le {\delta}$. Since $\norm{d_0 - d_S}_2 \le \frac{2}{1-\gamma}$, the 
%
desired upper bound on the number of iterations becomes the following.
\begin{align*}
\frac{\ln\left(\norm{d_0 - d_S}_2 / {\delta}\right) }{\ln(1/\beta)} \le 2 \left( 1-\beta \right)^{-1}  \ln \left( \frac{2}{\delta (1-\gamma)  }\right)
\end{align*}
\end{proof}

\begin{lemma}\label{lem:deviation-d-to-h}
Consider two state-action occupancy measures $d$ and $\hat{d}$. Let $\lambda \ge 2\left(2\epsilon_r + 3\epsilon_p \norm{\hat{h}}_2\right)$. Then we have the following bound.
$$
\norm{d-\hat{d}}_2 \le \left( 1 + \frac{4\epsilon_r + 6\epsilon_p \norm{\hat{h}}_2}{\lambda}\right) \frac{3\sqrt{AS}}{\lambda} \norm{h - \hat{h}}_2
$$
\end{lemma}
\begin{proof}
Recall the relationship between the dual and the primal variables.
$$
d(s,a) = \frac{r_t(s,a)}{\lambda} - \frac{h(s)}{\lambda} + \frac{\gamma}{\lambda} \sum_{\tilde{s}} h(\tilde{s}) P(s,a,\tilde{s})
$$
This gives us the following bound on the difference $(d(s,a) - \hat{d}(s,a))^2$.
\begin{align*}
\left(d(s,a) - \hat{d}(s,a) \right)^2 &\le \frac{3}{\lambda^2} \left(r(s,a) - \hat{r}(s,a) \right)^2 + \frac{1}{\lambda^2} \left( h(s) - \hat{h}(s)\right)^2 \\
&+ \frac{3\gamma^2}{\lambda^2} \left( \sum_{s'} h(s') P(s,a,s') - \sum_{s'} \hat{h}(s') \hat{P}(s,a,s') \right)^2 \ \textrm{[By Jensen's inequality]}\\
&= \frac{3}{\lambda^2} \left(r(s,a) - \hat{r}(s,a) \right)^2 + \frac{1}{\lambda^2} \left( h(s) - \hat{h}(s)\right)^2 \\
&+ \frac{3}{\lambda^2}  \left( \sum_{s'} \left(h(s') - \hat{h}(s') \right) P(s,a,s') + \hat{h}(s') \left( P(s,a,s') - \hat{P}(s,a,s')\right) \right)^2\\
&\le \frac{3}{\lambda^2} \left(r(s,a) - \hat{r}(s,a) \right)^2 + \frac{1}{\lambda^2} \left( h(s) - \hat{h}(s)\right)^2\\
&+ \frac{6}{\lambda^2} \left( \sum_{s'} \left(h(s') - \hat{h}(s') \right) P(s,a,s')\right)^2 \\&+ \frac{6}{\lambda^2} \left(\sum_{s'} \hat{h}(s') \left( P(s,a,s') - \hat{P}(s,a,s')\right) \right)^2\ \textrm{[By Jensen's inequality]}\\
&\le \frac{3}{\lambda^2} \left(r(s,a) - \hat{r}(s,a) \right)^2 + \frac{1}{\lambda^2} \left( h(s) - \hat{h}(s)\right)^2\\
&+\frac{6}{\lambda^2} \norm{h - \hat{h}}_2^2 + \frac{6}{\lambda^2} \norm{\hat{h}}_2^2 \sum_{s'} \left( P(s,a,s')  - \hat{P}(s,a,s')\right)^2\ \textrm{[By Cauchy-Schwarz inequality]}
\end{align*}
Now summing over $s$ and $a$ we get the following bound.
$$
\norm{d - \hat{d}}_2^2 \le \frac{3}{\lambda^2} \norm{r - \hat{r}}_2^2 + \frac{7AS}{\lambda^2} \norm{h - \hat{h}}_2^2 + \frac{6}{\lambda^2} \norm{\hat{h}}_2^2 \norm{P - \hat{P}}_2^2
$$
We now use the assumptions $\norm{r - \hat{r}}_2 \le \epsilon_r \norm{d - \hat{d}}_2$ and $\norm{P - \hat{P}}_2 \le \epsilon_p \norm{d - \hat{d}}_2$.
\begin{align*}
\norm{d - \hat{d}}_2 \le \frac{2\epsilon_r}{\lambda} \norm{d - \hat{d}}_2 + \frac{3 \sqrt{AS}}{\lambda} \norm{h - \hat{h}}_2 + \frac{3\epsilon_p}{\lambda} \norm{\hat{h}}_2 \norm{d - \hat{d}}_2
\end{align*}
Rearranging we get the following bound.
\begin{align*}
\norm{d-\hat{d}}_2 \le \left( 1 - \frac{2\epsilon_r + 3\epsilon_p \norm{\hat{h}}_2}{\lambda}\right)^{-1} \frac{3\sqrt{AS}}{\lambda} \norm{h - \hat{h}}_2 \le \left( 1 + \frac{4\epsilon_r + 6\epsilon_p \norm{\hat{h}}_2}{\lambda}\right) \frac{3\sqrt{AS}}{\lambda} \norm{h - \hat{h}}_2
\end{align*}
The last inequality uses the fact that $\lambda \ge 2(2\epsilon_r + 3\epsilon_p \norm{\hat{h}}_2)$.
\end{proof}

\begin{lemma}\label{lem:strong-convexity}
The dual objective $\calL_d$ (as defined in \ref{eq:dual-discounted}) is $\frac{A(1-\gamma)^2}{\lambda}$-strongly convex.
\end{lemma}
\begin{proof}
The derivative of the dual objective $\calL_d$ with respect to $h(s)$ is given as follows.
\begin{align}
\begin{split}\label{eq:derivative-dual}
\frac{\partial \calL_d(h)}{\partial h(s)} &= -\frac{1}{\lambda} \sum_a r_t(s,a) + \frac{\gamma}{\lambda} \sum_{s',a} r_t(s',a) P_t(s',a,s) + \rho(s) + \frac{A}{\lambda } h(s)\\
&- \frac{\gamma}{\lambda} \sum_{\tilde{s}, a} h(\tilde{s}) \left( P_t(s,a,\tilde{s}) + P_t(\tilde{s},a,s)\right) + \frac{\gamma^2}{\lambda} \sum_{s',a,\tilde{s}} h(\tilde{s}) P_t(s',a,\tilde{s}) P_t(s',a,s)
\end{split}
\end{align}
This gives us the following identity.
\begin{align*}
\left(\nabla \calL_d(h) - \nabla \calL_d(\tilh) \right)^\top (h - \tilde{h}) &= \frac{A}{\lambda} \norm{h - \tilh}_2^2
\\
&- \frac{\gamma}{\lambda} \sum_{s,\tilde{s},a} (h(s) - \tilh(s))\left( P_t(s,a,\tilde{s}) + P_t(\tilde{s},a,s) \right) (h(\tilde{s}) - \tilh(\tilde{s}))\\
&+ \frac{\gamma^2}{\lambda} \sum_{s',a} \sum_{s,\tilde{s}} (h(s) - \tilh(\tilde{s})) P_t(s',a,\tilde{s}) P_t(s',a,s) (h(s) - \tilh(s))
\end{align*}
Let us now define the  matrix $M_a \in \R^{S \times S}$ with entries 
 $M_a(s,s') = P_t(s,a,s')$. 
\begin{align*}
&\left(\nabla \calL_d(h) - \nabla \calL_d(\tilh) \right)^\top (h - \tilde{h}) = \frac{A}{\lambda} \norm{h - \tilh}_2^2
\\
&-\frac{\gamma}{\lambda} \sum_a (h-\tilh)^\top \left(M_a + M_a^\top\right) (h - \tilh) + \frac{\gamma^2}{\lambda} \sum_a (h-\tilh)^\top M_a^\top M_a (h - \tilh)\\
&= \frac{1}{\lambda}\sum_a (h -\tilh)^\top \left( \textrm{Id} - \gamma M_a - \gamma M_a^\top + \gamma^2 M_a^\top M_a\right) (h - \tilh)\\
&\ge   \frac{A(1-\gamma)^2}{\lambda} \norm{h - \tilh}_2^2
\end{align*}
The last inequality uses lemma~\ref{lem:eigenvalue-bound}.
\end{proof}

\begin{lemma}\label{lem:lipschitz-wrt-M}
The dual function $\calL_d$ (as defined in \cref{eq:dual-discounted}) satisfies the following bound for any $h$ and MDP $M, \widehat{M}$.
$$
\norm{\nabla \calL_d (h,M) - \nabla\calL_d (h,\widehat{M})}_2 \le \frac{4 S \sqrt{A}}{\lambda} \norm{r - \hat{r}}_2 + \left(\frac{4\gamma \sqrt{SA}}{\lambda} + \frac{6\gamma S \sqrt{A}}{\lambda} \norm{h}_2 \right) \norm{P - \hat{P}}_2
$$
\end{lemma}
\begin{proof}
From the expression of the derivative of $\calL_d$ with respect to $h$ (\cref{eq:derivative-dual}) we get the following bound.
\begin{align*}
&\norm{\nabla \calL_d (h,M) - \nabla\calL_d (h,\widehat{M})}_2^2 = \sum_s \left\{ -\frac{1}{\lambda} \sum_a (r(s,a) - \hat{r}(s,a))\right. \\
& + \frac{\gamma}{\lambda} \sum_{s',a} \left(r(s',a) P(s',a,s) - \hat{r}(s',a) \hat{P}(s',a,s) \right) -\frac{\gamma}{\lambda} \sum_{\tilde{s},a} h(\tilde{s}) (P(\tilde{s},a,s) - \hat{P}(\tilde{s},a,s))\\
&-\left.\frac{\gamma}{\lambda} \sum_{\tilde{s},a} h(\tilde{s}) (P(s,a,\tilde{s}) - \hat{P}(s,a,\tilde{s}) ) + \frac{\gamma^2}{\lambda} \sum_{s',a,\tilde{s}} h(\tilde{s}) \left(P(s',a,\tilde{s}) P(s',a,s) - \hat{P}(s',a,\tilde{s}) \hat{P}(s',a,s) \right) \right\}\\
&\le \frac{5A}{\lambda^2} \norm{r-\hat{r}}_2^2 + \frac{5\gamma^2}{\lambda^2} \sum_s \left( \sum_{s',a} \left(r(s',a) P(s',a,s) - \hat{r}(s',a) \hat{P}(s',a,s) \right) \right)^2\\
&+ \frac{5 \gamma^2}{\lambda^2} \sum_s \left(\sum_{\tilde{s},a} h(s) (P(\tilde{s},a,s) - \hat{P}(\tilde{s},a,s)) \right)^2 + \frac{5 \gamma^2}{\lambda^2} \sum_s \left(\sum_{\tilde{s},a} h(\tilde{s}) (P(s,a,\tilde{s}) - \hat{P}(s,a,\tilde{s})) \right)^2 \\
&+ \frac{5 \gamma^4}{\lambda^2}\sum_s \left(\sum_{s',a,\tilde{s}} h(\tilde{s}) \left(P(s',a,\tilde{s}) P(s',a,s) - \hat{P}(s',a,\tilde{s}) \hat{P}(s',a,s) \right) \right)^2
\end{align*}
We now use four bounds to complete the proof. The following bounds use Jensen's inequality and Cauchy-Schwarz inequality.
\begin{align*}
&\textrm{Bound 1}: \sum_{s',a} \left(r(s',a) P(s',a,s) - \hat{r}(s',a) \hat{P}(s',a,s) \right) \\&
\le \sum_{s',a} \abs{r(s',a) - \hat{r}(s',a)} P(s',a,s) + \hat{r}(s',a)\abs{P(s',a,s) - \hat{P}(s',a,s)} \\
&\le \norm{r - \hat{r}}_1 + \sum_{s',a} \abs{P(s',a,s) - \hat{P}(s',a,s)} \\
\end{align*}
\begin{align*}
&\textrm{Bound 2}: \sum_s \left(\sum_{\tilde{s},a} h(\tilde{s}) (P(\tilde{s},a,s) - \hat{P}(\tilde{s},a,s)) \right)^2 \le A \sum_s \sum_{\tilde{s}} h(\tilde{s})^2 \sum_{\tilde{s},a} \left(P(\tilde{s},a,s) - \hat{P}(\tilde{s},a,s) \right)^2\\
&\le A \norm{h}_2^2 \norm{P-\hat{P}}_2^2\\
&\textrm{Bound 3}: \sum_s \left(\sum_{\tilde{s},a} h(\tilde{s}) (P(s,a,\tilde{s}) - \hat{P}(s,a,\tilde{s})) \right)^2 \le A \sum_s \sum_{\tilde{s}} h(\tilde{s})^2 \sum_{\tilde{s},a} \left(P(s,a,\tilde{s}) - \hat{P}(s,a,\tilde{s}) \right)^2\\
&\le A \norm{h}_2^2 \norm{P-\hat{P}}_2^2\\
\end{align*}

\begin{align*}
&\textrm{Bound 4}: \sum_s \left(\sum_{s',a,\tilde{s}} h(\tilde{s}) \left(P(s',a,\tilde{s}) P(s',a,s) - \hat{P}(s',a,\tilde{s}) \hat{P}(s',a,s) \right) \right)^2 \\
&\le \sum_s \sum_{\tilde{s}}h(\tilde{s})^2 \sum_{\tilde{s}}\left(\sum_{s',a} \left(P(s',a,\tilde{s}) P(s',a,s) - \hat{P}(s',a,\tilde{s}) \hat{P}(s',a,s) \right) \right)^2\\
&\le SA \norm{h}_2^2 \sum_{s,\tilde{s}} \sum_{s',a} \left(P(s',a,\tilde{s}) P(s',a,s) - \hat{P}(s',a,\tilde{s}) \hat{P}(s',a,s) \right)^2\\
&\le SA \norm{h}_2^2 \sum_{s,\tilde{s}} \sum_{s',a} \left(P(s',a,\tilde{s})(P(s',a,s) - \hat{P}(s',a,s)) + \hat{P}(s',a,s) (P(s',a,\tilde{s}) - \hat{P}(s',a,\tilde{s})) \right)\\
&\le 2SA \norm{h}_2^2 \sum_{s,\tilde{s}} \sum_{s',a} \abs{P(s',a,s) - \hat{P}(s',a,s)}^2 + \abs{P(s',a,\tilde{s}) - \hat{P}(s',a,\tilde{s})}^2\\
&\le 4S^2A \norm{h}_2^2 \norm{P-\hat{P}}_2^2
\end{align*}
Using the four upper bounds shown above, we can complete the proof.
\begin{align*}
&\norm{\nabla \calL_d (h,M) - \nabla\calL_d (h,\widehat{M})}_2^2 \le \frac{5A}{\lambda^2} \norm{r - \hat{r}}_2^2 + \frac{5\gamma^2}{\lambda^2}\sum_s \left(\norm{r - \hat{r}}_1 + \norm{P(\cdot,\cdot,s) - \hat{P}(\cdot,\cdot,s)}_1 \right)^2\\
&+ \frac{10 A\gamma^2}{\lambda^2} \norm{h}_2^2 \norm{P - \hat{P}}_2^2 + \frac{20S^2 A \gamma^4}{\lambda^2} \norm{h}_2^2 \norm{P - \hat{P}}_2^2\\
&\le \left( \frac{5A}{\lambda^2} + \frac{10 S^2 A \gamma^2}{\lambda^2}\right) \norm{r - \hat{r}}_2^2 + \left( \frac{10 \gamma^2 SA}{\lambda^2} + \frac{10 A\gamma^2}{\lambda^2} \norm{h}_2^2 +\frac{20S^2 A \gamma^4}{\lambda^2} \norm{h}_2^2 \right) \norm{P - \hat{P}}_2^2
\end{align*}
\end{proof}

\begin{lemma}\label{lem:bound-optimal-dual}
The norm of the optimal solution to the dual problem (defined in \ref{eq:dual-discounted}) is bounded by $\frac{3S}{(1-\gamma)^2}$ for any choice of MDP $M$.
\end{lemma}
\begin{proof}
The dual objective $\calL_d$ is strongly-convex and has a unique solution. The optimal solution can be obtained by setting the derivative with respect to $h$ to zero. Rearranging the derivative of the dual objective (\ref{eq:derivative-dual}) we get the following systems of equations.
\begin{align*}
&h(s) \left(\frac{A}{\lambda} - \frac{2 \gamma}{\lambda} \sum_a P_t(s,a,s) + \frac{\gamma^2}{\lambda} \sum_{s',a} P_t(s',a,s)^2 \right)\\
&+ \sum_{\hat{s} \neq s} h(\hat{s}) \left(-\frac{\gamma}{\lambda} \sum_a P_t(s,a,\hat{s}) -\frac{\gamma}{\lambda} \sum_a P_t(\hat{s},a,{s}) + \frac{\gamma^2}{\lambda} \sum_{s',a}  P_t(s',a,\hat{s}) P_t(s',a,s) \right) \\
&= \frac{1}{\lambda} \sum_a r_t(s,a) - \rho(s) - \frac{\gamma}{\lambda} \sum_{s',a} r_t(s',a) P_t(s',a,s)
\end{align*}
Therefore let us define a matrix $B \in \R^{S\times S}$ and a vector $b \in \R^S$ with the following entries.
\begin{align*}
B(s,\hat{s}) = \left\{ \begin{array}{cc}
\frac{A}{\lambda} - \frac{2 \gamma}{\lambda} \sum_a P_t(s,a,s) + \frac{\gamma^2}{\lambda} \sum_{s',a} P_t(s',a,s)^2  & \textrm{ if } s = \hat{s}\\
-\frac{\gamma}{\lambda} \sum_a P_t(s,a,\hat{s}) -\frac{\gamma}{\lambda} \sum_a P_t(\hat{s},a,{s}) + \frac{\gamma^2}{\lambda} \sum_{s',a}  P_t(s',a,\hat{s}) P_t(s',a,s) & \textrm{ o.w.} 
\end{array}\right.
\end{align*}
\begin{align*}
b(s) = \frac{1}{\lambda} \sum_a r_t(s,a) - \rho(s) - \frac{\gamma}{\lambda} \sum_{s',a} r_t(s',a) P_t(s',a,s)
\end{align*}
Then the optimal solution is the solution of the system of equations $Bh = b$. We now provide a bound on the $L_2$-norm of such a solution. For each $a$, we define matrix $M_a \in R^{S \times S}$ with entries $M_a(s,\hat{s}) = P_t(s,a,\hat{s})$. Then the matrix $B$ can be expressed as follows.
\begin{align*}
    B = \frac{A}{\lambda} \Identity - \frac{\gamma}{\lambda} \sum_a \left( M_a + M_a^\top\right) + \frac{\gamma^2}{\lambda} \sum_a M_a^\top M_a \succcurlyeq \frac{A(1-\gamma)^2}{\lambda} \Identity
\end{align*}
The last inequality uses lemma~\ref{lem:eigenvalue-bound}.
Notice that for $\gamma < 1$ this also shows that the matrix is invertible. We can also bound the norm of the vector $b$.
\begin{align*}
\norm{b}_2^2 &\le \sum_s \left( \frac{A}{\lambda} + \rho(s)  + \frac{\gamma}{\lambda}\sum_{s',a} P_t(s',a,s) \right)^2\\
&\le 3\frac{A^2 S}{\lambda^2} + 3\norm{\rho}_2^2 + 3\frac{\gamma^2}{\lambda^2} \sum_s \left( \sum_{s',a} P_t(s',a,s) \right)^2\\
&\le 3\frac{A^2 S}{\lambda^2} + 3 + \frac{3SA\gamma^2}{\lambda^2} \sum_s \sum_{s',a} P_t(s',a,s) \le \frac{9S^2 A^2}{\lambda^2}
\end{align*}
Therefore, we have the following bound on the optimal value.
\begin{align*}
\norm{A^{-1} b}_2 \le \frac{\norm{b}_2}{\lambda_{\min}(A)} \le \frac{3S}{(1-\gamma)^2}
\end{align*}
\end{proof}

\begin{lemma}\label{lem:eigenvalue-bound}
For each $a$, let the matrix $M_a \in \R^{S\times S}$ be defined so that $M_a(s,s') = P(s,a,s')$.
$$\lambda_{\min}\left(\sum_a \Identity - \gamma(M_a + M_a^\top) + \gamma^2 M_a^\top M_a \right) \ge A(1-\gamma)^2$$
\end{lemma}
\begin{proof}
Let $M_a = U_a \Sigma_a U_a^\top$ be the Eigen-decomposition of the matrix $M_a$. Then we have
\begin{align*}
    &\sum_a \Identity - \gamma(M_a + M_a^\top) + \gamma^2 M_a^\top M_a = \sum_a \left(\Identity - \gamma M_a \right)^\top \left(\Identity - \gamma M_a \right) \\
    = &\sum_a \left(U_a(\Identity - \gamma \Sigma_a) U_a^\top \right)^\top \left( U_a(\Identity - \gamma \Sigma_a) U_a^\top\right) = \sum_a U_a \left(\Identity - \gamma \Sigma_a \right)^2 U_a^\top\\
    \succcurlyeq &\sum_a \Identity (1-\gamma)^2 = A (1-\gamma)^2 \Identity
\end{align*}
The last line follows the largest eigenvalue of $M_a$ is $1$, and therefore the smallest diagonal entry of the matrix $(\Identity - \gamma \Sigma_a)^2$ is at least $(1-\gamma)^2$.
\end{proof}

%% file: rga_convergence.tex
\subsection{Proof of Convergence of Repeated Gradient Ascent (Theorem~\ref{thm:rga-convergence}) }
 \begin{proof}
The dual of the  optimization problem~\ref{eq:primal-projection} is given as follows.
\begin{align}
    \begin{split}
    \max_{h \in \R^S}\ &\sum_{s,a} h(s) \left((1-\eta \lambda) d_t(s,a) + \eta r_t(s,a) \right) - \sum_s h(s) \rho(s)\\
    &-\gamma \cdot \sum_s h(s) \sum_{s',a} P_t(s',a,s) \left((1-\eta \lambda) d_t(s',a) + \eta r_t(s',a) \right)- \frac{A}{2} \sum_s h(s)^2\\
    & + \gamma \cdot \sum_s h(s) \sum_{s',a} h(s') P_t(s',a,s)
    -\frac{\gamma^2}{2} \sum_{s',s''} h(s') h(s'') \sum_{s,a} P_t(s,a,s') P_t(s,a,s'')
    \end{split}
\end{align}
We will consider the equivalent minimization problem.
\begin{align}
    \begin{split}\label{eq:dual-projection}
    \min_{h \in \R^S}\ &-\sum_{s,a} h(s) \left((1-\eta \lambda) d_t(s,a) + \eta r_t(s,a) \right) + \sum_s h(s) \rho(s)\\
    &+\gamma \cdot \sum_s h(s) \sum_{s',a} P_t(s',a,s) \left((1-\eta \lambda) d_t(s',a) + \eta r_t(s',a) \right)+ \frac{A}{2} \sum_s h(s)^2 \\
    &- \gamma \cdot \sum_s h(s) \sum_{s',a} h(s') P_t(s',a,s)
    +\frac{\gamma^2}{2} \sum_{s',s''} h(s') h(s'') \sum_{s,a} P_t(s,a,s') P_t(s,a,s'')
    \end{split}
\end{align}
Let us call the above objective function $\calP(\cdot; M)$ for a given MDP $M$. Consider two occupancy measures $d$ and $\hat{d}$. Let $r$ (resp. $\hat{r}$)  be the reward functions in response to the occupancy measure $d$ (resp. $\hat{d}$) i.e. $r = \calR(d)$ and $\hat{r} = \calR(\hat{d})$. Similarly let $P$ (resp. $\hat{P}$) be the probability transition functions in response to the occupancy measures $d$ (resp. $\hat{d}$). 

We will write $GD(\cdot)$ to denote the projected gradient ascent step defined in \cref{eq:gradient-ascent-pp}. In particular, if we write $\calC$ to define the set of occupancy measures feasible with respect to $P$, then we have
\begin{equation}\label{eq:defn-GD-operator}
GD(d) = \proj_{\calC}\left( (1-\eta \lambda) d + \eta r\right) 
\end{equation}

Note that $GD_\eta(d)$ is the optimal solution to the primal problem~\ref{eq:primal-projection} with $d_t = d$. Let $h$ be the corresponding dual optimal solution. Similarly let $\hat{h}$ be the optimal dual solution corresponding to the occupancy measure $\hat{d}$. Since $h$ is the unique minimizer of $\calP(\cdot; M)$ and $\calP(\cdot; M)$ is $A(1-2\gamma)$-strongly convex for any $M$ (lemma~\ref{lem:strong-convexity-projection}) we have the following set of inequalities.
\begin{align*}
\calP(h; M, d) - \calP(\hat{h}; M, d) &\ge (h - \hat{h})^\top \nabla \calP(\hat{h}; M, d) + A(1-\gamma)^2/2 \norm{h - \hat{h}}_2^2 \\
\calP(\hat{h}; M, d) - \calP(h; M, d) &\ge A(1-\gamma)^2/2 \norm{h - \hat{h}}_2^2
\end{align*}
These two inequalities give us the following bound.
\begin{align}\label{eq:bound-inter-projection}
    -A(1-\gamma)^2 \norm{h - \hat{h}}_2^2 \ge (h - \hat{h})^\top \nabla \calP(\hat{h}; M, d)
\end{align}
We now apply lemma~\ref{lem:dual-smoothness-projection} to bound the Lipschitz constant of the term $(h - \hat{h})^\top \nabla \calP(\hat{h}; M)$.
\begin{align*}
    &\norm{\nabla \calP(\hat{h}; M, d) - \nabla \calP(\hat{h}; \hat{M}, \hat{d})}_2^2 \le 5A(1-\eta \lambda)^2 (1+2\gamma^2 S^2) \norm{d - \hat{d}}_2^2 + 5\eta^2 A \left( (1-\eta \lambda)^2 + 2\gamma^2 S^2 \right) \norm{r - \hat{r}}_2^2\\
&+ 5 \gamma^2 SA \left( \frac{2(1-\eta \lambda)^2}{(1-\gamma)^2} + 2\eta^2 + 6 \gamma^2 \norm{\hat{h}}_2^2 \right) \norm{P - \hat{P}}_2^2\\
&\le \left( 5A(1-\eta \lambda)^2 (1 + \eta^2\epsilon_r^2 + 2\gamma^2 S^2) + 10 \eta^2 \gamma^2 A S^2 \epsilon_r^2 \right) \norm{d - \hat{d}}_2^2\\
&+ 5 \gamma^2 SA \left( \frac{2(1-\eta \lambda)^2}{(1-\gamma)^2} + 2\eta^2 + 12 \gamma^2 \left( \frac{(1+2\eta S)^2}{(1-\gamma)^4} + 4 (1-\eta \lambda)^2 \frac{S^2}{A(1-\gamma)^6}\right) \right) \epsilon_p^2 \norm{d - \hat{d}}_2^2\\
&\le \left( 5A(1-\eta \lambda)^2 \left( 1 + \eta^2 \epsilon_r^2 + 2\gamma^2 S^2 +\frac{2\gamma^2 S}{(1-\gamma)^2}\epsilon_p^2 + \frac{48 \gamma^4 S^3}{A(1-\gamma)^6}\epsilon_p^2\right)  \right.\\
&+ \left. 10 \eta^2 \gamma^2 A S^2 \epsilon_r^2 + 10 \eta^2 \gamma^2 SA \epsilon_p^2 + 60 \gamma^4 SA \epsilon_p^2 \frac{(1+2\eta S)^2}{(1-\gamma)^4}\right) \norm{d - \hat{d}}_2^2\\
&\le \underbrace{ \left( 5A(1-\eta \lambda)^2 \left( 1 + \eta^2 \epsilon_r^2 + 2\gamma^2 S^2 + \frac{50 \gamma^2 S^3}{(1-\gamma)^6}\epsilon_p^2\right)  + 10 \eta^2 \gamma^2 S^2A (\epsilon_p^2 +\epsilon_r^2) + 60 \gamma^4 SA \epsilon_p^2 \frac{(1+2\eta S)^2}{(1-\gamma)^4}\right)}_{:=\Delta^2} \norm{d - \hat{d}}_2^2
\end{align*}
The last inequality uses lemma~\ref{lem:bound-optimal-dual-projection} and assumption~\ref{sensitivity-assumption}. Substituting this bound in equation~\ref{eq:bound-inter-projection} we get the following inequality.
\begin{align*}
    &-A(1-\gamma)^2 \norm{h - \hat{h}}_2^2 \ge (h - \hat{h})^\top \nabla \calP(\hat{h}; M, d) \\
    &= (h - \hat{h})^\top \nabla \calP(\hat{h}; M, d) - (h - \hat{h})^\top \nabla \calP(\hat{h}; \hat{M}, \hat{d})\\
    &\ge - \norm{h - \hat{h}}_2 \norm{\nabla \calP(\hat{h}; M, d) - \nabla \calP(\hat{h}; \hat{M}, \hat{d}) }_2 \ge - \Delta \norm{h - \hat{h}}_2 \norm{d - \hat{d}}_2
\end{align*}
Rearranging we get the following inequality.
\begin{align*}
    &\norm{h - \hat{h}}_2 \le \frac{\Delta}{A(1-\gamma)^2} \norm{d - \hat{d}}_2\\
    \Rightarrow & \norm{h - \hat{h}}_2^2 \le \frac{\Delta^2}{A^2(1-\gamma)^4} \norm{d - \hat{d}}_2^2\\
    \Rightarrow &\frac{\norm{\GD_\eta(d) - \GD_\eta(\hat{d})}_2^2}{8 \gamma^2 SA} - \frac{4(1-\eta \lambda)^2 + 4\eta^2 \epsilon_r^2 + 8\gamma^2 \norm{\hat{h}}_2^2 \epsilon_p^2 }{8 \gamma^2 SA} \norm{d - \hat{d}}_2^2 \le  \frac{\Delta^2}{A^2(1-\gamma)^4} \norm{d - \hat{d}}_2^2
\end{align*}
The last line uses lemma~\ref{lem:deviation-d-to-h-projection}. After rearranging we get the following inequality.
\begin{align*}
    &\norm{\GD_\eta(d) - \GD_\eta(\hat{d})}_2^2 \le \left(4(1-\eta \lambda)^2 + 4\eta^2 \epsilon_r^2 + 8\gamma^2 \norm{\hat{h}}_2^2 \epsilon_p^2 + \frac{8 \gamma^2 \Delta^2 S}{A(1-\gamma)^4} \right) \norm{d - \hat{d}}_2^2\\
    &\le \left(4(1-\eta \lambda)^2 + 4\eta^2 \epsilon_r^2 + \frac{16\gamma^2 \epsilon_p^2 (1+2\eta S)^2}{(1-\gamma)^4} + 32 (1-\eta \lambda)^2 \frac{\gamma^2 \epsilon_p^2 S^2}{A (1-\gamma)^6} + \frac{8 \gamma^2 \Delta^2 S}{A(1-\gamma)^4} \right) \norm{d - \hat{d}}_2^2
\end{align*}
For $\eta = 1/\lambda$ we get the following bound.
\begin{align*}
&\norm{\GD_\eta(d) - \GD_\eta(\hat{d})}_2^2 \le \left( \frac{4\epsilon_r^2}{\lambda^2} + \frac{16\gamma^2 \epsilon_p^2 (1+2S/\lambda)^2}{(1-\gamma)^4} + \frac{80 \gamma^4 S^3 (\epsilon_r^2 + \epsilon_p^2)}{\lambda^2 (1-\gamma)^4} + \frac{480 \gamma^6 S^2\epsilon_p^2 (1+2S/\lambda)^2}{(1-\gamma)^8 }\right) \norm{d - \hat{d}}_2^2
\end{align*}
If we choose $\lambda \ge \max\set{4\epsilon_r, 2S, \frac{20 \gamma^2 S^{1.5}(\epsilon_r + \epsilon_p)}{(1-\gamma)^2} }$ we get the following condition.
\begin{align*}
&\norm{\GD_\eta(d) - \GD_\eta(\hat{d})}_2^2 \le \left( \frac{1}{2} +  \frac{64\gamma^2 \epsilon_p^2}{(1-\gamma)^4} + \frac{1920 \gamma^6 S^2\epsilon_p^2 }{(1-\gamma)^8 }\right) \norm{d - \hat{d}}_2^2
\end{align*}
For a contraction mapping we need the following condition.
$$
\frac{64\gamma^2 \epsilon_p^2}{(1-\gamma)^4}\left( 1 + \frac{30 \gamma^4 S^2}{(1-\gamma)^4}\right) < \frac{1}{2}
$$
We consider two cases. First, if $\frac{30\gamma^4 S^2}{(1-\gamma)^4} < 1$ then one can show that a sufficient condition is $\epsilon_p < \gamma S / 3$. On the other hand, if $\frac{30\gamma^4 S^2}{(1-\gamma)^4} \ge 1$ then we need $\epsilon_p < \frac{(1-\gamma)^4}{100\gamma^3 S}$. Combining the two conditions above a sufficient condition for contraction is the following.
$$
\epsilon_p < \min \set{\frac{\gamma S}{3},\frac{(1-\gamma)^4}{100 \gamma^3 S}}
$$
Now if we set $\mu = \sqrt{ \frac{1}{2} +  \frac{64\gamma^2 \epsilon_p^2}{(1-\gamma)^4} + \frac{1920 \gamma^6 S^2\epsilon_p^2 }{(1-\gamma)^8 }}$ we get the contraction mapping: $\norm{\GD_\eta(d) - \GD_\eta(\hat{d})}_2 \le \mu \norm{d - \hat{d}}_2$. Let $d^\star$ be the fixed point of this contraction mapping. Using $d = d_{t}$ and $\hat{d} = d^\star$ we get the following sequence of inequalities.
$$
\norm{d_{t+1} - d^\star}_2 \le \mu \norm{d_t - d^\star}_2 \le \ldots \le \mu^t \norm{d_1 - d^\star}_2 \le \mu^t \frac{2}{1-\gamma}
$$
The last inequality uses the fact that for any occupancy measure $d$ we have $\norm{d}_2 \le \norm{d}_1 \le \frac{1}{1-\gamma}$. Rearranging we get that as long as $t \ge \ln\left( \frac{2}{\delta(1-\gamma)}\right)/\ln(1/\mu)$ we have $\norm{d_{t} - d^\star}_2 \le \delta$.

We now show that the fixed point $d^\star$ is a stable point. In response to $d^\star$, let the probability transition function (resp. reward function) be $P^\star$ (resp. $d^\star$). Let $\mathcal{C}^\star$ be the set of occupancy measures corresponding to $d^\star$. Note that $\mathcal{C}^\star$ is a convex set. We consider two cases. First, $(1-\eta \lambda) d^\star + \eta r^\star \in \mathcal{C}^\star$. Then $d^\star = \GD_\eta(d^\star) = (1-\eta \lambda) d^\star + \eta r^\star$ and $r^\star - \lambda d^\star = \nabla \calP(d^\star; P^\star, r^\star) = 0$. Since $ \calP(\cdot; P^\star, r^\star)$ is a concave function the occupancy measure $d^\star$ is the optimal point and is a stable point.

Second, we consider the case when $(1-\eta \lambda) d^\star + \eta r^\star \notin \mathcal{C}^\star$. Since $d^\star = \proj_{\mathcal{C}^\star}\left( (1-\eta \lambda)d^\star + \eta r^\star \right)$. Since $\mathcal{C}^\star$ is a convex set, by the  projection theorem (see e.g. \cite{Bertie09}) we have the following inequality for any $d \in \mathcal{C}^\star$.
\begin{align*}
&\left( (1-\eta \lambda)d^\star + \eta r^\star  - d^\star\right)^\top (d - d^\star) \le 0\\
\Rightarrow\ &\eta (\lambda d^\star - r^\star)^\top (d - d^\star) \le 0\\
\Rightarrow\ &\nabla \calP(d^\star; P^\star, r^\star)^\top (d - d^\star) \ge 0
\end{align*}
This implies that $d^\star$ maximizes the function $\calP(\cdot; P^\star, r^\star)$ over the set $\mathcal{C}^\star$ and is a stable point.
\end{proof}

\begin{lemma}\label{lem:deviation-d-to-h-projection}
Consider two state-action occupancy measures $d$ and $\hat{d}$.  Let $h$ (resp. $\hat{h}$) be the dual optimal solutions to the projection (\cref{eq:dual-projection}) corresponding to occupancy measure $d_t = d$ (resp. $\hat{d}$). Then we have the following inequality.
$$
\norm{\GD_\eta(d) - \GD_\eta(\hat{d})}_2^2 \le \left( 4(1-\eta \lambda)^2 + 4 \eta^2 \epsilon_r^2 + 8 \gamma^2 \norm{\hat{h}}_2^2 \epsilon_p^2 \right) \norm{d - \hat{d}}_2^2 + 8 \gamma^2 SA \norm{h - \hat{h}}_2^2
$$
\end{lemma}
\begin{proof}
Recall the relationship between the dual and the primal variables.
$$
\GD_\eta(d)(s,a) = (1-\eta \lambda) d(s,a) + \eta r(s,a) - h(s) + \gamma  \sum_{\tilde{s}} h(\tilde{s}) P(s,a,\tilde{s})
$$
This gives us the following bound on the difference $(\GD_\eta(d)(s,a) - \GD_\eta(\hat{d})(s,a))^2$.
\begin{align*}
&\left(\GD_\eta(d)(s,a) - \GD_\eta(\hat{d})(s,a) \right)^2 \le 4(1-\eta \lambda)^2  \left(d(s,a) - \hat{d}(s,a) \right)^2 + 4\eta^2 \left(r(s,a) - \hat{r}(s,a) \right)^2\\
&+  4\left( h(s) - \hat{h}(s)\right)^2 + 4\gamma^2 \left( \sum_{s'} h(s') P(s,a,s') - \sum_{s'} \hat{h}(s') \hat{P}(s,a,s') \right)^2\\
&\le 4(1-\eta \lambda)^2  \left(d(s,a) - \hat{d}(s,a) \right)^2 + 4\eta^2 \left(r(s,a) - \hat{r}(s,a) \right)^2\\
&+ 4\left( h(s) - \hat{h}(s)\right)^2 + 4\gamma^2 \left(  \sum_{s'} \left(h(s') - \hat{h}(s') \right) P(s,a,s') + \hat{h}(s') \left( P(s,a,s') - \hat{P}(s,a,s')\right)\right)^2\\
&\le 4(1-\eta \lambda)^2  \left(d(s,a) - \hat{d}(s,a) \right)^2 + 4\eta^2 \left(r(s,a) - \hat{r}(s,a) \right)^2\\
&+ 8\gamma^2 \left( \sum_{s'} \left(h(s') - \hat{h}(s') \right) P(s,a,s')\right)^2 + 8 \gamma^2 \left(\sum_{s'} \hat{h}(s') \left( P(s,a,s') - \hat{P}(s,a,s')\right) \right)^2\\
&\le 4(1-\eta \lambda)^2  \left(d(s,a) - \hat{d}(s,a) \right)^2 + 4\eta^2 \left(r(s,a) - \hat{r}(s,a) \right)^2\\
&+ 8\gamma^2 \norm{h - \hat{h}}_2^2 + 8 \gamma^2 \norm{\hat{h}}_2^2 \sum_{s'}  \left( P(s,a,s') - \hat{P}(s,a,s') \right)^2\\
\end{align*}
Now summing over $s$ and $a$ we get the following bound.
$$
\norm{\GD_\eta(d) - \GD_\eta(\hat{d})}_2^2 \le 4(1-\eta \lambda)^2 \norm{d - \hat{d}}_2^2 + 4\eta^2 \norm{r - \hat{r}}_2^2 + 8 \gamma^2 SA \norm{h - \hat{h}}_2^2 + 8 \gamma^2 \norm{\hat{h}}_2^2 \norm{P - \hat{P}}_2^2
$$
We now use the assumptions $\norm{r - \hat{r}}_2 \le \epsilon_r \norm{d - \hat{d}}_2$ and $\norm{P - \hat{P}}_2 \le \epsilon_p \norm{d - \hat{d}}_2$.
\begin{align*}
\norm{\GD_\eta(d) - \GD_\eta(\hat{d})}_2^2 \le \left( 4(1-\eta \lambda)^2 + 4 \eta^2 \epsilon_r^2 + 8 \gamma^2 \norm{\hat{h}}_2^2 \epsilon_p^2 \right) \norm{d - \hat{d}}_2^2 + 8 \gamma^2 SA \norm{h - \hat{h}}_2^2
\end{align*}

\end{proof}

\begin{lemma}\label{lem:strong-convexity-projection}
The objective function $\calP(\cdot; M)$ (as defined in ~\ref{eq:dual-projection}) is $A(1-\gamma)^2$-strongly convex.
\end{lemma}
\begin{proof}
The derivative of the objective function $\calP(\cdot; M)$ with respect to $h(s)$ is given as follows.
\begin{align}
    \begin{split}\label{eq:derivative-dual-projection}
        \frac{\partial \calP(h;M_t)}{\partial h(s)} &= - \sum_a \left((1-\eta \lambda) d_t(s,a) + \eta r_t(s,a) \right) + \rho(s) + \gamma \cdot \sum_{s',a} P_t(s',a,s) \left((1-\eta \lambda) d_t(s,a) + \eta r_t(s,a) \right) \\
        &+ A h(s) - \gamma \cdot \sum_{\tilde{s},a} h(\tilde{s}) \left(P_t(\tilde{s},a,s) + P_t(s,a,\tilde{s}) \right) + {\gamma^2} \sum_{s'} h(s') \sum_{\tilde{s},a} P_t(\tilde{s},a,s') P_t(\tilde{s},a,s)
    \end{split}
\end{align}
This gives us the following identity.
\begin{align*}
    \left(\nabla \calP(h;M_t) - \nabla \calP(\tilde{h};M_t) \right)^\top (h - \tilde{h}) &= A \norm{h - \tilde{h}}_2^2\\
    &- \gamma \cdot \sum_{s,s',a} (h(s') -\tilde{h}(s')) (P_t(s',a,s) + P_t(s,a,s')) (h(s) - \tilde{h}(s))\\
    &+ \gamma^2 \sum_{s',s}(h(s') - \tilde{h}(s')) \sum_{\tilde{s},a} P_t(\tilde{s},a,s') P_t(\tilde{s},a,s) (h(s) - \tilde{h}(s))
\end{align*}
Now for each action $a$, we define the following matrix $M_a \in R^{S \times S}$ with entries $M_a(s,s') = P_t(s,a,s')$. Note that matrix $M_a$ is row-stochastic and has eigenvalues bounded between $-1$ and $1$. 
\begin{align*}
     \left(\nabla \calP(h;M_t) - \nabla \calP(\tilde{h};M_t) \right)^\top (h - \tilde{h}) &= A \norm{h - \tilde{h}}_2^2 - \gamma \cdot (h - \tilde{h})^\top \left( \sum_a M_a + M_a^\top \right) (h - \tilde{h}) \\
     &+ \gamma^2 (h - \tilde{h})^\top \left(\sum_a M_a^\top M_a \right) (h - \tilde{h})\\
     &= (h -\tilh)^\top \sum_a \left( \Identity - \gamma (M_a + M_a^\top) + \gamma^2 M_a^\top M_a\right) (h - \tilh)\\
     &\ge A(1-\gamma)^2 \norm{h - \tilh}_2^2
\end{align*}
The last line uses lemma~\ref{lem:eigenvalue-bound}.
\end{proof}

\begin{lemma}\label{lem:dual-smoothness-projection}
The dual function (as defined in \ref{eq:dual-projection}) satisfies the following guarantee for any $h$, occupancy measures $d, \hat{d}$, and MDP $M, \hat{M}$.
\begin{align*}
\norm{\nabla \calP(h; M, d) - \nabla \calP(h; \widehat{M}, \hat{d})}_2 &\le 5A(1-\eta \lambda)^2 (1+2\gamma^2 S^2) \norm{d - \hat{d}}_2^2 \\&+ 5\eta^2 A \left( (1-\eta \lambda)^2 + 2\gamma^2 S^2 \right) \norm{r - \hat{r}}_2^2\\
&+ 5 \gamma^2 SA \left( \frac{2(1-\eta \lambda)^2}{(1-\gamma)^2} + 2\eta^2 + 6 \gamma^2 \norm{h}_2^2 \right) \norm{P - \hat{P}}_2^2
t\end{align*}
\end{lemma}
\begin{proof}
From the expression of the derivative of the function $\calP(\cdot; M, d)$ (\ref{eq:derivative-dual-projection}) we have the following bound.
\begin{align*}
&\norm{\nabla \calP(h; M, d) - \nabla \calP(h; \widehat{M}, \hat{d})}_2^2 = \sum_s \left\{- (1-\eta \lambda)\sum_a (d(s,a) - \hat{d}(s,a)) \right. \\
&+ \eta (1-\eta \lambda) \sum_a (r(s,a) - \hat{r}(s,a)) + \gamma \eta \sum_{s',a}\left(P(s',a,s) r(s,a) - \hat{P}(s',a,s) \hat{r}(s,a) \right)\\
&+ \gamma (1-\eta \lambda) \sum_{s',a} \left( P(s',a,s) d(s,a) - \hat{P}(s',a,s) \hat{d}(s,a) \right)\\
&- \gamma \cdot \sum_{\tilde{s},a} h(\tilde{s}) \left( P(\tilde{s},a,s) + P(s,a,\tilde{s}) - \hat{P}(\tilde{s},a,s) - \hat{P}(s,a,\tilde{s})\right) \\
& + \left. \gamma^2 \cdot \sum_{s'} h(s') \sum_{\tilde{s},a} \left(P(\tilde{s},a,s') P(s,a,\tilde{s}) - \hat{P}(\tilde{s},a,s') \hat{P}(s,a,\tilde{s})\right) \right\}^2\\
&\le 5(1-\eta \lambda)^2 \sum_s \left(\sum_a (d(s,a) - \hat{d}(s,a)) \right)^2 + 5 \eta^2 (1-\eta \lambda)^2 \sum_s \left( \sum_a (r(s,a) - \hat{r}(s,a)) \right)^2 \\
&+ 5 \gamma^2 \eta^2 \sum_s \left(\sum_{s',a} \left( P(s',a,s) r(s,a) - \hat{P}(s',a,s) \hat{r}(s,a) \right) \right)^2\\
&+ 5 \gamma^2 (1-\eta \lambda)^2 \sum_s \left(\sum_{s',a} \left( P(s',a,s) d(s,a) - \hat{P}(s',a,s) \hat{d}(s,a) \right) \right)^2\\
&+ 5 \gamma^2 \sum_s \left( \sum_{\tilde{s},a} h(\tilde{s}) \left( P(\tilde{s},a,s) + P(s,a,\tilde{s}) - \hat{P}(\tilde{s},a,s) - \hat{P}(s,a,\tilde{s})\right)\right)^2\\
&+ 5 \gamma^4 \sum_s \left(\sum_{s'} h(s') \sum_{\tilde{s},a} \left(P(\tilde{s},a,s') P(s,a,\tilde{s}) - \hat{P}(\tilde{s},a,s') \hat{P}(s,a,\tilde{s})\right) \right)^2
\end{align*}
We now establish several bounds to complete the proof. The bounds mainly use the Cauchy-Schwarz inequality and the Jensen's inequality.

\begin{align*}
\textrm{Bound 1}: &\sum_s \left(\sum_{s',a} \left( P(s',a,s) d(s,a) - \hat{P}(s',a,s) \hat{d}(s,a) \right) \right)^2 \\
&\le \sum_s \left( \sum_{s',a} P(s',a,s) \left(d(s,a) - \hat{d}(s,a) \right) + \hat{d}(s,a) \left(P(s',a,s) - \hat{P}(s',a,s) \right)\right)^2\\
&\le 2 \sum_s \left( \sum_{s',a} P(s',a,s) \left(d(s,a) - \hat{d}(s,a) \right) \right)^2 + 2 \sum_s \left( \sum_{s',a} \hat{d}(s,a) \left(P(s',a,s) - \hat{P}(s',a,s) \right) \right)^2\\
&\le 2 \sum_s \sum_{s',a}\left( d(s,a) - \hat{d}(s,a)\right)^2 \sum_{s',a} \left( P(s',a,s) \right)^2 \\&+ 2 \sum_s \sum_a (\hat{d}(s,a))^2 \sum_a \left( \sum_{s'} P(s',a,s) - \hat{P}(s',a,s)\right)^2\\
&\le 2 S \norm{d - \hat{d}}_2^2 \sum_{s,a,s'} P(s',a,s) + \frac{2AS}{(1-\gamma)^2} \sum_{s,a,s'}\abs{P(s',a,s) - \hat{P}(s',a,s)}^2\\
&\le 2 S^2 A \norm{d - \hat{d}}_2^2 + \frac{2AS}{(1-\gamma)^2} \norm{P - \hat{P}}_2^2
\end{align*}
Similarly one can establish the following bound.
\begin{align*}
\textrm{Bound 2}: &\sum_s \left(\sum_{s',a} \left( P(s',a,s) r(s,a) - \hat{P}(s',a,s) \hat{r}(s,a) \right) \right)^2 \le 2 S^2 A \norm{r - \hat{r}}_2^2 + {2AS} \norm{P - \hat{P}}_2^2
\end{align*}
\begin{align*}
\textrm{Bound 3}: &\sum_s \left( \sum_{\tilde{s},a} h(\tilde{s}) \left( P(\tilde{s},a,s) + P(s,a,\tilde{s}) - \hat{P}(\tilde{s},a,s) - \hat{P}(s,a,\tilde{s})\right)\right)^2\\
&\le \sum_s \sum_{\tilde{s}} h(\tilde{s})^2 \sum_{\tilde{s}} \left( \sum_a \left( P(\tilde{s},a,s) + P(s,a,\tilde{s}) - \hat{P}(\tilde{s},a,s) - \hat{P}(s,a,\tilde{s})\right) \right)^2\\
&\le A \norm{h}_2^2 \sum_{s,\tilde{s},a} \left( P(\tilde{s},a,s) + P(s,a,\tilde{s}) - \hat{P}(\tilde{s},a,s) - \hat{P}(s,a,\tilde{s})\right)^2\\
&\le 2 A \norm{h}_2^2 \norm{P - \hat{P}}_2^2
\end{align*}
\begin{align*}
\textrm{Bound 4}: & \sum_s \left(\sum_{s'} h(s') \sum_{\tilde{s},a} \left(P(\tilde{s},a,s') P(s,a,\tilde{s}) - \hat{P}(\tilde{s},a,s') \hat{P}(s,a,\tilde{s})\right) \right)^2\\
&\le \sum_s \sum_{s'} h(s')^2 \sum_{s'} \left(\sum_{\tilde{s},a} \left(P(\tilde{s},a,s') P(s,a,\tilde{s}) - \hat{P}(\tilde{s},a,s') \hat{P}(s,a,\tilde{s})\right) \right)^2\\
&= \norm{h}_2^2 \sum_{s,s'} \left( \sum_{\tilde{s},a} P(s,a,\tilde{s})\left(P(\tilde{s},a,s') - \hat{P}(\tilde{s},a,s') \right) + \hat{P}(\tilde{s},a,s') \left( P(s,a,\tilde{s}) - \hat{P}(s,a,\tilde{s}) \right) \right)^2\\
&\le 2 \norm{h}_2^2 \sum_{s,s'} \left(\sum_{\tilde{s},a}  P(s,a,\tilde{s})\left(P(\tilde{s},a,s') - \hat{P}(\tilde{s},a,s') \right) \right)^2\\
&+  2 \norm{h}_2^2 \sum_{s,s'} \left(\sum_{\tilde{s},a}  \hat{P}(\tilde{s},a,s') \left( P(s,a,\tilde{s}) - \hat{P}(s,a,\tilde{s}) \right) \right)^2\\
&\le 2 \norm{h}_2^2 \sum_{s,s'} \sum_{\tilde{s},a} P(s,a,\tilde{s}) \sum_{\tilde{s},a} \left( P(\tilde{s},a,s') - \hat{P}(\tilde{s},a,s')\right)^2\\
&+ 2 \norm{h}_2^2 \sum_{s,s'} \sum_{\tilde{s},a}\hat{P}(\tilde{s},a,s') \left( P(s,a,\tilde{s}) - \hat{P}(s,a,\tilde{s}) \right)^2\\
&\le 2 \norm{h}_2^2 \norm{P-\hat{P}}_2^2 \sum_{s,\tilde{s},a} P(s,a,\tilde{s}) + 2 \norm{h}_2^2 \sum_{s,\tilde{s},a} \left(P(s,a,\tilde{s}) - \hat{P}(s,a,\tilde{s}) \right)^2\\
&\le 4SA \norm{h}_2^2 \norm{P - \hat{P}}_2^2
\end{align*}
We now substitute the four bounds above to complete the proof.
\begin{align*}
&\norm{\nabla \calP(h; M, d) - \nabla \calP(h; \widehat{M}, \hat{d})}_2^2 \le 5(1-\eta \lambda)^2 A \norm{d-\hat{d}}_2^2 + 5\eta^2 (1-\eta \lambda)^2 A \norm{r - \hat{r}}_2^2\\
&+ 5 \gamma^2 (1-\eta\lambda)^2 \left( 2 S^2 A \norm{d - \hat{d}}_2^2 + \frac{2AS}{(1-\gamma)^2} \norm{P - \hat{P}}_2^2\right) \\
&+ 5\gamma^2 \eta^2 \left(2 S^2 A \norm{r - \hat{r}}_2^2 + {2AS} \norm{P - \hat{P}}_2^2 \right)\\
&+ 10 \gamma^2 \norm{h}_2^2 \norm{P - \hat{P}}_2^2 + 20\gamma^4 SA \norm{h}_2^2 \norm{P - \hat{P}}_2^2  \\
&\le 5A(1-\eta \lambda)^2 (1+2\gamma^2 S^2) \norm{d - \hat{d}}_2^2 + 5\eta^2 A \left( (1-\eta \lambda)^2 + 2\gamma^2 S^2 \right) \norm{r - \hat{r}}_2^2\\
&+ 5 \gamma^2 SA \left( \frac{2(1-\eta \lambda)^2}{(1-\gamma)^2} + 2\eta^2 + 6 \gamma^2 \norm{h}_2^2 \right) \norm{P - \hat{P}}_2^2\\
\end{align*}
\end{proof}

\begin{lemma}\label{lem:bound-optimal-dual-projection}
For any choice of MDP $M$ and occupancy measure $d$, the $L_2$-norm of the optimal solution to the dual objective (as defined in \ref{eq:dual-projection}) is bounded by
$$
\frac{1+2\eta S}{(1-\gamma)^2} + 2\abs{1-\eta \lambda} \frac{S/\sqrt{A}}{(1-\gamma)^3}.
$$
\end{lemma}
\begin{proof}
The objective function $\calP$ is strongly-convex and has a unique solution. We set the derivative with respect to $h$ to zero and get the following system of equations.
\begin{align*}
&h(s) \left( A - 2\gamma \sum_a P(s,a,s) + \gamma^2 \sum_{\tilde{s},a} P(\tilde{s},a,s)^2 \right) \\
&+ \sum_{s' \neq s} h(s') \left( -\gamma \sum_a (P(s',a,s) + P(s,a,s')) + \gamma^2 \sum_{\tilde{s},a} P(\tilde{s},a,s') P(\tilde{s},a,s) \right)\\
&= \sum_a \left((1-\eta \lambda) d(s,a) + \eta r(s,a) \right) - \rho(s) - \gamma \cdot \sum_{s',a} P(s',a,s) \left((1-\eta \lambda) d(s,a) + \eta r(s,a) \right)
\end{align*}
Let us now define matrix $B \in \R^{S \times S}$ and a vector $b \in \R^S$ with the following entries.
\begin{align*}
B(s,s') = \left\{ \begin{array}{cc}
A - 2\gamma \sum_a P(s,a,s) + \gamma^2 \sum_{\tilde{s},a} P(\tilde{s},a,s)^2 & \textrm{if } s' = s\\
 - \gamma \sum_a (P(s',a,s) + P(s,a,s')) + \gamma^2 \sum_{\tilde{s},a} P(\tilde{s},a,s') P(\tilde{s},a,s) & \textrm{o.w.}
\end{array}\right.
\end{align*}
\begin{align*}
b(s) = \sum_a \left((1-\eta \lambda) d(s,a) + \eta r(s,a) \right) - \rho(s) - \gamma \cdot \sum_{s',a} P(s',a,s) \left((1-\eta \lambda) d(s,a) + \eta r(s,a) \right)
\end{align*}
Then the optimal solution is the solution to the system of equations $Bx=b$. We now provide a bound on the $L_2$-norm of such a solution. For each action $a$, we define a matrix $M_a \in \R^{S \times S}$ with entries $M_a(s,s') = P(s,a,s')$. Then the matrix $B$ can be expressed as
follows.
$$
B = A \cdot \Identity - \gamma \sum_a (M_a + M_a^\top) + \gamma^2 \sum_a M_a^\top M_a \succcurlyeq A(1-\gamma)^2 \Identity
$$
The last inequality uses lemma~\ref{lem:eigenvalue-bound}. This also implies that for $\gamma < 1$ the matrix $B$ is invertible.

We now bound the norm of the vector $b$. We will use the fact that for any state $\sum_{s,a} d(s,a) = 1/(1-\gamma)$.
\begin{align*}
    \norm{b}_2 &\le \norm{b}_1 \le \abs{1-\eta \lambda} \sum_{s,a}  d(s,a) + \eta \sum_{s,a} r(s,a)  + \sum_{s,a} \rho(s)\\
    &+ \gamma \abs{1-\eta \lambda} \cdot \sum_{s, s',a} P(s',a,s) d(s,a)  + \eta \gamma \sum_{s,s',a} P(s',a,s) \abs{r(s,a)} \\
    &\le \frac{\abs{1-\eta \lambda}}{1-\gamma} + \eta SA + A + \gamma \abs{1-\eta \lambda} \sum_{s'} \sqrt{\sum_{s,a} P(s',a,s)^2} \sqrt{\sum_{s,a} d(s,a)^2} + \eta \gamma SA\\
    &\le \frac{\abs{1-\eta \lambda}}{1-\gamma} + \eta SA + A + \gamma \abs{1-\eta \lambda} \norm{d}_2 \sum_{s'} \sqrt{\sum_{s,a} P(s',a,s) } + \eta \gamma SA\\
    &\le \frac{\abs{1-\eta \lambda}}{1-\gamma} + \eta SA + A + \gamma \abs{1-\eta \lambda} \frac{S\sqrt{A}}{1-\gamma}  + \eta \gamma SA \le A (1 + 2\eta S) + 2\abs{1-\eta \lambda} \frac{S\sqrt{A}}{1-\gamma} 
\end{align*}
The optimal solution to the dual objective is bounded by 
$$
\norm{A^{-1} b}_2 \le \frac{\norm{b}_2}{\lambda_{\min}(A) } \le \frac{1+2\eta S}{(1-\gamma)^2} + 2\abs{1-\eta \lambda} \frac{S/\sqrt{A}}{(1-\gamma)^3}
$$
\end{proof}

%% file: finite_sample_convergence.tex
\subsection{Formal Statement and Proof of Convergence with Finite Samples (Theorem~\ref{thm:finite-samples-convergence})}
\begin{theorem}
Suppose assumption~\ref{sensitivity-assumption} holds with $\lambda \ge \frac{24 S^{3/2} (2\epsilon_r + 5S\epsilon_p)}{(1-\gamma)^4}$, and assumption~\ref{asn:overlap} holds with parameter $B$. For a given $\delta$, and error probability $p$, if we repeatedly solve the optimization problem~\ref{eq:repeated-optim-finite} with number of samples 
\begin{align*}
    m_t \ge \frac{64 A (B+\sqrt{A})^2}{\beta^4 \delta^4 (2\epsilon_r + 5S \epsilon_p)^2} \left( \ln\left(\frac{t}{p}\right) + \ln\left(\frac{4S(B+\sqrt{A})}{\beta \delta(2\epsilon_r + 5S \epsilon_p)} \right)\right)
\end{align*}
then we have
\begin{align*}
    \norm{d_t - d_S}_2 \le \delta \ \forall t \ge (1-\mu)^{-1}\ln \left( \frac{2}{\delta(1-\gamma)}\right)\quad \textrm{with probability at least } 1-p
\end{align*}
where $\mu = \frac{24 S^{3/2} (2\epsilon_r + 5S\epsilon_p)}{(1-\gamma)^4}$.
\end{theorem}
 
\begin{proof}
 
We will write $\widehat{\GD}(d_t)$ to denote the occupancy measure $d_{t+1}$. 
$$
(\widehat{\GD}(d_t), \hat{h}_{t+1}) = \argmax_d \ \argmin_h \ \hat{\calL}(d,h; M_t)
$$
Let us also write $\widehat{\GD}(d_S)$ to denote the primal solution corresponding to the stable solution $d_S$ i.e.
$$
(\widehat{\GD}(d_S), \hat{h}_S) = \argmax_d \ \argmin_h \ \hat{\calL}(d,h; M_S)
$$
Let us also recall the definition of the operator $\GD(\cdot)$. Given occupancy measure $d_{t}$, $\GD(d_t)$ is the optimal solution to the optimization problem~\ref{eq:regularized-rl-discounting} when we use the exact model $M_t$. Because of strong-duality this implies there exists $h_{t+1}$ so that
$$
({\GD}(d_t), {h}_{t+1}) = \argmax_d \ \argmin_h \ {\calL}(d,h; M_t)
$$
Since $\GD(d_S) = d_S$ there also exists $h_S$ so that 
$$
(d_S, h_S) = \argmax_d \ \argmin_h \ {\calL}(d,h; M_S)
$$
Because of lemma~\ref{lem:bound-optimal-dual} we can assume the $L_2$-norms of the dual solutions $h_{t+1}, \hat{h}_S$, and $\hat{h}_S$ are bounded by $\frac{3S}{(1-\gamma)^2}$. Since there exists a saddle point with bounded norm, we can just consider the restricted set $\calH = \set{h : \norm{h}_2 \le \frac{3S}{(1-\gamma)^2}}$.\footnote{See lemma 3 of \cite{ZHHJ+22} for a proof of this statement.} Moreover, by assumption~\ref{asn:overlap} we know that $\GD(d_t)(s,a) / d_t(s,a) \le B$ for any $(s,a)$. Therefore, we can apply lemma \ref{lem:concentration-Lagrangian} with $\delta_1 = p / 2 t^2$ and $H = 3S/(1-\gamma)^2$ to get the following bound,
\begin{align}\label{eq:bound-empirical-Lagrangian}
    \abs{\hat{\calL}(d,h; M_t) - \calL(d,h; M_t)} \le \frac{18 S^{1.5}(B + \sqrt{A})\epsilon}{(1-\gamma)^3}
\end{align}
as long as 
\begin{align}\label{eq:bound-num-samples}
m_t \ge \frac{4A}{\epsilon^2} (\ln(t/p) + \ln(S/(1-\gamma) \epsilon))
\end{align}
$h \in \calH$ and $\max_{s,a} d(s,a) / d_t(s,a) \le B$. Since the event~\eqref{eq:bound-empirical-Lagrangian} holds at time $t$ with probability at least $1-\frac{p}{2t^2}$, by a union bound over all time steps, this event holds with probability at least $1-p$.

Note that the objective $\calL(\cdot, h_{t+1}; M_t) $ is $\lambda$-strongly concave. Therefore, we have
$$
\calL(\widehat{\GD}(d_t), {h}_{t+1}; M_t) - \calL(\GD(d_t), {h}_{t+1}; M_t) \le - \frac{\lambda}{2} \norm{\GD(d_t) - \widehat{\GD}(d_t)}_2^2.
$$
Rearranging and using lemma~\ref{lem:apx-saddle-point} we get the following bound.
\begin{align*}
\norm{\GD(d_t) - \widehat{\GD}(d_t)}_2 &\le \sqrt{\frac{2 \left(\calL(\GD(d_t), {h}_{t+1}; M_t) - \calL(\widehat{\GD}(d_t), {h}_{t+1}; M_t) \right)}{\lambda}} \\
&\le \frac{6\sqrt{ S^{1.5}(B + \sqrt{A}) \epsilon }}{(1-\gamma)^{1.5} } \frac{1}{\sqrt{\lambda}}
\end{align*}

The proof of theorem~\ref{thm:primal-convergence} establishes that the operator $\GD(\cdot)$ is a contraction. In particular, it shows that
$$
\norm{\GD(d_t) - d_S}_2 \le \beta \norm{d_t - d_S}_2\quad \textrm{for } \beta = \frac{12 S^{3/2} (2\epsilon_r + 5\epsilon_p)}{\lambda (1-\gamma)^4} \ \textrm{ and } \lambda > 12 S^{3/2}(1-\gamma)^{-4} (2\epsilon_r + 5S \epsilon_p)
$$
This gives us the following recurrence relation on the iterates of the algorithm.
\begin{align*}
    \norm{d_{t+1} - d_S}_2 &= \norm{\widehat{\GD}(d_{t}) - d_S}_2 \le \norm{\widehat{\GD}(d_{t}) - {\GD}(d_{t})}_2 + \norm{{\GD}(d_{t}) - d_S}_2\\
    &\le \frac{6\sqrt{ S^{1.5}(B + \sqrt{A}) \epsilon }}{(1-\gamma)^{1.5} } \frac{1}{\sqrt{\lambda}} + \beta \norm{d_t - d_S}_2
\end{align*}
We choose $\lambda = 24 S^{3/2}(1-\gamma)^{-4}(2\epsilon_r + 5S\epsilon_p)$ which ensures $\beta < 1/2$ and gives the following recurrence.
\begin{align}\label{eq:recurrence-sampling}
    \norm{d_{t+1} - d_S}_2 \le 2\sqrt{1-\gamma} \sqrt{\frac{(B + \sqrt{A}) \epsilon}{2\epsilon_r + 5S\epsilon_p}} + \beta \norm{d_t - d_S}_2 \le \beta \delta + \beta \norm{d_t - d_S}_2
\end{align}
The last line requires the following bound on $\epsilon$.
\begin{align}
    \label{eq:epsilon-bound}
    \epsilon \le \frac{\beta^2 \delta^2 (2 \epsilon_r + 5S \epsilon_p)}{4 (1-\gamma) (B + \sqrt{A})}
\end{align}
Substituting the bound on $\epsilon$ in equation~\eqref{eq:bound-num-samples} the required number of samples at time-step $t$ is given as follows.
\begin{align*}
    m_t \ge \frac{64 A (B+\sqrt{A})^2}{\beta^4 \delta^4 (2\epsilon_r + 5S \epsilon_p)^2} \left( \ln\left(\frac{t}{p}\right) + \ln\left(\frac{4S(B+\sqrt{A})}{\beta \delta(2\epsilon_r + 5S \epsilon_p)} \right)\right)
\end{align*}

In order to analyze the recurrence relation~\eqref{eq:recurrence-sampling} we consider two cases. First, if $\norm{d_t - d_S}_2 \ge \delta$ we have
\begin{align*}
    \norm{d_{t+1} - d_S}_2 \le 2\beta \norm{d_t - d_S}_2
\end{align*}
Since $\beta < 1/2$ this is a contraction, and after $\ln(\norm{d_0 - d_S}_2)/\ln(1/2\beta)$ iterations we are guaranteed that $\norm{d_t - d_S}_2 \le \delta$. Since $\norm{d_0 - d_S}_2 \le \frac{2}{1-\gamma}$, the required number of iterations for this event to occur is given by the following upper bound.
$$
\frac{\ln(\norm{d_0 - d_S}_2)}{\ln(1/2\beta)} \le 2(1-2\beta)^{-1} \ln \left( \frac{2}{\delta(1-\gamma)}\right)
$$

On the other hand, if $\norm{d_t - d_S} \le \delta$, equation~\eqref{eq:recurrence-sampling} gives us
$$
\norm{d_{t+1}-d_S}_2 \le 2\beta \delta < \delta\quad \textrm{[Since $\beta < 1/2$]}
$$
Therefore, once $\norm{d_t - d_S}_2 \le \delta$ we are guaranteed that $\norm{d_{t'} - d_S}_2 \le \delta$ for any $t' \ge t$.
\end{proof}

\begin{lemma}\label{lem:concentration-Lagrangian}
Suppose $m \ge \frac{1}{\epsilon^2} \left( A\ln(2/\delta_1) +  \ln(4H/\epsilon) + 2 A \ln(\ln(SABH/\epsilon) /\epsilon)\right)$. Then for any occupancy measure $d$ satisfying $\max_{s,a} d(s,a) / d_t(s,a) \le B$ and any $h$ with $\norm{h}_2 \le H$ the following bound holds with probability at least $1-\delta_1$.
$$
\abs{\hat{\calL}(d,h; M_t) - \calL(d,h; M_t)} \le \frac{6H\sqrt{S}(B + \sqrt{A}) \epsilon}{1-\gamma}
$$
\end{lemma}
\begin{proof}
Note that the expected value of the objective above equals $\calL(d,h; M_t)$.
\begin{align*}
    &\E\left[ \hat{\calL}(d,h; M_t)\right] = -\frac{\lambda}{2} \norm{d}_2^2 + \sum_s h(s) \rho(s) + \frac{1}{m} \sum_{i=1}^m \E\left[\frac{d(s_i,a_i)}{d_t(s_i,a_i)} \left( r(s_i,a_i) - h(s_i) + \gamma h(s_i') \right)\right]\\
    &=  -\frac{\lambda}{2} \norm{d}_2^2 + \sum_s h(s) \rho(s) + \sum_{s,a} d_t(s,a) \frac{d(s,a)}{d_t(s,a)} \left( r(s,a) - h(s) + \gamma \sum_{s'} P_t(s'|s,a) h(s')\right)\\
    &= \calL(d,h;M_t)
\end{align*}
By the overlap assumption~\ref{asn:overlap} and the assumption $\norm{h}_2 \le H$ the following bound holds for each $i$, $$\frac{1}{1-\gamma} \frac{d(s_i,a_i)}{d_t(s_i,a_i)}\left( r(s_i,a_i) - h(s_i) + \gamma h(s_i')\right) \le \frac{2BH}{1-\gamma}.$$ 
Therefore we can apply the Chernoff-Hoeffding inequality and obtain the following inequality.
$$
P\left( \abs{\hat{\calL}(d,h; M_t) - \calL(d,h; M_t)} \ge \frac{2BH}{1-\gamma} \sqrt{\frac{\ln(2/\delta_1) }{m} }\right) \le \delta_1
$$
We now extend this bound to any occupancy measure $d \in \calD$ and $h$ in the set $\calH = \set{h : \norm{h}_2 \le H}$. 
By lemma 5.2 of \cite{Vershy10} we can choose an $\epsilon$-net, $\calH_\epsilon$ of the set $\calH$ of size at most $\left(1 + \frac{2H}{\epsilon} \right)^S$ and for any  point $h \in \calH$ we are guaranteed to find $h' \in \calH_\epsilon$ so that $\norm{h - h'}_2$. 

However, such an additive error bound is not sufficient for the set of occupancy measures because of the overlap assumption~\ref{asn:overlap}. So instead we choose a multiplicative $\epsilon$-net as follows. For any $(s,a)$ we consider the grid points $d_t(s,a), (1+\epsilon) d_t(s,a), \ldots, (1+\epsilon)^p d_t(s,a)$ for $p = \log(B/d_t(s,a))/\log(1+\epsilon)$. Although $d_t(s,a)$ can be arbitrarily small, without loss of generality we can assume that $d_t(s,a) \ge \frac{\epsilon}{4SABH}$. This is because from the expression of $\calL(d,h; M_t)$~\eqref{eq:empirical-Lagrangian}, it is clear that ignoring such small terms introduces an error of at most $\epsilon/ 4$. Therefore we can choose $p = 2 \log(SABH/\epsilon) / \log(1+\epsilon)$. Taking a Cartesian product over the set of all state, action pairs we see that we can choose an $\epsilon$-net $\calD_\epsilon$ so that $\abs{\calD_\epsilon} \le \left( \frac{2 \log(SABH/\epsilon) }{\log(1+\epsilon)}\right)^{SA} \le \left( \frac{2 \log(SABH/\epsilon) }{\epsilon}\right)^{SA}$. Notice that we are guaranteed that for any $d \in \calD$ there exists a $d' \in \calD_\epsilon$ such that $d(s,a) / d'(s,a) \le 1+\epsilon$.

By a union bound over the elements in $\calH_\epsilon$ and $\calD_\epsilon$ the following bound holds for any $d \in \calD_\epsilon$ and $h \in \calH_\epsilon$ with probability at least $1-\delta_1$.
$$
\abs{\hat{\calL}(d,h; M_t) - \calL(d,h; M_t)} \le \frac{2BH}{1-\gamma} \underbrace{\sqrt{\frac{SA\ln\left(\frac{2}{\delta_1}\right) + S \ln\left(\frac{4H}{\epsilon}\right) + 2 SA \ln(\ln(SABH/\epsilon) /\epsilon) }{m} }}_{:=T_m} 
$$
We now extend the bound above for any $d \in \calD$ and $h \in \calH$ using lemma \ref{lem:epsnet-apx}. There exists $\tilde{d} \in \calD_\epsilon$ so that $\max_{s,a} d(s,a) / \tilde{d}(s,a) \le \epsilon$. Similarly there exists $h_\epsilon \in \calH_\epsilon$ so that $\norm{h - \tilde{h}}_2 \le \epsilon$. Let $\calL^0(d,h;M_t) =\calL(d,h; M_t) +\frac{\lambda}{2} \norm{d}_2^2 - \sum_s h(s) \rho(s)$.

\begin{align*}
    &\abs{\hat{\calL}(d,h; M_t) - \calL(d,h; M_t)} \le \abs{\hat{\calL}^0(d,h; M_t) - \hat{\calL}^0(\tilde{d},\tilde{h}; M_t)}\\
    &+ \abs{\hat{\calL}(\tilde{d},\tilde{h}; M_t) - {\calL}(\tilde{d},\tilde{h}; M_t)} + \abs{{\calL}^0(\tilde{d},\tilde{h}; M_t) - {\calL}^0({d},{h}; M_t)}\\
    &\le \frac{2BHT_m}{1-\gamma} +  \frac{6\sqrt{SA}H\epsilon}{1-\gamma} + \frac{4BH\sqrt{S}\epsilon}{1-\gamma}
\end{align*}
Therefore, if $m \ge \frac{1}{\epsilon^2} \left( A\ln(2/\delta_1) +  \ln(4H/\epsilon) + 2 A \ln(\ln(SABH/\epsilon) /\epsilon)\right)$ then $T_m \le \sqrt{S} \epsilon$ and we have the following bound.
$$
\abs{\hat{\calL}(d,h; M_t) - \calL(d,h; M_t)} \le \frac{6H\sqrt{S}(B + \sqrt{A}) \epsilon}{1-\gamma}
$$
\end{proof}

\begin{lemma}\label{lem:epsnet-apx}
Suppose we are guaranteed that $\frac{d(s,a)}{\tilde{d}(s,a)} \le 1+\epsilon$ and $\norm{h - \tilde{h}}_2 \le \epsilon$, and $\norm{h}_2, \norm{\tilde{h}}_2 \le H$. Let $\calL^0(d,h;M_t) =\calL(d,h; M_t) +\frac{\lambda}{2} \norm{d}_2^2 - \sum_s h(s) \rho(s)$ and define $\hat{\calL}^0(d,h; M_t)$ analogously. Then the following inequalities hold.
$$
\abs{\calL^0(d,h;M_t) - \calL^0(\tilde{d},\tilde{h}; M_t)} \le \frac{6 \sqrt{SA} H \epsilon}{1-\gamma} \quad \textrm{ and }
$$
$$
\abs{\hat{\calL}^0(d,h;M_t) - \hat{\calL}^0(\tilde{d},\tilde{h}; M_t)} \le  \frac{4BH\sqrt{S} \epsilon}{1-\gamma}
$$
\end{lemma}
\begin{proof}
First note that $\norm{d - \tilde{d}}_2^2 = \sum_{s,a}\left( d(s,a) - \tilde{d}(s,a)\right)^2 \le \sum_{s,a} d(s,a)^2 \epsilon^2 \le \frac{\epsilon^2}{(1-\gamma)^2}$.
\begin{align*}
    &\abs{\calL^0(d,h;M_t) - \calL^0(\tilde{d},\tilde{h}; M_t)} \le \sum_{s,a} \abs{d(s,a) - \tilde{d}(s,a)} \\
    &+ \underbrace{\abs{\sum_{s,a} d(s,a) h(s) - \tilde{d}(s,a) \tilde{h}(s)}}_{:=T_1} + \gamma \underbrace{\sum_{s,a} \abs{d(s,a) \sum_{s'} P_t(s,a,s') h(s') - \tilde{d}(s,a) \sum_{s'} P_t(s,a,s') \tilde{h}(s')} }_{:=T_2}
\end{align*}
We now bound the terms $T_1$ and $T_2$.
\begin{align*}
    T_1 &= \abs{\sum_{s,a} d(s,a) \left( h(s) - \tilde{h}(s)\right) + \tilde{h}(s) \left( d(s,a) - \tilde{d}(s,a)\right)}\\
    &\le \norm{h - \tilde{h}}_1 \sum_{s,a} d(s,a) + \sqrt{\sum_s (\tilde{h}(s))^2} \sqrt{\sum_s \left( \sum_a (d(s,a) - \tilde{d}(s,a)\right)^2}\\
    &\le \frac{\sqrt{S}\epsilon}{1-\gamma} + \frac{H \sqrt{A} \epsilon }{1-\gamma}
\end{align*}
\begin{align*}
    T_2 &= \sum_{s,a} \abs{d(s,a) \sum_{s'} P_t(s,a,s') h(s') - \tilde{d}(s,a) \sum_{s'} P_t(s,a,s') \tilde{h}(s')} \\
    &= \sum_{s,a} \abs{d(s,a) - \tilde{d}(s,a)} \sum_{s'} P_t(s,a,s') \abs{h(s')}
    + \sum_{s,a} \tilde{d}(s,a) \sum_{s'}P_t(s,a,s') \abs{h(s') - \tilde{h}(s')}\\
    &\le \norm{h}_2 \norm{d - \tilde{d}}_1 + \norm{h - \tilde{h}}_1 \sum_{s,a} \tilde{d}(s,a)\\
    &\le \frac{\sqrt{SA}H \epsilon}{1-\gamma} + \frac{\sqrt{S} \epsilon}{1-\gamma}
\end{align*}
Substituting the bounds on $T_1$ and $T_2$ we get the following bound on $\abs{\calL(d,h; M_t) - \calL(\tilde{d}, \tilde{h}; M_t)}$.
\begin{align*}
    \abs{\calL^0(d,h;M_t) - \calL^0(\tilde{d},\tilde{h}; M_t)} \le \sqrt{SA} \epsilon + \frac{2\sqrt{S} \epsilon}{1-\gamma} + \frac{2 \sqrt{SA} H \epsilon}{1-\gamma}
\end{align*}
We now consider bounding the difference $\abs{\hat{\calL}^0(d,h;M_t) - \hat{\calL}^0(\tilde{d},\tilde{h}; M_t)}$.
\begin{align*}
    &\abs{\hat{\calL}^0(d,h;M_t) - \hat{\calL}^0(\tilde{d},\tilde{h}; M_t)} \le\\ 
    &+ \underbrace{\frac{1}{m(1-\gamma)} \sum_{i=1}^m \abs{\frac{d(s_i,a_i)}{d_t(s_i,a_i)} (r(s_i,a_i) - h(s_i) + \gamma h(s'_i)) - \frac{\tilde{d}(s_i,a_i)}{d_t(s_i,a_i)} (r(s_i,a_i) - \tilde{h}(s_i) + \gamma \tilde{h}(s'_i))} }_{:= T_3}
\end{align*}
We now bound the term $T_3$.
\begin{align*}
    T_3 &= \frac{1}{m(1-\gamma)} \sum_{i=1}^m \abs{\frac{d(s_i,a_i)}{d_t(s_i,a_i)} (r(s_i,a_i) - h(s_i) + \gamma h(s'_i)) - \frac{\tilde{d}(s_i,a_i)}{d_t(s_i,a_i)} (r(s_i,a_i) - \tilde{h}(s_i) + \gamma \tilde{h}(s'_i))} \\
    &\le \frac{1}{m(1-\gamma)} \sum_{i=1}^m B \left( \abs{h(s_i) - \tilde{h}(s_i)} + \gamma \abs{h(s'_i) - \tilde{h}(s'_i)} \right)\\
    &+ \frac{1}{m(1-\gamma)} \sum_{i=1}^m \abs{\frac{d(s_i,a_i) - \tilde{d}(s_i,a_i)}{d_t(s_i,a_i)}} \abs{r(s_i,a_i) - \tilde{h}(s_i) + \gamma \tilde{h}(s'_i)} \\
    &\le \frac{B}{1-\gamma} \norm{h - \tilde{h}}_1 + \frac{\epsilon B}{1-\gamma} \left(1 + 2 \norm{h}_1\right)\\
    &\le \frac{B\sqrt{S} \epsilon}{1-\gamma} + \frac{3BH\sqrt{S}\epsilon}{1-\gamma}
\end{align*}
\end{proof}

\begin{lemma}\label{lem:apx-saddle-point}
Let $(d^\star, h^\star) \in \argmax_d\ \argmin_h \calL(d,h; M)$ and $(\hat{d}, \hat{h}) \in \argmax_d\ \argmin_h {\calL}(d,h; \widehat{M})$. Moreover, 
suppose $\abs{\calL(d,h; M) - \calL(d,h; \widehat{M})} \le \epsilon$ for all $d$ and $h$ with $\norm{h}_2 \le 3S/(1-\gamma)^2$. Then we have
$$
\calL(d^\star, h^\star) - \calL(\hat{d}, h^\star) \le 2\epsilon
$$
\end{lemma}
\begin{proof}
 We will drop the dependence on the underlying model and write $\calL(d, h; M)$ as $\calL(d,h)$ and $\calL(d,h; \widehat{M})$ as $\hat{\calL}(d,h)$.
\begin{align*}
    \calL(d^\star, h^\star) - \calL(\hat{d}, h^\star) &= \underbrace{\calL(d^\star, h^\star) - \calL(d^\star, \hat{h}(d^\star)}_{:=T_1} + \underbrace{\calL(d^\star, \hat{h}(d^\star)) - \hat{\calL}(d^\star, \hat{h}(d^\star))}_{:= T_2} \\&+ \underbrace{\hat{\calL}(d^\star, \hat{h}(d^\star)) - \hat{\calL}(\hat{d}, \hat{h})}_{:= T_3} + \underbrace{\hat{\calL}(\hat{d}, \hat{h}) - \hat{\calL}(\hat{d}, h^\star)}_{:=T_4} + \underbrace{\hat{\calL}(\hat{d}, h^\star) - \calL(\hat{d}, h^\star)}_{:= T_5}
\end{align*}
Here we write $\hat{h}(\tilde{d}) = \argmin_h \hat{\calL}(\tilde{d},h)$ i.e. the dual solution that minimizes the objective $\hat{\calL}(\tilde{d}, \cdot)$. Since $h^\star$ minimizes $\calL(d^\star, \cdot)$, by lemma~\ref{lem:bound-optimal-dual} we have $\norm{h^\star}_2 \le 3S/(1-\gamma)^2$. By a similar argument $\norm{\hat{h}}_2 \le 3S/(1-\gamma)^2$. Therefore,
both $T_2$ and $T_5$ are at most $\epsilon$. 

Given $d^\star$, $h^\star$ minimizes $\calL(d^\star, \cdot)$. Therefore, the term $T_1$ is at most zero. By a similar argument the term $T_4$ is at most zero. Now for the term $T_3$, notice that $\hat{h} = \hat{h}(\hat{d})$ and it also minimizes the objective $\hat{\calL}(d, \hat{h}(d))$. Therefore, the term $T_3$ is also at most zero.
\end{proof}